%% file: main.tex
\documentclass[11pt]{article}

\usepackage{setspace}
 
\usepackage{microtype}
\usepackage{graphicx}
\usepackage{subfigure}
\usepackage{booktabs}  
 
\usepackage{amsmath}
\usepackage{amssymb}
\usepackage{amsthm}
 
\usepackage{hyperref}       % hyperlinks
\usepackage{url}            % simple URL typesetting
\hypersetup{colorlinks,linkcolor={red},citecolor={blue},urlcolor={red}}  

\usepackage[utf8]{inputenc} % allow utf-8 input
\usepackage[T1]{fontenc}    % use 8-bit T1 fonts 
\usepackage{lmodern}  
 
\usepackage{multirow}

\usepackage{natbib}

\usepackage{enumitem}
 \usepackage{array}
\usepackage{makecell}
\usepackage{mystyle1}
 
\title{Single-Timescale Stochastic Nonconvex-Concave Optimization for Smooth Nonlinear TD Learning}

\date{}

\begin{document}

\author{Shuang Qiu\thanks{University of Michigan. 
Email: \texttt{qiush@umich.edu}.} 
       \qquad
       Zhuoran Yang\thanks{
       Princeton University. 
    Email: \texttt{zy6@princeton.edu}.}
	\qquad
		Xiaohan Wei\thanks{Facebook, Inc.
Email: \texttt{ubimeteor@fb.com}.} 
	\qquad
       Jieping Ye\thanks{University of Michigan. 
    Email: \texttt{jpye@umich.edu}.}
	\qquad    
       Zhaoran Wang\thanks{Northwestern University.
    Email: \texttt{zhaoranwang@gmail.com}.}
       }       

\maketitle

\begin{abstract}
Temporal-Difference (TD) learning with nonlinear smooth function approximation for policy evaluation has achieved great success in modern reinforcement learning. It is shown that such a problem can be reformulated as a \emph{stochastic nonconvex-strongly-concave} optimization problem, which is challenging as naive stochastic gradient descent-ascent algorithm suffers from slow convergence. Existing approaches for this problem are based on two-timescale or double-loop stochastic gradient algorithms, which may also require sampling large-batch data. However, in practice, a \emph{single-timescale} \emph{single-loop} stochastic algorithm is preferred due to its simplicity and also because its step-size is easier to tune. In this paper, we propose two \emph{single-timescale} \emph{single-loop} algorithms which require only \emph{one} data point each step. Our first algorithm implements momentum updates on both primal and dual variables achieving an $O(\varepsilon^{-4})$ sample complexity, which shows the important role of momentum in obtaining a single-timescale algorithm. Our second algorithm improves upon the first one by applying variance reduction on top of momentum, which matches the best known $O(\varepsilon^{-3})$ sample complexity in existing works. Furthermore, our variance-reduction algorithm does not require a large-batch checkpoint. Moreover, our theoretical results for both algorithms are expressed in a \emph{tighter} form of simultaneous primal and dual side convergence.
\end{abstract}

\section{Introduction}

Reinforcement Learning (RL) powered by neural networks has recently achieved state-of-the-art performance on many high-dimensional control and planning tasks. Policy evaluation (PE), which aims at estimating the value function corresponding to a certain policy, is a stepping stone of policy improvements and serves as an essential component of various RL algorithms.  It is therefore crucial to design sample efficient PE algorithms estimating value functions with approximation guarantees.

One of the most prevailing classes of PE methods is the TD learning with function approximation \citep{dann2014policy}, 
whose goal is to minimize the Bellman error by approximating the value function of a policy via smooth functions, e.g. neural networks, of certain learnable parameters. Much like other fields of modern machine learning, despite its recent extensive empirical successes, e.g. \citet{schulman2015high, silver2017mastering}, theoretical understanding of general nonlinear function approximation remains limited.

Most existing TD learning algorithms with theoretical guarantees are restricted to linear function approximations owning to its mathematical conciseness \citep{tsitsiklis1997analysis,sutton2009fast,sutton2009convergent,liu2015finite,du2017stochastic,wang2017finite,dalal2017finite,yu2017convergence,touati2017convergent,bhandari2018finite,srikant2019finite}. However, the linearity assumption is more often than not oversimplified and insufficient to explain the effectiveness of general function approximations in practice.  

This motivates researchers to study a more practical setting where value functions are parametrized by nonlinear and smooth functions. The work \citep{bhatnagar2009convergent} is among the first to propose a general framework for such a setup via minimizing a generalized mean squared projected Bellman error (MSPBE) objective. It can be shown via Fenchel's duality theory that minimizing the generalized MSPBE is equivalent to solving a \emph{nonconvex-strongly-concave} (NCSC) minimax optimization problem, where the primal side of the objective is nonconvex and the dual side is strongly concave \citep{wai2019variance}. Generalizing the SAGA method \citep{defazio2014saga} for nonconvex minimization problem, the recent work \citep{wai2019variance} proposed a variance-reduced stochastic gradient algorithm to solve this NCSC minimax problem. However, since their primary objective is a \emph{finite-sum} of errors from a large batch of data points, the corresponding algorithm can only run \emph{offline}. 

In the current paper, we aim at improving upon the previous algorithm by proposing and analyzing \emph{single-timescale} stochastic algorithms for the \emph{online} setting with streaming data points. Formally, the NCSC minimax optimization under the online setting can be formulated as
\begin{align}
\min_{\theta\in \Theta} \max_{\omega\in \Omega}\{F(\theta, \omega) := \EE_{\xi\sim \cD}[ f(\theta, \omega; \xi)] \}, \label{eq:ncsc_intro}
\end{align}
where $F(\theta, \omega)$ is nonconvex w.r.t. $\theta$ when fixing $\omega$ and strongly concave w.r.t. $\omega$ when fixing $\theta$. Here $\xi$ is a random variable following the data distribution $\cD$. Each individual function $f(\theta, \omega; \xi)$ is a continuously differentiable function. For this problem, two-timescale algorithms refer to the subclass of algorithms updating $\theta$ and $\omega$ with significantly different frequencies or step sizes (see Section \ref{sec:timescale} in the supplementary material for detailed discussion). The work \citep{bhatnagar2009convergent} suggested a two-timescale stochastic algorithm for solving MSPBE with only asymptotic analysis via theory of ordinary differential equation (ODE). To generally solve the problem \eqref{eq:ncsc_intro}, \citet{lin2019gradient} proposed two-timescale algorithms with the faster timescale updated by sampling large batches of data or iterating until converging to small error before updating the slower timescale. \citet{yansharp} proposed a double-loop algorithm, though single-timescale, requiring to reset the parameters for the inner loop after a certain number of iterations. Both \citet{lin2019gradient} and \citet{yansharp} provide a non-asymptotic convergence analysis with $\tilde{O}(\varepsilon^{-4})$ sample complexity attaining $\varepsilon$ convergence error. To further obtain a faster convergence, the recent work \citep{luo2020stochastic} proposed a double-loop algorithm achieving an $O(\varepsilon^{-3})$ sample complexity\footnote{The sample complexity in the aforementioned works are associated with finding the $\varepsilon$-stationary point of the primal side of the problem \eqref{eq:ncsc_intro}, e.g. the $\varepsilon$-stationary point of the function $J(\theta):= \max_{\omega \in \Omega} F(\theta, \omega)$ in \citet{lin2019gradient, luo2020stochastic}. As shown in Section \ref{sec:single_time_theory}, the sample complexity in our results is associated with a tighter convergence metric proposed in this paper to measure the convergence of both primal and dual sides of $F(\theta, \omega)$.} based on the variance reduction technique, which runs certain number of iterates for the dual side before updating the primal side. In addition, the algorithm requires a large batch of data for construct variance reduction checkpoint each round. However, in practice, a \emph{single-timescale} and  \emph{single-loop} stochastic algorithm is preferred due to its simplicity and also because its step-size is easier to tune. Therefore, our paper aims at answering the following two questions: 

\vspace{5pt}

\emph{Is it possible to design a single-timescale and  single-loop stochastic algorithm for the problem \eqref{eq:ncsc_intro}? Moreover, can we further accelerate this algorithm to obtain a lower sample complexity?}

\vspace{5pt}

The main challenge lies in the asymmetry in the primal and dual sides of the objective, which potentially affects the updating rule design on both sides. Our contributions are three folds:

\begin{itemize}[leftmargin=*]
\setlength\itemsep{0em}
    \item First, we develop a \emph{single-timescale} and \emph{single-loop} stochastic gradient algorithm with only \emph{one} data point each step, which implements momentum updates on both primal and dual variables. We prove that this algorithm can achieve $O(\varepsilon^{-4})$ sample complexity. This result provides an insight that the momentum updates can lead to single-timescale algorithms for minimax optimization, shedding light on further developing algorithms in this area.
    
    \item Moreover, our second algorithm improves upon the first one by applying variance reduction on top of momentum. This algorithm achieves an $O(\varepsilon^{-3})$ sample complexity matching the best known result. Meanwhile, this accelerated algorithm still remains single-timescale and single-loop, and does \emph{not} require a large-batch checkpoint each round. It also extends the recent proposed momentum variance-reduction algorithm in \citet{cutkosky2019momentum} from nonconvex minimization to minimax optimization.

    \item Third, the convergence of our algorithms is expressed in a tighter form of simultaneous primal and dual-side convergence, which extends existing convergence metrics under the online setting with only showing the convergence of the primal iterates. Moreover, our paper studies a general setting where both primal and dual-side feasible sets, i.e., $\Theta$ and $\Omega$, can be any convex closed set without requiring $\Theta$ being compact and $\Omega = \RR^n$ as in the nonlinear TD learning problem .
\end{itemize}

The detailed comparison of our results in this paper to existing works is listed in Table \ref{tab:comparison}.

\vspace{5pt}
\noindent\textbf{Related Work.} There have been a large number of existing works focusing on the linear function approximation for the PE task \citep{tsitsiklis1997analysis,sutton2009fast,sutton2009convergent,liu2015finite,du2017stochastic,wang2017finite,dalal2017finite,yu2017convergence,touati2017convergent,bhandari2018finite,doan2019finite,wang2019multistep,srikant2019finite,xu2020reanalysis}. This line of works can enjoy a benign property of the linear function approximation that the objectives of the associated minimax optimization are convex-strongly-concave or even strongly-convex-strongly-concave \citep{liu2015finite,du2017stochastic}. This property can potentially result in fast convergence or even linear convergence rate for the offline setting, e.g. \citet{du2017stochastic}. On the other hand, TD learning with nonlinear smooth function approximation is studied in several works. \citet{bhatnagar2009convergent} proposed several two-timescale stochastic algorithms with asymptotic convergence analysis by ODE. \citet{chung2018two} studied the nonlinear approximation with neural network proposing two-timescale algorithm also with asymptotic convergence analysis. Recently, \citet{wai2019variance} proposed a variance-reduced algorithm to solve the finite-sum objective in the offline setting with non-asymptotic convergence analysis.

From the perspective of optimization theory, there have been several works studying the NCSC minimax problem. The works \citep{nouiehed2019solving,thekumparampil2019efficient,lu2019hybrid,lin2019gradient} studied the deterministic algorithms for the NCSC minimax problem. \citet{lin2019gradient} further studied the sample complexity of two-timescale stochastic algorithms where the faster timescale (dual side) is performed by sampling a large batch of data points or iterating to a tiny convergence error before updating the primal variable. \citet{rafique2018non} proposed double-loop proximally guided stochastic subgradient methods for solving a class of nonconvex-concave minimax optimization problem with special structures. \citet{yansharp} presented a single-timescale and double-loop algorithm for solving NCSC minimax optimization, where the inner loop is restarted to reset parameters after a certain number of iterations. There is another line of works \citep{xu2019finite,cai2019neural} investigating PE with non-smooth approximation by ReLU networks, which is not the same setting as in this paper.

\vspace{5pt}
\noindent\textbf{Notation.}  For column vectors $x$ and $y$, we denote their concatenation $[x^\top, y^\top]^\top$ by $(x,y)$. For a function $g(x,y)$, we let $g(x, \cdot)$ denote the function w.r.t. the second argument and let $g(\cdot, y)$ denote the function w.r.t. the first argument. We use $\nabla_x g(x,y)$ and $\nabla_y g(x,y)$ to denote its gradients w.r.t. $x$ and $y$ respectively and further let $\nabla g(x,y)$ denote their concatenation $\big(\nabla_x g(x,y), \nabla_y g(x,y)\big)$.  We let $\|\cdot\|$ denote the $\ell_2$ norm for vectors. The projection to a set $\mathcal{C}$ is defined as $\mathcal{P}_{\mathcal{C}}(x) = \argmin_{\tilde{x}\in \mathcal{C}} \|\tilde{x}-x\|^2$. 

\begin{table}[t] \renewcommand{\arraystretch}{1.3}
\caption{Comparison with results from existing works for solving NCSC minimax problem. The column `Variance Reduction' exhibits whether an algorithm employs variance reduction techniques. In the `Convergence Metric' column, `Primal' indicates the sample complexity is evaluated based on the convergence of the primal variable while `Primal \& Dual' means the sample complexity is measured in terms of the convergence of both primal and dual variables. (See Section \ref{sec:single_time_theory} for detailed discussions.) The column `Large Batch' shows whether an algorithm requires sampling a large batch of data points each round. }
\vspace{0.2cm}
\centering
\begin{tabular}{ | >{\centering\arraybackslash}m{3cm} | >{\centering\arraybackslash}m{2.0cm} | >{\centering\arraybackslash}m{1.8cm} | >{\centering\arraybackslash}m{2.5cm} | >{\centering\arraybackslash}m{3.0cm} | >{\centering\arraybackslash}m{1.5cm} | } 
 \hline
 Paper & Sample Complexity & Variance Reduction &  Convergence Metric & Single-Timescale and Single-Loop & Large Batch \\ \hline
\citet{lin2019gradient}   & $O(\varepsilon^{-4})$ & No & Primal & No  &  Yes \\ \hline
 \citet{yansharp} & $\tilde{O}(\varepsilon^{-4})$ & No & Primal & No  &  No\\ \hline
 \citet{luo2020stochastic} & $O(\varepsilon^{-3})$ & Yes & Primal & No  &  Yes \\  \hline
\multirow{2}{*}{  This work } & $O(\varepsilon^{-4})$ & No & \multirow{2}{*}{Primal \& Dual }  & \multirow{2}{*}{  Yes } &  \multirow{2}{*}{No} \\
& $O(\varepsilon^{-3})$ & Yes &  &   & \\ \hline 
\end{tabular}
\label{tab:comparison}
\end{table}
\setlength{\textfloatsep}{0.5cm}

\section{Policy Evaluation with Smooth Function Approximation}
\vspace{-0.12cm}

Consider a Markov Decision Process (MDP) described by $(\mathcal{S}, \mathcal{A}, \mathcal{P}_{s,s'}^a, \mathcal{R}, \rho)$. We denote by $\cS$ the set of states, and denote by $\cA$ the set of actions. Let $\mathcal{P}_{s,s'}^a$ denote the transition probability from state $s\in\cS$ to state $s'\in\cS$ with action $a\in\cA$. Note that $\cS$ and $\cA$ can be infinite such that $\mathcal{P}_{s,s'}^a$ becomes a Markov kernel. Let $r(s,a,s')$ be an immediate reward once an agent takes action $a$ at state $s$ and transits to state $s'$. The reward function $\cR(s, a)$ is then defined as $\cR(s, a) := \EE_{s'\sim \mathcal{P}_{s,\cdot}^a} [r(s,a,s')]$. The discount factor is denoted by $\rho \in [0, 1)$. Let $\pi(a\given s)$ be the policy which is the probability of taking action $a$ given current state $s$. Then, we have the state value function defined as $V^\pi(s):=\EE[\sum_{t=0}^{\infty} \rho^t \cR(s_t,a_t)\given s_0 = s, \pi]$. Further letting $R^\pi(s) := \EE_{a\sim \pi(\cdot|s)}[\cR(s,a)]$ and $P^\pi (s,s') := \EE_{a\sim\pi(\cdot\given s)} [\cP^a_{s,s'}]$, we define the Bellman operator as $T^\pi V(s)  := R^\pi(s) + \rho \cdot \EE_{s'\sim P^\pi(s,\cdot)}[ V(s')]$. Then, $V^\pi$ satisfies the Bellman equation in the form of $
V^\pi(s) = T^\pi V^\pi(s), \forall s \in \cS$.
 
We consider a TD learning problem which focuses on solving the Bellman equation for $V^\pi$. One typical approach to solve Bellman equation is to approximate $V^\pi$ by a parameterized function $V_\theta$ with parameters $\theta\in\Theta \subset \RR^d$. The feasible set $\Theta$ can be a compact set to guarantee the learning parameter $\theta$ not drifting far from its initialization. In the next section, we show that this assumption is only a special case of our general theory. Letting $V_\theta$ be a smooth nonlinear function \citep{bhatnagar2009convergent}, according to \citet{liu2015finite,wai2019variance}, we can solve Bellman equation by minimizing a generalized MSPBE 
\vspace{-0.1cm}
\begin{align}
\textsc{Mspbe}(\theta) = \frac{1}{2} \big\|\EE_{s\sim d^\pi(\cdot)}\{[T^\pi V_\theta(s)-V_\theta(s)] \nabla_\theta V_\theta(s) ]^\top\}\big\|^2_{K_\theta^{-1}},\label{eq:mspbe_ori} 
\end{align}
where $K_\theta = \EE_s[  \nabla_\theta V_\theta(s) \nabla_\theta V_\theta(s)^\top ]\in \RR^{d\times d}$. Here $d^\pi(\cdot)$ denotes stationary distribution of states. In this paper, we assume that $K_\theta$ is non-singular for all $\theta \in \Theta$ such that its smallest eigenvalue $\lambda_{\min}(K_\theta)>0$. Via the Fenchel’s duality that $1/2\cdot \|x\|^2_{A^{-1}} = \max_{y \in \RR^d} \langle x, y \rangle -1/2 \cdot y^\top A y$, we thus have a primal-dual minimax formulation of MSPBE minimization problem as
\begin{align} \label{eq:pd_mspbe}
\hspace*{-0.1cm}\min_{\theta \in \Theta}~\textsc{Mspbe}(\theta) = \min_{\theta\in\Theta} \max_{\omega\in \RR^d} \Big\{ \cL(\theta, \omega) := \EE_{s,a,s'}[\ell(\theta, \omega; s,a,s')] \Big\},    
\end{align}
where $\EE_{s,a,s'}$ is taking expectation for $s\sim d^\pi(\cdot),a\sim \pi(\cdot\given s),s'\sim \cP_{s,\cdot}^a$, and we define 
\vspace{-0.1cm}
\begin{align*}
\ell(\theta, \omega; s,a,s') := \langle \delta \cdot \nabla_\theta V_\theta(s), ~\omega \rangle -\frac{1}{2} \omega^\top [\nabla_\theta V_\theta(s) \nabla_\theta V_\theta(s)^\top] \omega,
\end{align*}
with $\delta := r(s,a,s') + \rho V_\theta(s') - V_\theta(s)$. Once fixing $\theta$, due to non-singularity of $K_\theta$, then $\cL(\theta, \cdot)$ is a strongly concave (quadratic) function, while if fixing $\omega$, then $\cL(\cdot, \omega)$ is a nonconvex function. Thus, we have an NCSC minimax formulation for policy evaluation with smooth function approximation. \citet{bhatnagar2009convergent} suggested a stochastic algorithm that once given an online data point $(s_t, a_t, s_{t+1})$, this algorithm updates $\theta$ and $\omega$ by
\vspace{-0.1cm}
\begin{align*}
    &\theta_{t+1} = \cP_\Theta (\theta_t - \nu_t \nabla_\theta \ell(\theta_t, \omega_t; s_t, a_t, s_{t+1})), ~~\text{and}~~ 
   \omega_{t+1} = \omega_t + \mu_t \nabla_\omega \ell(\theta_t, \omega_t; s_t, a_t, s_{t+1}),
\end{align*}
where $\nu_t/\mu_t \rightarrow 0$ as $t \rightarrow \infty$. In essence, this is a two-timescale stochastic gradient descent-ascent algorithm to update primal variable $\theta$ and dual variable $\omega$ alternately.

\section{Problem Formulation}

In theory, we consider to solve a more general NCSC minimax optimization problem under the online setting, formulated as \eqref{eq:ncsc_intro}, i.e., $\min_{\theta\in \Theta} \max_{\omega\in \Omega} \{F(\theta, \omega) := \EE_{\xi\sim \mathcal{D}} [f(\theta, \omega; \xi)] \}$.

For each time $t$, one data point is observed $\xi_t \sim \mathcal{D}$ with a function $f(\cdot, \cdot;~\xi_t)$ as well as its first-order derivative. In particular, $(s,a,s')$ in \eqref{eq:pd_mspbe} is equivalent to $\xi$ here. Each time $t$, the data point $\xi_{t+1}$ is $(s_t, a_t, s_{t+1})$.
Specifically, the generality of the problem \eqref{eq:ncsc_intro} we studied here is reflected in the following two aspects: First, the feasible sets $\Theta\subseteq \RR^d$ and $\Omega\subseteq \RR^n$ are only convex and closed without requiring boundedness of $\Theta$. Thus, we can set $\Theta = \RR^d$ or $\Omega = \RR^n$ such that \eqref{eq:ncsc_intro} is reduced to one-sided or two-sided unconstrained minimax optimization problems. Therefore, TD learning with smooth function approximation in \eqref{eq:pd_mspbe}, where $\Theta \subseteq \RR^d$ is compact and $\Omega \subset \RR^n$ with $ n = d$, becomes a special case of the problem \eqref{eq:ncsc_intro} with $f(\theta, \omega; \xi)$ being $\ell(\theta, \omega;s,a,s')$. Second, the function $F(\theta, \omega)$ in \eqref{eq:ncsc_intro} is a general nonconvex strongly concave function, and 
not limited to the form of $\cL(\theta, \omega)$ in \eqref{eq:pd_mspbe} whose dual side is quadratic.

Furthermore, we define a function $J(\theta)$ by $J(\theta) := \max_{\omega\in \Omega} F(\theta, \omega)$, which implies that the problem \eqref{eq:ncsc_intro} can be equivalently written as
\begin{align*}
\min_{\theta\in \Theta} \max_{\omega\in \Omega} F(\theta, \omega) = \min_{\theta\in \Theta} J(\theta).
\end{align*}
We can observe that $\textsc{Mspbe}(\theta)$ in \eqref{eq:mspbe_ori} is equivalent to $J(\theta)$ with setting $\Omega = \RR^n$. In addition, $F(\theta, \omega)$ is strongly concave w.r.t. $\omega \in \Omega$ which guarantees the existence and uniqueness of the solution to the problem $\max_{\omega\in \Omega} F(\theta, \omega), \forall \theta\in \Theta$. Then, given $\theta\in \Theta$, we define the solution as
\begin{align*}
\omega^*(\theta) :=\argmax_{\omega \in \Omega} F(\theta, \omega),
\end{align*}
which is a mapping from $\Theta$ to $\Omega$. Thus, $J(\theta)$ can be further written as $J(\theta) = F(\theta, \omega^*(\theta))$.

Due to the nonconvexity of the primal side and the strong concavity of the dual side, our goal is to design efficient algorithms so that the primal iterate $\theta_t$ converges to a stationary point or local minimizer of the function $F(\cdot, \omega_t)$ while the dual side $\omega_t$ converges to $\omega^*(\theta_t)$.

\begin{remark}
Note that since we consider a general constrained problem, the mapping  $\omega^*(\theta)$ may not be the solution to the unconstrained maximization problem $\argmax_{\omega \in \RR^n} F(\theta, \omega)$ if $\Omega \neq \RR^n$. Thus, we may have $\nabla_\omega F(\theta, \omega^*(\theta)) \neq 0$. On the other hand, the iterate $\theta_t$ may also only converge to a point $\widehat{\theta}$ on the boundary of $\Theta$ with $\nabla_\theta F(\widehat{\theta}, \omega) \neq 0$ if this point is a local minimum but not stationary point. This motivates us to find a proper metric to measure the convergence of algorithms to solve \eqref{eq:ncsc_intro}, which is considered as one of our main contribution. The detailed discussion of the convergence metric is presented in Section \ref{sec:single_time_theory}.
\end{remark}

\vspace{-0.2cm}
\section{Single-Timescale Stochastic Algorithm for NCSC Optimization}
\vspace{-0.1cm}
%In this section, we first introduce the single-timescale single-loop algorithm for stochastic NCSC minimax problem shown in Algorithm \ref{alg:single_time}. Then, we provide a non-asymptotic convergence analysis.  

%\subsection{Algorithm}
The single-timescale and  single-loop stochastic algorithm for solving \eqref{eq:ncsc_intro} is introduced in Algorithm \ref{alg:single_time}. Specially, the step size associated with time $t$ is $\nu_t$, which is applied to update the parameter $\theta$ and $\omega$ in the same timescale of $O(t^{-1/2})$. At time $t$, viewing $p_t$ and $d_t$ as stochastic gradient approximation, Lines 3 and 4 perform stochastic gradient descent for $\theta$ and ascent for $\omega$ with parameters $\gamma$ and $\eta$, and then project the iterates. Then, averaging steps are taken between the projected iterates $\tilde{\theta}_{t+1}$, $\tilde{\omega}_{t+1}$ and the previous iterates $\theta_t$, $\omega_t$ to get $\theta_{t+1}$, $\omega_{t+1}$, which is 
\begin{align*}
&\theta_{t+1} = \theta_t + \nu_t(\tilde{\theta}_{t+1} - \theta_t) = (1-\nu_t)\theta_t + \nu_t \tilde{\theta}_{t+1}, \\    
&\omega_{t+1} =\omega_t + \nu_t(\tilde{\omega}_{t+1} - \omega_t) = (1-\nu_t)\omega_t + \nu_t \tilde{\omega}_{t+1}.   
\end{align*}
Here $\theta_{t+1}$, $\omega_{t+1}$ are guaranteed to stay in $\Theta$ and $\Omega$ by a simple induction proof if $\theta_t \in \Theta$ and $\omega_t \in \Omega$ when $\nu_t\leq 1$. In particular, we initialize $\theta_0 \in \Theta$ and $\omega_0 \in \Omega$. Note that the parameters $\gamma$ and $\eta$ are two \emph{constants} (not related to time $t$) to balance the updates of $\theta$ and $\omega$. For a more clear understanding of the updates in Lines 3 and 4, setting $\Theta = \RR^d$ and $\Omega = \RR^n$ yields
\begin{align*}
\theta_{t+1} = \theta_t - \gamma \nu_t p_t,\quad \text{ and }~~  \omega_{t+1} = \omega_t + \eta \nu_t d_t,
\end{align*}
which are stochastic gradient descent-ascent steps for unconstrained problems with gradient approximation by $p_t$ and $d_t$. Similar updating rules as Lines 3 and 4 can also be found in existing papers for constrained minimization problems, e.g. \citet{ruszczynski1987linearization, ghadimi2018single}.

Lines 5 and 6 are the \emph{momentum} updates for primal and dual sides. Specifically, letting $q_t = (p_t, d_t)$, $g_{t+1} = \nabla f(\theta_{t+1}, \omega_{t+1}; \xi_{t+1})$, and $c_t = \alpha \nu_t = \beta \nu_t$ with assuming $\alpha = \beta$, Lines 5 and 6 can be interpreted as
\begin{align} \label{eq:momentum_grad}
q_{t+1} = (1-c_t) q_t + c_t g_{t+1}.   
\end{align}
In particular, this update only requires \emph{one} date point $\xi_{t+1}$ each time. This momentum step shows that $q_{t+1}$ is a recursive average of historical stochastic gradients. The intuition behind the application of momentum updates is: 
%As shown in \cite{lin2019gradient}, a deterministic full gradient method for NCSC problem can lead to a single-timescale algorithm. 
%Intuitively, 
the momentum updates can use history averaging to counteract
the effect of noise of each stochastic gradient, which potentially results in a single-timescale algorithm.

\begin{algorithm}[t]\caption{Single-Timescale Stochastic Gradient Algorithm for NCSC Minimax Optimization} 
    \setstretch{0.8}
	\begin{algorithmic}[1]
		\State {\bfseries Initialize:} $\theta_0 \in \Theta$, $\omega_0\in \Omega$, $p_0 = \nabla_\theta f(\theta_0, \omega_0, \xi_0) $, $d_0 = \nabla_\omega f(\theta_0, \omega_0, \xi_0) $.
		\For{$t=0,\ldots,T-1$}
		        \State Update primal variable $\theta_{t+1}$:
		        \begin{align*}
		        &\tilde{\theta}_{t+1} =  \mathcal{P}_\Theta (\theta_t - \gamma p_t ) , \\
		        &\theta_{t+1} = \theta_t + \nu_t(\tilde{\theta}_{t+1} - \theta_t).
	            \end{align*}
                \State Update dual variable $\omega_{t+1}$:
		        \begin{align*}
		        &\tilde{\omega}_{t+1} = \mathcal{P}_\Omega (\omega_t + \eta d_t) , \\
		        &\omega_{t+1} = \omega_t +  \nu_t (\tilde{\omega}_{t+1} - \omega_t).
	            \end{align*}
	            \State Update primal stochastic gradient $p_{t+1}$ :
		        \begin{align*}
		        & p_{t+1} = (1-\alpha \nu_t) p_t + \alpha \nu_t  \nabla_\theta f(\theta_{t+1},\omega_{t+1}; \xi_{t+1}).
		        \end{align*}
                \State Update dual stochastic gradient $d_{t+1}$ :
		        \begin{align*}
		        & d_{t+1} = (1-\beta\nu_t) d_t + \beta\nu_t  \nabla_\omega f(\theta_{t+1},\omega_{t+1}; \xi_{t+1}).
		        \end{align*}
        \EndFor
	\end{algorithmic}\label{alg:single_time}
\end{algorithm}

\subsection{Theoretical Results} \label{sec:single_time_theory}

\textbf{Assumptions.} We first make several standard assumptions which are the same as or \emph{even weaker} than the ones in recent papers, e.g. \citet{bhatnagar2009convergent, wai2019variance,lin2019gradient}.  

\begin{assumption}[Existence of Solution]\label{assump:exist_sol} There exists at least one global minimizer $\theta^* \in \Theta$ such that $ J(\theta) \geq J^* > -\infty, \forall \theta \in \Theta$, where we denote $J^* = J(\theta^*)$. 
\end{assumption}

\begin{assumption}[Convex Sets] \label{assump:fea_set}
The feasible sets $\Theta$ and $\Omega$ are closed convex sets.  
\end{assumption}

\begin{assumption}[Lipschitz Smoothness] \label{assump:lip_grad}
For any $\theta, \theta' \in \Theta$ and any $\omega, \omega' \in \Omega$, the gradient $\nabla F(\theta, \omega) = \big( \nabla_\theta F(\theta, \omega), \nabla_\omega F(\theta, \omega) \big)$ satisfies $\|\nabla F(\theta, \omega) - \nabla F(\theta', \omega')\| \leq L_F \|(\theta, \omega) - (\theta', \omega')\|$.
\end{assumption}
Assumption \ref{assump:lip_grad} further implies that both $F(\theta, \cdot)$ and $F(\cdot, \omega)$ are $L_F$-Lipschitz smooth. As shown in \citet{wai2019variance}, this assumption is satisfied for nonlinear TD learning in \eqref{eq:pd_mspbe}.

\begin{assumption}[Strong Concavity] \label{assump:non_singular} For any given $\theta \in \Theta$, the function $F(\theta, \cdot)$ is $\mu$-strongly concave, i.e., $\forall \theta\in \Theta$ and $\forall \omega, \omega' \in \Omega$, $F(\theta, \cdot)$ is concave and $\|\nabla_\omega F(\theta, \omega) - \nabla_\omega F(\theta, \omega')\| \geq \mu \|\omega-\omega'\|$.
\end{assumption}

For the TD learning problem \eqref{eq:pd_mspbe}, this assumption is equivalent to that for $ \forall \theta \in \Theta$, $  \mu = \lambda_{\min}(K_\theta) >0$.

As shown in \citet{lin2019gradient}, $J(\theta) = \max_{\omega\in\Omega} F(\theta, \omega)$ enjoys a benign property of being gradient Lipschitz when both Assumptions \ref{assump:lip_grad} and \ref{assump:non_singular} hold, which is applied in our theoretical analysis.

\begin{assumption}[Bounded Variance] \label{assump:bounded_var}
The variance of the stochastic gradient $\nabla f(\theta, \omega, \xi) = \big(\nabla_\theta f(\theta, \omega, \xi),\nabla_\omega f(\theta, \omega, \xi) \big)$ is bounded as $\EE_{\xi\sim \mathcal{D}} \|\nabla f(\theta, \omega, \xi) - \nabla F(\theta, \omega)\|^2 \leq \sigma^2$.
\end{assumption}
Assumption \ref{assump:bounded_var} further implies $\nabla_\theta f(\theta, \omega, \xi)$ and $\nabla_\omega f(\theta, \omega, \xi)$ have bounded variance respectively, i.e., $\EE_{\xi} \|\nabla_\theta f(\theta, \omega, \xi) - \nabla_\theta F(\theta, \omega)\|^2 \leq \sigma^2$ as well as $\EE_{\xi} \|\nabla_\omega f(\theta, \omega, \xi) - \nabla_\omega F(\theta, \omega)\|^2 \leq \sigma^2$.

\begin{remark}
In our paper, we only need a weaker assumption as Assumption \ref{assump:fea_set} without assuming boundedness of $\Theta$ and $\Omega$. Some papers, e.g. \citet{wai2019variance}, explicitly assume that all the iterates $\{\omega_t\}_{t\geq 0}$ are bounded even when $\Omega = \RR^n$ to theoretically derive the convergence guarantee while our analysis does not require this assumption.
\end{remark}

\noindent\textbf{Convergence Metric.} We propose the following metric to measure the convergence of algorithms 
\begin{align}
\begin{aligned}\label{eq:converg_mea}
    \mathfrak{M}_t &:= \gamma^{-1}\|\tilde{\theta}_{t+1} - \theta_t\| + \|\nabla_\theta F(\theta_t,\omega_t)- p_t\|  + L_F\|\omega_t-\omega^*(\theta_t)\|. 
\end{aligned}
\end{align}
The first two terms of RHS in \eqref{eq:converg_mea} measures the convergence of the primal side $\{\theta_t\}_{t\geq 0}$. If a point $\widehat{\theta}\in \Theta$ is a local minimum or a stationary point for the function $F(\cdot, \omega)$, then there must be $\widehat{\theta} = \mathcal{P}_\Theta(\widehat{\theta}-\gamma \nabla_\theta F(\widehat{\theta}, \omega))$. This indicates either $\nabla_\theta F(\widehat{\theta}, \omega) = 0$ with $\widehat{\theta}$ being a stationary point or $\nabla_\theta F(\widehat{\theta}, \omega) \neq 0$ but $\widehat{\theta}$ a local minimizer on the boundary of $\Theta$ so that the projected gradient descent step at $\widehat{\theta}$ returns to itself. Thus, it inspires us to use $\gamma^{-1} \|\theta - \mathcal{P}_\Theta(\theta-\gamma p)\| + \|p-\nabla_\theta F(\theta, \omega)\|$ to measure the convergence $\theta$ such that these two terms being $0$ implies $p=\nabla_\theta F(\theta, \omega)$ and $\theta = \mathcal{P}_\Theta(\theta-\gamma p) = \mathcal{P}_\Theta(\theta-\gamma \nabla_\theta F(\theta, \omega))$. The first two terms of RHS in \eqref{eq:converg_mea} is also used in existing works, e.g. \citet{ghadimi2018single}, to measure the convergence of algorithms for constrained minimization problem. The last term in \eqref{eq:converg_mea} measures the convergence of $\omega_t$ to the unique maximizer $\omega^*(\theta_t)$ for $F(\theta_t, \cdot)$. When the primal side is unconstrained, i.e. $\Theta = \RR^d$, we have 
\begin{align}
\mathfrak{M}_t = \|p_t\|  + \|\nabla_\theta F(\theta_t,\omega_t)- p_t\|   + L_F\|\omega_t-\omega^*(\theta_t)\|  \geq\|\nabla_\theta F(\theta_t,\omega_t)\| +  L_F\|\omega_t-\omega^*(\theta_t)\|, \label{eq:metric_unconstr_primal}
\end{align} 
which implies that if $\mathfrak{M}_t \rightarrow 0$, then $\|\nabla_\theta F(\theta_t, \omega_t)\|\rightarrow 0$ and $\omega_t \rightarrow  \omega^*(\theta_t)$. If the problem \eqref{eq:ncsc_intro} is further  unconstrained on both primal and dual sides, i.e. $\Theta = \RR^d$ and $\Omega = \RR^n$, we have 
\begin{align*}
\mathfrak{M}_t = \|p_t\|  + \|\nabla_\theta F(\theta_t,\omega_t)- p_t\|   + L_F\|\omega_t-\omega^*(\theta_t)\|  \geq\|\nabla_\theta F(\theta_t,\omega_t)\| +\|\nabla_\omega F(\theta_t,\omega_t)\|,
\end{align*} 
where the last inequality uses the Lipschitz continuity of gradients and fact that $\nabla_\omega F(\theta_t, \omega^*(\theta_t)) = 0$ in unconstrained case. This means when $\mathfrak{M}_t \rightarrow 0$, then both $\|\nabla_\theta F(\theta_t, \omega_t)\|\rightarrow 0$ and $\|\nabla_\omega F(\theta_t, \omega_t)\|\rightarrow 0$. We also see that the Lipschitz constant $L_F$ exists to balance the scale of the norm of gradients and the norm of variables.

\begin{remark}[Comparisons of Metrics] When $\Omega = \RR^n$, the convergence metric \eqref{eq:converg_mea} is similar to the one used in the paper \citet{wai2019variance} for the finite-sum setting. But the metric in \citet{wai2019variance} employs $\|\nabla_\omega F(\theta_t, \omega_t)\|$ to evaluate the dual-side convergence, which is not as general as \eqref{eq:converg_mea} for the case $\Omega \neq \RR^n$ . On the other hand, in most existing papers studying online settings, their metrics aim to measure the convergence of the primal variable, e.g. $\|\nabla J(\theta_t)\| \rightarrow 0$ in \citet{lin2019gradient, luo2020stochastic}. This is equivalent to showing $\|\nabla_\theta F(\theta_t, \omega^*(\theta_t))\| \rightarrow 0$ when $\Theta = \RR^d$, which ignores showing the convergence of dual side $\omega_t \rightarrow \omega^*(\theta_t)$. Our convergence metric is tighter in that it measures both primal and dual-side convergence for any convex closed feasible sets $\Theta$ and $\Omega$. For the special case where  $\Theta = \RR^d$, the convergence measured by $\mathfrak{M}_t$ can imply the convergence measured by $\|\nabla J(\theta_t)\|$ according to the following inequality
\begin{align*}
\|\nabla J(\theta_t)\| = \|\nabla_\theta F(\theta_t, \omega^*(\theta_t))\| \leq \|\nabla_\theta F(\theta_t, \omega^*(\theta_t)) - \nabla_\theta F(\theta_t, \omega_t)\| + \|\nabla_\theta F(\theta_t, \omega_t)\| \leq \mathfrak{M}_t
\end{align*}
where the last inequality is by \eqref{eq:metric_unconstr_primal} and the Lipschitz continuity of the gradient $\nabla_\theta F(\theta, \omega)$.
\end{remark}

For stochastic algorithms in this paper, we adopt the average of expectation, i.e., $T^{-1} \sum_{t=0}^{T-1} \EE [\mathfrak{M}_t]$, to measure their ergodic convergence, where the expectation is taken over all sampling randomness $\{ \xi_{t}\}_{t= 0}^{T-1}$. This is the standard way to show the convergence of stochastic algorithms.

\vspace{5pt}
\noindent\textbf{Convergence Analysis.} Based on this convergence metric, we show our theoretical results below.  Specifically, we present two theorems to show the sample complexity under Algorithm \ref{alg:single_time} with either decaying step size or fixed step size.

\begin{theorem}[Decaying Step Size] \label{thm:single_time} Under Assumptions \ref{assump:exist_sol}, \ref{assump:fea_set}, \ref{assump:lip_grad} and \ref{assump:non_singular}, setting the parameters $\alpha = \beta = 3$, $0< \eta \leq \mu /(4 L_F^2)$, $0 < \gamma \leq \eta\mu^2 /(9 L_F^2)$, and $\nu_t = 1/[16(t+b)]^{1/2}$ with $b \geq  \max \{ (2\gamma L_F^2/\mu)^2, 3 \}$, with the updating rules in Algorithm \ref{alg:single_time}, the convergence rate is\footnote{We use $\tilde{O}$ to hide logarithmic factors. More specifically, $\tilde{O}$ hides $\log T$ factors for the convergence results and $\log (\varepsilon^{-1})$ factors for the results of sample complexity.}
%\begin{align*}
%\frac{1}{T} \sum_{t=0}^{T-1} \EE[\mathfrak{M}_t] \leq \frac{C_1}{T^{1/4}} + \frac{C_2}{\sqrt{T}},
%\end{align*}
%where $C_1 = [768 Q_0/\gamma +  216 \sigma^2/(\mu\eta)]^{1/2}$, $C_2 = C_1 b^{1/4}$, and $Q_0 = J(\theta_0) - J^* +  10L_F^2\gamma/(\mu\eta) \|\omega_0 - \omega^*(\theta_0)\|^2 + 2\gamma/(\mu\eta) (\|\nabla_\omega F(\theta_0,\omega_0)-d_0\|^2 + \|\nabla_\theta F(\theta_0,\omega_0)-p_0\|^2).$
\begin{align*}
\frac{1}{T} \sum_{t=0}^{T-1} \EE[\mathfrak{M}_t] \leq   \tilde{O} \left ( \frac{1}{T^{1/4}} \right),
\end{align*}
Then, the sample complexity $T_\varepsilon$ to achieve $\varepsilon$ error of $T^{-1} \sum_{t=0}^{T-1} \EE[\mathfrak{M}_t]$ is $T_\varepsilon \geq \tilde{O}(\varepsilon^{-4})$.
\end{theorem}

In Theorem \ref{thm:single_time}, we set $\nu_t$ decaying so as to clearly show its dependence on time $t$. This theorem shows that if we set $\nu_t = O(t^{-1/2})$, it requires $\tilde{O}(\varepsilon^{-4})$ rounds (or number of data points) to attain an $\varepsilon$ error of $T^{-1} \sum_{t=0}^{T-1} \EE[\mathfrak{M}_t]$. Note that in this theorem, it is not necessary to fix the total number of rounds $T$ in advance. However, it introduces an extra logarithmic factor, namely $\log (\varepsilon^{-1})$, in the sample complexity. In the next theorem, we show that this logarithmic factor can be removed via setting a fixed step size if assuming the total number of rounds $T$ is pre-set.

\begin{theorem}[Fixed Step Size]\label{thm:single_time_fix}  Under Assumptions \ref{assump:exist_sol}, \ref{assump:fea_set}, \ref{assump:lip_grad} and \ref{assump:non_singular}, setting the parameters $\alpha = \beta = 3$, $0< \eta \leq \mu /(4 L_F^2)$, $0 < \gamma \leq \eta\mu^2 /(9 L_F^2)$, and fixing the step size $\nu_t = \nu = 1/[16(T+b)^{1/2}]$ with $b \geq  \max \{ (\gamma L_F^2/\mu)^2, 3 \}$ for all $t\in [0, T]$, with the updating rules in Algorithm \ref{alg:single_time}, the convergence rate of this algorithm is
\begin{align*}
\frac{1}{T} \sum_{t=0}^{T-1} \EE[\mathfrak{M}_t] \leq   O \left ( \frac{1}{T^{1/4}} \right),
\end{align*}
Then, the sample complexity $T_\varepsilon$ to achieve $\varepsilon$ error of $T^{-1} \sum_{t=0}^{T-1} \EE[\mathfrak{M}_t]$ is $T_\varepsilon \geq O(\varepsilon^{-4})$.
\end{theorem}

\begin{remark}[Comparisons] The sample complexity of Algorithm \ref{alg:single_time}, i.e. $\tilde{O}(\varepsilon^{-4})$ for the decaying step size and $O(\varepsilon^{-4})$ for the fixed step size, matches the results of the two-timescale or double-loop algorithms \citep{lin2019gradient, yansharp} without using variance reduction techniques. Moreover, our sample complexity is based on the convergence of both primal and dual variables while the existing algorithms  \citep{lin2019gradient, yansharp} only measure the complexity in terms of the convergence of the primal variable.
\end{remark}

\section{Accelerated Single-Timescale Stochastic Method for NCSC Optimization }

%In this section, we introduce the accelerated variant of our single-timescale stochastic algorithm for NCSC minimax optimization shown in Algorithm \ref{alg:variance_reduce} with theoretical convergence guarantee.

%\subsection{Accelerated Algorithm}

In this section, we propose an accelerated variant of the single-timescale stochastic algorithm in a single-loop form for NCSC minimax optimization as  summarized in Algorithm \ref{alg:variance_reduce}. The updates of the parameters $\theta$ and $\omega$ in Line 3 and Line 4 follow a similar rule as the one in Algorithm \ref{alg:single_time}. The main modification lies in the updates of the gradient approximation terms $p$ and $d$ in Line 5 and Line 6. If we adopt similar notations to \eqref{eq:momentum_grad} and further let $\overline{g}_t = \nabla f(\theta_t, \omega_t; \xi_{t+1})$ and $c_t = \alpha \nu_t^2  = \beta \nu_t^2$ with assuming $\alpha = \beta$, then Lines 5 and 6 can be interpreted as
\begin{align}
q_{t+1} = (1-c_t) q_t + c_t g_{t+1} - (1-c_t) (\overline{g}_t - g_{t+1}).   \label{eq:vr_momentum_grad}
\end{align}
Inspired by the recently proposed momentum variance reduction technique \citep{cutkosky2019momentum} for minimization problems, we adapt it to the NCSC minimax optimization problem by applying variance reduction upon the primal-dual momentum updates. The first two terms of RHS in \eqref{eq:vr_momentum_grad} recover the gradient updates in \eqref{eq:momentum_grad} (if ignoring the difference of $c_t$) while the third term of RHS operates as a bias correctness to $q_t$. In contrast to \eqref{eq:momentum_grad}, the update \eqref{eq:vr_momentum_grad} results in that $q_{t+1}$ is an \emph{unbiased} stochastic approximation with $\EE [q_{t+1}] = \nabla F(\theta_{t+1}, \omega_{t+1})$ by induction if $\EE [q_t] = \nabla F(\theta_t, \omega_t)$, which can be guaranteed by our initialization step $q_0 = (p_0, d_0) = \nabla f(\theta_0, \omega_0; \xi_0)$ in Algorithm \ref{alg:variance_reduce}. With the bias correctness to averaging of stochastic gradients, the gradient approximation error can  be further reduced by proper setting of $c_t$.
% As demonstrated later in our proofs, the momentum variance reduction leads to a smaller variance such that Algorithm \ref{alg:variance_reduce} converges faster than Algorithm \ref{alg:single_time}.

Algorithm \ref{alg:variance_reduce} does not require computing an averaged variance reduction checkpoint gradient with sampling a large batch of data after a certain number of updates. In contrast, it only needs \emph{one} data point each step, e.g., $\xi_{t+1}$ at the $t$-th round. Then, its associated stochastic gradients are evaluated with two sets of parameters, e.g., the previous one $(\theta_t, \omega_t)$ and the current one $(\theta_{t+1}, \omega_{t+1})$. Algorithm \ref{alg:variance_reduce} remains a single-timescale  and single-loop algorithm as both 
$\theta$ and $\omega$ are updated in the same timescale.% and so are $p$ and $d$.

\begin{algorithm}[t]\caption{Variance-Reduced Single-Timescale Stochastic Algorithm for NCSC Optimization} 
    \setstretch{0.7}
	\begin{algorithmic}[1]
		\State {\bfseries Initialize:} $\theta_0\in \Theta$, $\omega_0\in \Omega$, $p_0 = \nabla_\theta f(\theta_0, \omega_0, \xi_0) $, and $d_0 = \nabla_\omega f(\theta_0, \omega_0, \xi_0) $, $\alpha = \beta = 6$.
		\For{$t=1,\ldots,T$}
		        \State Update primal variable $\theta_{t+1}$:
		        \begin{align*}
		        &\tilde{\theta}_{t+1} =  \mathcal{P}_\Theta (\theta_t - \gamma p_t ), \\
		        &\theta_{t+1} = \theta_t + \nu_t(\tilde{\theta}_{t+1} - \theta_t).
	            \end{align*}
                \State Update dual variable $\omega_{t+1}$:
		        \begin{align*}
		        &\tilde{\omega}_{t+1} = \mathcal{P}_\Omega (\omega_t + \eta d_t), \\
		        &\omega_{t+1} = \omega_t +  \nu_t (\tilde{\omega}_{t+1} - \omega_t).
	            \end{align*}
	            \State Update primal stochastic gradient $p_{t+1}$ :
		        \begin{align*}
		        & p_{t+1} = (1- \alpha \nu^2_t) (p_t-\nabla_\theta f(\theta_t,\omega_t; \xi_{t+1}))  +  \nabla_\theta f(\theta_{t+1},\omega_{t+1}; \xi_{t+1}).
		        \end{align*}
                \State Update dual stochastic gradient $d_{t+1}$ :
		        \begin{align*}
		        & d_{t+1} = (1-\beta \nu^2_t) (d_t-\nabla_\omega f(\theta_t,\omega_t; \xi_{t+1}))  +  \nabla_\omega f(\theta_{t+1},\omega_{t+1}; \xi_{t+1}).
		        \end{align*}
        \EndFor
	\end{algorithmic}\label{alg:variance_reduce}
\end{algorithm}

\subsection{Theoretical Results}

\textbf{Assumptions.} In this section, we make the same assumptions as Assumptions \ref{assump:fea_set}, \ref{assump:non_singular}, and \ref{assump:bounded_var}. In addition, we make an assumption of Lipschitz stochastic gradients instead of Assumption \ref{assump:lip_grad}.  

\begin{assumption}[Lipschitz Stochastic Gradient]\label{assump:lip_sto_grad}
We assume that $\forall \theta, \theta' \in \Theta$ and $\forall \omega, \omega' \in \Omega$, the stochastic gradient $\nabla f(\theta, \omega) = \big( \nabla_\theta f(\theta, \omega), \nabla_\omega f(\theta, \omega) \big)$ satisfies $\|\nabla f(\theta, \omega;\xi) - \nabla f(\theta', \omega';\xi)\| \leq L_f \| (\theta,\omega)-(\theta', \omega') \|$, which further leads to $\|\nabla F(\theta, \omega) - \nabla F(\theta', \omega')\| \leq L_f \| (\theta,\omega)-(\theta', \omega') \|$.
\end{assumption}
The second inequality in Assumption \ref{assump:lip_sto_grad} can be obtained by simply applying Jensen's inequality $\|\EE X\|^2 \leq \EE \|X\|^2$. Assumption \ref{assump:lip_sto_grad} also shows that the Lipschitz continuity of stochastic gradients can imply Lipschitz continuity of population gradients. It is a common assumption in variance-reduced stochastic algorithms, e.g. \citet{johnson2013accelerating, cutkosky2019momentum}. Besides, as shown in \citet{wai2019variance}, this assumption can also be verified for $\ell(\theta, \omega; s,a,s')$ in \eqref{eq:pd_mspbe}. 

\vspace{5pt}
\noindent\textbf{Convergence Analysis.} We adopt the same convergence measure \eqref{eq:converg_mea} here by replacing $L_F$ with $L_f$. Analogous to the last section, we present two theorems to show the sample complexity under Algorithm \ref{alg:variance_reduce} with either decaying or fixed step size.

\begin{theorem}\label{thm:variance_reduce}  Under Assumptions \ref{assump:exist_sol}, \ref{assump:fea_set}, \ref{assump:non_singular}, and \ref{assump:lip_sto_grad}, setting the parameters $\alpha = \beta = 6$, $0< \eta \leq \mu /(6 L_f^2)$, $0 < \gamma \leq  \mu^2 \eta /(9  L_f^2)$, and $\nu_t = 1/[3(t+b)^{1/3}]$ with $b \geq  \max \{ (6\gamma L_f^2/\mu)^3, 256 \}$, with the updating rules in Algorithm \ref{alg:variance_reduce}, the convergence rate is 
%\begin{align*}
%\frac{1}{T} \sum_{t=0}^{T-1} \EE [\mathfrak{M}_t] \leq \frac{C_1}{T^{1/3}} + \frac{C_2}{\sqrt{T}},
%\end{align*}
%where $C_1 = [150 S_0/\gamma +  1536 \sigma^2/(\mu\eta)]^{1/2}$, $C_2 = C_1 b^{1/6}$, and $S_0 = J(\theta_0) - J^* + 10L_f^2 \gamma/(\mu\eta) \|\omega_0 - \omega^*(\theta_0)\|^2 + 2\gamma/(\mu\eta\nu_0) (\|\nabla_\omega F(\theta_0,\omega_0)-d_0\|^2 + \|\nabla_\theta F(\theta_0,\omega_0)-p_0\|^2).$
\begin{align*}
\frac{1}{T} \sum_{t=0}^{T-1} \EE [\mathfrak{M}_t] \leq \tilde{O}\left( \frac{1}{T^{1/3}} \right).
\end{align*}
Then, the sample complexity $T_\varepsilon$ to achieve $\varepsilon$ error of $T^{-1} \sum_{t=0}^{T-1} \EE [\mathfrak{M}_t]$ is $T_\varepsilon \geq \tilde{O}(\varepsilon^{-3})$.
\end{theorem}

Theorem \ref{thm:variance_reduce} shows that if we set $\nu_t = O(t^{-1/3})$, it requires $\tilde{O}(\varepsilon^{-3})$ rounds (or data points) to attain an $\varepsilon$ error of $T^{-1} \sum_{t=0}^{T-1} \EE[\mathfrak{M}_t]$. Note that in this theorem, it is not necessary to fix the total number of rounds $T$ in advance. The next theorem can remove the logarithmic factor $\log (\varepsilon^{-1})$ in the sample complexity via setting a fixed step size when the total number of rounds $T$ is known.

\begin{theorem} \label{thm:variance_reduce_fix} Under Assumptions \ref{assump:exist_sol}, \ref{assump:fea_set}, \ref{assump:non_singular}, and \ref{assump:lip_sto_grad}, setting the parameters $\alpha = \beta = 6$, $0< \eta \leq \mu /(6 L_f^2)$, $0 < \gamma \leq  \mu^2 \eta /(9  L_f^2)$, and fixing the step size $\nu_t = \nu = 1/[3(T+b)^{1/3}]$ with $b \geq  \max \{ [6\gamma (L_f + L_f^2/\mu)]^3, 256 \}$ for all $t\in [0, T]$, with the updating rules in Algorithm \ref{alg:variance_reduce}, the convergence rate is
\begin{align*}
\frac{1}{T} \sum_{t=0}^{T-1} \EE[\mathfrak{M}_t] \leq   O \left ( \frac{1}{T^{1/3}} \right),
\end{align*}
Then, the sample complexity $T_\varepsilon$ to achieve $\varepsilon$ error of $T^{-1} \sum_{t=0}^{T-1} \EE[\mathfrak{M}_t]$ is $T_\varepsilon \geq O(\varepsilon^{-3})$.
\end{theorem}

\begin{remark}[Comparisions] Algorithm \ref{alg:variance_reduce} achieves the sample complexity of the order $\tilde{O}(\varepsilon^{-3})$ for the decaying step size and $O(\varepsilon^{-3})$ for the fixed step size, matching the best known sample complexity for the NCSC minimax problem in the recent work \citet{luo2020stochastic}. However, the algorithm in \citet{luo2020stochastic} is double-loop and requires a large batch of data to construct a variance reduction checkpoint. In contrast, Algorithm \ref{alg:variance_reduce} is single-timescale and  single-loop with only sampling one data point each round.  
\end{remark}

%\begin{remark}\label{re:fixed} In our theorems, we set $\nu_t$ decaying so as to clearly show its dependence on time $t$. Following our analysis in the next section, we can remove logarithmic factors in Theorems \ref{thm:single_time} and \ref{thm:variance_reduce} by setting step sizes fixed as $\nu_t = O(T^{-\frac{1}{2}})$ and $O(T^{-\frac{1}{3}}), \forall t \leq T,$ respectively, assuming that $T$ is known. Then, our results become $O(\varepsilon^{-4})$ and $O(\varepsilon^{-3})$. We omit these results due to space limitation. 
%\end{remark}

%As an initial attempt, our work sheds lights on designing efficient single-timescale algorithms for NCSC minimax problem in an online setting. As our future work, we will further investigate the convergence of our proposed algorithms with non-i.i.d. sample (e.g. Markovian sample).

\section{Proof Outlines}

In this section, we provide proof outlines for Theorem \ref{thm:single_time}, Theorem \ref{thm:single_time_fix}, Theorem \ref{thm:variance_reduce}, and Theorem \ref{thm:variance_reduce_fix}. In our proofs, we set the parameters $\alpha = \beta$ for simplicity, which is also the setting in the main theorems. The proofs can be extended to the case where $\alpha \neq \beta$ without much effort. For compact notations, we define the error between $(p_t, d_t)$ and the gradient $\nabla F(\theta, \omega)$ as follows
\begin{align*}
\Delta_t^p :=  p_t-\nabla_\theta F(\theta_t, \omega_t), ~\text{ and }~ \Delta_t^d :=  d_t-\nabla_\omega F(\theta_t, \omega_t). 
\end{align*}
We also denote $\overline{\mathfrak{M}}_t :=\gamma^{-2}\|\tilde{\theta}_{t+1} - \theta_t\|^2 + \|\nabla_\theta F(\theta_t,\omega_t)- p_t\|^2 + L_F^2\|\omega_t-\omega^*(\theta_t)\|^2$. Thus, by Jensen's inequality, there is $T^{-1} \sum_{t=0}^{T-1}\EE[\mathfrak{M}_t] \leq \big( 3 T^{-1} \sum_{t=0}^{T-1}  \EE [\overline{\mathfrak{M}}_t]\big)^{1/2}$.

\subsection{Proof Outline for Theorem \ref{thm:single_time} and Theorem \ref{thm:single_time_fix}}

\begin{proof}[Proof Sketch of Theorem \ref{thm:single_time} and Theorem \ref{thm:single_time_fix}]

With taking expectation over the randomness $\{\xi_t\}_{t=0}^{T-1}$, we start the proof by Lemma \ref{lem:primal_J_bound}, which gives the bound for $\EE\| \tilde{\theta}_{t+1} - \theta_t\|^2$ as follows
\begin{align}
\hspace*{-0.18cm} 3\nu_t/(4\gamma)\cdot \EE  \|\tilde{\theta}_{t+1} - \theta_t \|^2 \leq 2L_F^2 \gamma \nu_t \EE \|\omega_t - \omega(\theta_t)^*\|^2  + 4\gamma\nu_t \EE \|\Delta_t^p\|^2 +\EE [J(\theta_t)-J(\theta_{t+1})].\label{eq:ref_1}
\end{align}
To guarantee the convergence, we still need to understand the upper bounds of the remaining terms, namely, $\EE \|\omega_t - \omega^*(\theta_t)\|^2$ and $\EE\|\Delta^p_t\|^2$. Then, by Lemma \ref{lem:decomp_omega_opt}, we obtain
\begin{align}
\begin{aligned}\label{eq:ref_2}
\EE\|\omega_{t+1} - \omega^*(\theta_{t+1})\|^2 &\leq (1-\mu\eta\nu_t/4) \EE \|\omega_t - \omega^*(\theta_t)\|^2 - 3\nu_t/4\cdot\EE\|\tilde{\omega}_{t+1} - \omega_t \|^2  \\
&\quad + 75L_\omega^2 \nu_t /(16\mu\eta)\cdot \EE \|\tilde{\theta}_{t+1}-\theta_t\|^2  + 75\eta \nu_t/(16\mu)\cdot \EE \|\Delta_t^d\|^2,
\end{aligned}
\end{align}
which shows a contraction of the term $\EE \|\omega_{t+1} - \omega^*(\theta_{t+1})\|^2$ plus some noise terms as well as a deduction of the term $\EE\|\tilde{\omega}_{t+1} - \omega_t\|^2$ which will be eliminated in the end. However, this inequality introduces an extra $\EE \|\Delta^d_t\|^2$ term. Thus, it remains to explore the bounds for $\EE\|\Delta^p_t\|^2$ and $\EE \|\Delta^d_t\|^2$. By Lemma \ref{lem:grad_var_bound}, we obtain the contraction of the terms $\EE\|\Delta^p_{t+1}\|^2$ and $\EE \|\Delta^d_{t+1}\|^2$ plus gradient variance terms $\sigma^2$ and some other noise terms, which are
\begin{align*}
&\EE \|\Delta_{t+1}^p\|^2 \leq  (1- \alpha\nu_t)  \EE \|\Delta_t^p\|^2 +  \alpha^2\nu_t^2 \sigma^2 +  9\nu_t L_F^2/(8\alpha) \cdot \EE (\|\tilde{\theta}_{t+1} - \theta_t \|^2 +\|\tilde{\omega}_{t+1}-\omega_t\|^2), \\ %\label{eq:grad_var_bound_1}
&\EE \|\Delta_{t+1}^d\|^2 \leq (1-\alpha\nu_t)  \EE \|\Delta_t^d\|^2 + \alpha^2 \nu_t^2 \sigma^2 + 9\nu_t L_F^2/(8\alpha)\cdot \EE (\|\tilde{\theta}_{t+1} - \theta_t \|^2 +\|\tilde{\omega}_{t+1}-\omega_t\|^2).%\label{eq:grad_var_bound_2}
\end{align*}    
The above two inequalities also control the difference between the stochastic gradient approximation and the true gradients. Next, we define the Lyapunov function as $Q_t := J(\theta_t) - J^*  +  \tfrac{10L_F^2 \gamma}{\mu\eta} \|\omega_t - \omega^*(\theta_t)\|^2 + \tfrac{2\gamma}{\mu\eta} (\|\Delta^p_t\|^2 + \|\Delta^d_t\|^2)$, where $Q_t \geq 0$ since $J(\theta_t) \geq J^*$. Properly scaling the above inequalities, summing them together, and rearranging the terms, we have 
\begin{align}
&\nu_T  \EE[\overline{\mathfrak{M}}_t] \leq \nu_t  \EE[\overline{\mathfrak{M}}_t] \leq 576  \sigma^2 \nu_t^2/(\mu\eta) + 16/\gamma \cdot \EE [Q_t - Q_{t+1}], \label{eq:single_time_al}
\end{align}
where we set the parameters as in Theorem \ref{thm:single_time} to control $\gamma$, $\eta$, and $\nu_t$ in a proper small scale to guarantee all the terms in $\overline{\mathfrak{M}}_t$ to be positive, and also use $\nu_T \leq \nu_t, \forall t \leq T$ since $\{\nu_t\}_{t\geq 0}$ is non-increasing. 

Note that the above analysis holds for non-increasing step size, i.e., $\nu_t$ can be either decaying or constant, under certain conditions. Thus, one can make use of \eqref{eq:single_time_al} to derive the results for both Theorem \ref{thm:single_time} and Theorem \ref{thm:single_time_fix}. Multiplying both sides by $1/\nu_T$ and taking average on both sides from $t=0$ to $T-1$, we have $T^{-1} \sum_{t=0}^{T-1}\EE [\mathfrak{M}_t] \leq ( 3T^{-1} \sum_{t=0}^{T-1}  \EE [\overline{\mathfrak{M}}_t] )^{1/2}  \leq \tilde{O}(T^{-1/4})$ with setting the decaying step size as $\nu_t = O(t^{-1/2})$, which proves Theorem \ref{thm:single_time}. On the other hand, if we adopt the fixed step size as $\nu_t = \nu = O(T^{-1/2})$, we have $T^{-1} \sum_{t=0}^{T-1}\EE [\mathfrak{M}_t] \leq ( 3T^{-1} \sum_{t=0}^{T-1}  \EE [\overline{\mathfrak{M}}_t] )^{1/2}  \leq O(T^{-1/4})$, which proves Theorem \ref{thm:single_time_fix}. Please refer to Section \ref{sec:proof_single_time} in the supplemental material for detailed proofs.
\end{proof}

\vspace{-0.35cm}
\subsection{Proof Outline for Theorem \ref{thm:variance_reduce} and Theorem \ref{thm:variance_reduce_fix}}

\vspace{-0.05cm}

\begin{proof}[Proof Sketch of Theorem \ref{thm:variance_reduce} and Theorem \ref{thm:variance_reduce_fix}]
The proof of Theorem \ref{thm:variance_reduce} also applies the same bounds as \eqref{eq:ref_1} and \eqref{eq:ref_2} by replacing the Lipschitz constant $L_F$ with $L_f$ due to Assumption \ref{assump:lip_sto_grad}. The main difference lies in the bounds for the terms $\EE\|\Delta^p_t\|^2$ and $\EE\|\Delta^d_t\|^2$. By applying Lemma \ref{lem:vr_grad_var_bound}, we have 
\begin{align}
\begin{aligned}
&\nu_t^{-1}\EE \|\Delta^p_{t+1}\|^2 -\nu_{t-1}^{-1}\EE \|\Delta^p_t\|^2 \\
&\qquad \leq  -\tfrac{21}{4}  \nu_t  \EE \|\Delta^p_t\|^2 + O(\nu_t) \cdot \EE(\|\tilde{\theta}_{t+1} - \theta_t  \|^2 + \|\tilde{\omega}_{t+1}-\omega_t  \|^2) + O(\nu_t^3 \sigma^2), \label{eq:del_p}
\end{aligned}
\end{align}
\vspace*{-0.2cm}
\begin{align}
\begin{aligned}
&\nu_t^{-1}\EE \|\Delta^d_{t+1}\|^2 -\nu_{t-1}^{-1}\EE \|\Delta^d_t\|^2 \\
&\qquad \leq  -\tfrac{21}{4}  \nu_t  \EE \|\Delta^d_t\|^2 + O(\nu_t) \cdot \EE(\|\tilde{\theta}_{t+1} - \theta_t  \|^2 + \|\tilde{\omega}_{t+1}-\omega_t  \|^2) + O(\nu_t^3 \sigma^2). \label{eq:del_d}
\end{aligned}
\end{align}
%We also have the same bound for $ \EE \|\Delta^d_t\|^2$ by replacing $ \EE\|\Delta^p_{t+1}\|^2$ and $ \EE\|\Delta^p_t\|^2$ with $ \EE\|\Delta^d_{t+1}\|^2$ and $ \EE\|\Delta^d_t\|^2$ in \eqref{eq:del_p}.
%\begin{align}
%&\hspace*{-0.2cm}\nu_t^{-1}\EE \|\Delta^d_{t+1}\|^2 -\nu_{t-1}^{-1}\EE \|\Delta^d_t\|^2 \leq  -\tfrac{21}{4}  \nu_t \EE \|\Delta^d_t\|^2 \label{eq:del_d}\\
% &\hspace*{-0.2cm}+ O(\nu_t) \cdot \EE(\|\tilde{\theta}_{t+1} - \theta_t  \|^2 + \|\tilde{\omega}_{t+1}-\omega_t  \|^2) + O(\nu_t^3 \sigma^2).\nonumber       
%\end{align}
We define the Lyapunov function as $S_t := J(\theta_t) - J^* +  \tfrac{10L_f^2 \gamma}{\mu\eta} \|\omega_t - \omega^*(\theta_t)\|^2 + \tfrac{2\gamma}{\mu\eta\nu_{t-1}} (\|\Delta^p_t\|^2+\|\Delta^d_t\|^2)$, where $S_t \geq 0$. Properly scaling the above results and \eqref{eq:ref_1}, \eqref{eq:ref_2},  summing them together, and rearranging the terms, due to $\nu_T \leq \nu_t$, we have 
\begin{align}
\nu_T\EE[\overline{\mathfrak{M}}_t] \leq    \nu_t\EE[\overline{\mathfrak{M}}_t] \leq  16\times 288\sigma^2\nu_t^3/(\mu\eta) + 16/\gamma \cdot \EE[S_t-S_{t+1}], \label{eq:variance_al}
\end{align}
by setting the parameters $\gamma$, $\eta$ and $\nu_t$ in a proper small scale. 

The above analysis holds for non-decreasing $\nu_t$, where $\nu_t$ can be decaying or constant. Therefore, we employ \eqref{eq:variance_al} to complete the proofs of both Theorem \ref{thm:variance_reduce} and Theorem \ref{thm:variance_reduce_fix}. Further multiplying both sides of \eqref{eq:variance_al} by $1/\nu_T$ and taking average on both sides from $t=0$ to $T-1$, we have $T^{-1} \sum_{t=0}^{T-1}\EE [\mathfrak{M}_t] \leq ( 3T^{-1} \sum_{t=0}^{T-1}  \EE [\overline{\mathfrak{M}}_t] )^{1/2} \leq \tilde{O}(T^{-1/3})$ with setting the decaying step size as $\nu_t = O(t^{-1/3})$, which gives the proof of Theorem \ref{thm:variance_reduce}. If we set the step size fixed as $\nu_t = \nu = O(T^{-1/3})$, we have $T^{-1} \sum_{t=0}^{T-1}\EE [\mathfrak{M}_t] \leq ( 3T^{-1} \sum_{t=0}^{T-1}  \EE [\overline{\mathfrak{M}}_t] )^{1/2} \leq O(T^{-1/3})$, which finishes the proof of Theorem \ref{thm:variance_reduce_fix}. Please refer to Section \ref{sec:proof_variance_reduce} in the supplemental material for detailed proofs.
%Note that dividing $\sigma^2 \nu_t^3 $ by $\nu_t$ in \eqref{eq:variance_al},  the variance dependency is $\sigma^2 \nu_t^3 / \nu_t = \sigma^2 \nu_t^2 = O(\sigma^2 t^{-2/3})$ with setting $\nu_t = O(t^{-1/3})$, which is in a lower order than $\sigma^2 \nu_t^2/\nu_t = \sigma^2 \nu_t = O(\sigma^2 t^{-1/2})$ in \eqref{eq:single_time_al} with setting $\nu_t = O(t^{-1/2})$. This indicates the power of variance reduction upon the momentum for realizing faster convergence.  
\end{proof}

\section{Conclusions and Future Works}
In this paper, we develop two single-timescale stochastic gradient algorithms with provable approximation guarantees for stochastic NCSC minimax optimization inspired by nonlinear TD learning problem. The first algorithm implements a momentum update which can achieve an $O(\varepsilon^{-4})$ sample complexity reaching an $\varepsilon$ approximation, matching existing results without using variance reduction techniques. Our second algorithm improves upon the first one by further applying variance reduction on top of momentum, matching the best known $O(\varepsilon^{-3})$ sample complexity. As an initial attempt, our work sheds lights on designing efficient single-timescale algorithms for NCSC minimax problem in an online setting. As our future work, we will further investigate the convergence of our proposed algorithms with non-i.i.d. sample (e.g. Markovian sample).

\bibliographystyle{ims}
\bibliography{rl_ref}

\input{supplement.tex}

\end{document}

%% file: supplement.tex
\onecolumn
\vspace{1em}
\centerline{ {\LARGE Supplementary Material} }
\renewcommand{\thesection}{\Alph{section}}
\setcounter{section}{0}

\vspace{1em}

The supplementary material is organized as follows: \textbf{(1)} Section \ref{sec:supp_pre} provides preliminaries for our proofs. \textbf{(2)} Section \ref{sec:proof_single_time} provides proofs for Theorem \ref{thm:single_time} and Theorem \ref{thm:single_time_fix} as well as the associated lemmas. \textbf{(3)} Section \ref{sec:proof_variance_reduce} provides proofs for Theorem \ref{thm:variance_reduce}, Theorem \ref{thm:variance_reduce_fix}, and the associated lemmas. 
%\textbf{(4)} In Section \ref{sec:proof_fixed_rate}, we present additional proofs for the discussion in Remark \ref{re:fixed}, i.e., the cases that Algorithms \ref{alg:single_time} and \ref{alg:variance_reduce} use fixed step sizes.

\section{Preliminaries} \label{sec:supp_pre}
In this section, we first give a detailed explanation on the two-timescale and single-timescale algorithms. Then, we provide several supporting lemmas already given in previous papers that is helpful for the proofs of our main theorems.

\subsection{Two-Timescale and Single-Timescale Algorithms} \label{sec:timescale}

\textbf{Two-timescale algorithm.} In this paper, two-timescale algorithms refer to the subclass of algorithms updating $\theta$ and $\omega$ with significantly different frequencies or learning rates. More specifically, consider updating rules of a primal-dual algorithm in the following form
\begin{align*}
&\theta_{t+1} = \theta_t - \varrho_t A_t, \\
&\omega_{t+1} = \omega_t + \varsigma_t B_t,
\end{align*}
where $\varrho_t  \geq 0$ and $\varsigma_t \geq 0$ are learning rates, and $A_t, B_t$ could be either stochastic gradients or full gradients together. Then, this algorithm is a two-timescale algorithm if $\varrho_t / \varsigma_t \rightarrow 0$ or $\varrho_t /  \varsigma_t   \rightarrow +\infty$ as $ t \rightarrow \infty$, which means $\varrho_t, \varsigma_t$ have distinct dependency on the time $t$. In addition, there are also stochastic primal-dual algorithms that update the variable on the faster timescale by sampling a large batch of data points or by iterating to a tiny convergence error before a one-step update for the other variable on the slower timescale. These algorithms implicitly perform two-timescale updates by imbalanced sampling.

\vspace{5pt}
\noindent \textbf{Single-timescale algorithm.} In this paper, we consider an algorithm as a single-timescale algorithm if the algorithm update $\theta$ and $\omega$ with one stochastic gradient or full gradient each time and also with learning rates for primal sides and dual sides having the same orders on the time $t$, i.e.,
\begin{align*}
&\theta_{t+1} = \theta_t - \varrho_t A_t, \\
&\omega_{t+1} = \omega_t + \varsigma_t B_t,
\end{align*}
where $\varrho_t  \geq 0$ and $\varsigma_t \geq 0$ are learning rates and $0 < C \leq  \varrho_t / \varsigma_t \leq C'$ as $t\rightarrow +\infty$ \citep{dalal2017finite}. Here $0< C, C' < +\infty$ are two positive constants.  $A_t$ and $B_t$ could be both stochastic gradients or both full gradients. In our proposed algorithms, under the unconstrained setting, $\varrho_t / \varsigma_t =  (\gamma \nu_t) / (\eta \nu_t) = \gamma / \eta$ with $0 < \gamma / \eta < +\infty$ where $\gamma, \eta$ are two positive constants.

\subsection{Supporting Lemmas}

Recall that in Algorithm \ref{alg:single_time} and Algorithm \ref{alg:variance_reduce}, the projection operator $\mathcal{P}_\Theta(\theta - \gamma p)$ and $\Pi_{\Omega}(\theta - \eta d)$ are defined as 
\begin{align*}
\mathcal{P}_\Theta(\theta - \gamma p) = \argmin_{\tilde{\theta} \in \Theta} \|\tilde{\theta} -(\theta - \gamma p)\|^2,    
\end{align*}
and
\begin{align*}
\Pi_{\Omega}(\omega - \eta d) = \argmin_{\tilde{\omega} \in \Omega} \|\tilde{\omega} - (\omega - \eta d)\|^2.
\end{align*}
Moreover, recall the definition of the function $J(\theta)$, which is
\begin{align*}
    J(\theta) = \max_{\omega \in \Omega} \{F(\theta, \omega) = \EE_{\xi} [f(\theta, \omega; \xi)] \},
\end{align*}
with defining 
\begin{align*}
\omega^*(\theta):=\argmax_{\omega \in \Omega} F(\theta, \omega).    
\end{align*}
The following lemma show that the gradient of $J(\theta)$ is also Lipschitz continuous under our assumptions.
\begin{lemma}[\citet{lin2019gradient}] \label{lem:grad_lip_J} Under Assumptions \ref{assump:lip_grad}, \ref{assump:non_singular} and \ref{assump:lip_sto_grad}, the gradient of the function $J(\theta) = F(\theta, \omega^*(\theta))$ w.r.t. $\theta$ is Lipschitz continuous, which is
\begin{align*}
\|\nabla J(\theta) - \nabla J(\theta')\| \leq L_J \|\theta-\theta'\|, \quad \forall \theta, \theta' \in \Theta, 
\end{align*}
where the Lipschitz constant is 
\begin{align*}
L_J = L_F + \frac{L^2_F}{\mu} \quad \text{ for Algorithm \ref{alg:single_time}}, 
\end{align*}
or 
\begin{align*}
L_J = L_f + \frac{L^2_f}{\mu} \quad \text{ for Algorithm \ref{alg:variance_reduce}}. 
\end{align*}
\end{lemma}
Furthermore, viewing $\omega^*(\theta)$ as a mapping from the set $\Theta$ to the set $\Omega$, it also has the property of Lipschitz continuity. 
\begin{lemma}[\citet{lin2019gradient}] \label{lem:omega_lip} Under Assumptions \ref{assump:non_singular} and \ref{assump:lip_grad}, the mapping $\omega^*(\theta)=\argmax_{\omega \in \Omega} F(\theta, \omega)$ is Lipschitz continuous, which is
\begin{align*}
\|\omega^*(\theta) - \omega^*(\theta')\| \leq L_\omega \|\theta-\theta'\|, \quad \forall \theta, \theta' \in \Theta,     
\end{align*}
where the Lipschitz constant is 
\begin{align*}
L_\omega = \frac{L_F}{\mu}.
\end{align*}
\end{lemma}
Next, we present the optimality condition via variational inequality for the general constrained convex optimization problem.
\begin{lemma}[\citet{nesterov2018lectures}] \label{lem:var_ineq} Consider the constrained convex minimization problem in the following form
\begin{align*}
    \min_{x\in X} g(x),
\end{align*}
where $g(x)$ is a differentiable convex function and the set $X$ is a convex set. Then, a point $x^* \in X$ is the minimizer of this constrained minimization problem iff 
\begin{align*}
    \langle \nabla g(x^*), x-x^* \rangle \geq 0, \quad \forall x \in X.
\end{align*}
\end{lemma}
Lemma \ref{lem:var_ineq} still holds when the optimization problem is unconstrained. 
In addition, for any gradient $L$-Lipschitz function (could be non-convex), we have the following lemma.
\begin{lemma} \label{lem:lip_ineq} Let $g(x)$ be a non-convex function with $L$-Lipschitz gradient, i.e., 
\begin{align*}
\| \nabla g(x) - \nabla g(y) \| \leq L \|x-y\|, \forall x, y.    
\end{align*}
Then the following inequality holds
\begin{align*}
|g(x) - g(y) - \langle \nabla g(y), x-y \rangle| \leq \frac{L}{2} \|x-y\|^2, \forall x,y.
\end{align*}
    
\end{lemma}
Thus, we can apply this lemma to the functions $F(\theta, \omega)$ and $J(\theta)$. 
Moreover, for any $\mu$-strongly concave function, we have the following lemma.

\begin{lemma} \label{lem:sc_ineq} Let $h(x)$ be a $\mu$-strongly concave function. Then the following inequality holds
\begin{align*}
g(x) - g(y) - \langle \nabla g(y), x-y \rangle \leq -\frac{\mu}{2} \|x-y\|^2, \forall x,y.
\end{align*}
\end{lemma}
Due to the strong concavity of the function $F(\theta, \cdot)$ given any $\theta \in \Theta$ as Assumption \ref{assump:non_singular}, we can apply this lemma to the functions $F(\theta, \cdot)$. Now we are ready to provide the proofs of our main results.

\input{proof.tex}

\input{vr_proof.tex}

%% file: proof.tex
\section{Proofs for Algorithm \ref{alg:single_time}} \label{sec:proof_single_time}

\begin{lemma} \label{lem:primal_J_bound} Under Assumptions \ref{assump:fea_set}, \ref{assump:lip_grad}, and \ref{assump:bounded_var}, letting $0 < \gamma \nu_t \leq \mu/(16L_F^2)$ and  $\nu_t \leq 1 $, with the updating rules shown in Algorithm \ref{alg:single_time}, we have 
\begin{align*}
J(\theta_{t+1})-J(\theta_t)  & \leq -  \frac{3\nu_t}{4\gamma} \|\tilde{\theta}_{t+1} - \theta_t \|^2  + 2L_F^2 \gamma \nu_t \|\omega_t - \omega(\theta_t)^*\|^2 + 4\gamma\nu_t \|\nabla_\theta F(\theta_t, \omega_t) - p_t\|^2,
\end{align*}
where $J(\theta)=\max_{\omega\in \Omega}F(\theta, \omega)$ and $\omega^*(\theta):=\argmax_{\omega \in \Omega} F(\theta, \omega)$.
\end{lemma}

\begin{proof} According to Lemma \ref{lem:grad_lip_J}, we know that the function $J(\theta)$ has Lipschitz continuous gradients, which implies 
\begin{align*}
| J(\theta_{t+1})-J(\theta_t)-\langle \nabla J(\theta_t), \theta_{t+1} - \theta_t\rangle | \leq  \frac{L_{J}}{2} \|\theta_{t+1}-\theta_t\|^2,
\end{align*}
by applying Lemma \ref{lem:lip_ineq} to the function $J(\theta)$. This inequality thus leads to 
\begin{align*}
J(\theta_{t+1})-J(\theta_t) - \langle \nabla J(\theta_t), \theta_{t+1} - \theta_t \rangle \leq \frac{L_{J}}{2} \|\theta_{t+1}-\theta_t\|^2 .
\end{align*}
By plugging the updating rule $\theta_{t+1} = \theta_t + \nu_t (\tilde{\theta}_{t+1} - \theta_t)$ into the above inequality and rearranging the terms, we have
\begin{align}
J(\theta_{t+1})-J(\theta_t) \leq \nu_t \langle  \nabla J(\theta_t), \tilde{\theta}_{t+1} - \theta_t \rangle + \frac{L_J\nu^2_t}{2} \|\tilde{\theta}_{t+1} - \theta_t\|^2. \label{eq:primal_J_bound}
\end{align}
We decompose the term $\nu_t\langle\nabla J(\theta_t), \tilde{\theta}_{t+1} - \theta_t \rangle$ in the following way
\begin{align}
\begin{aligned}  \label{eq:decomp_J_grad} 
\nu_t\langle\nabla J(\theta_t), \tilde{\theta}_{t+1} - \theta_t \rangle &= \nu_t\langle p_t, \tilde{\theta}_{t+1} - \theta_t \rangle  + \nu_t\langle\nabla_\theta F(\theta_t, \omega_t) - p_t, \tilde{\theta}_{t+1} - \theta_t \rangle \\
&\quad + \nu_t\langle\nabla J(\theta_t) - \nabla_\theta F(\theta_t, \omega_t) , \tilde{\theta}_{t+1} - \theta_t \rangle.  
\end{aligned}
\end{align}
Therefore, we need to bound the three terms on the right-hand side of this equation. 

Before presenting upper bounds of these terms, we first show some understandings of projection operator $\mathcal{P}_\Theta(\cdot)$ and the feasibility of the iterates $\{\theta_{t}\}_{t\geq 0}$ generated by Algorithm \ref{alg:single_time}. Since $\tilde{\theta}_{t+1}=\mathcal{P}_\Theta (\theta_t-\gamma p_t)$ which is equivalently the solution to the problem $\min_{\theta \in \Theta} \|\theta - (\theta_t - \gamma p_t)\|^2$, according to the optimality condition for constrained convex optimization in Lemma \ref{lem:var_ineq}, we have 
\begin{align}\label{eq:theta_t_vari_tmp}
\big \langle \gamma p_t +  \tilde{\theta}_{t+1}-\theta_t, \theta - \tilde{\theta}_{t+1} \big \rangle \geq 0, \quad \forall \theta\in \Theta.
\end{align}
On the other hand, we can prove by induction that $\{\theta_{t}\}_{t \geq 0}\in \Theta$. The updating rule of $\theta$ can be rewritten as
\begin{align*}
\theta_t = \theta_{t-1} + \nu_{t-1} (\tilde{\theta}_t-\theta_{t-1}) = (1-\nu_{t-1}) \theta_{t-1} + \nu_{t-1} \tilde{\theta}_t,    
\end{align*}
which is an average of $\theta_{t-1}$ and $\tilde{\theta}_t$ with the condition of this lemma that $\nu_{t} \leq 1$. Since $\tilde{\theta}_t$ is a projected iterates on the convex set $\Theta$,  if there is $\theta_{t-1} \in \Theta$, we will have $\theta_t \in \Theta$. By induction, this can be guaranteed if we initialize $\theta_0 \in \Theta$ in Algorithm \ref{alg:single_time}.

Thus, we can set $\theta = \theta_t$ in \eqref{eq:theta_t_vari_tmp} and obtain 
\begin{align}\label{eq:theta_t_vari}
\langle \gamma p_t +  \tilde{\theta}_{t+1}-\theta_t, \theta_t - \tilde{\theta}_{t+1} \rangle \geq 0.
\end{align} 
Then, we will present the bounds for the terms in \eqref{eq:decomp_J_grad}. With rearranging the terms of \eqref{eq:theta_t_vari}, we can bound the first term $\nu_t\langle p_t, \tilde{\theta}_{t+1} - \theta_t \rangle$ in \eqref{eq:decomp_J_grad} as
\begin{align}
\nu_t\langle p_t, \tilde{\theta}_{t+1} - \theta_t \rangle   \leq -\frac{\nu_t}{\gamma} \|\tilde{\theta}_{t+1} - \theta_t \|^2. \label{eq:decomp_J_1}
\end{align}
For the last term in \eqref{eq:decomp_J_grad}, i.e., $\nu_t\langle\nabla_\theta F(\theta_t, \omega_t) - p_t, \tilde{\theta}_{t+1} - \theta_t \rangle$, we can bound it as
\begin{align}
\begin{aligned}\label{eq:decomp_J_3}
&\nu_t\langle\nabla_\theta F(\theta_t, \omega_t) - p_t, \tilde{\theta}_{t+1} - \theta_t \rangle  \\
&\qquad \leq \nu_t\|\nabla_\theta F(\theta_t, \omega_t) - p_t\|\cdot \|\tilde{\theta}_{t+1} - \theta_t \|   \\
&\qquad \leq  4\nu_t \gamma \|\nabla_\theta F(\theta_t, \omega_t) - p_t\|^2 + \frac{\nu_t}{16\gamma} \|\tilde{\theta}_{t+1} - \theta_t \|^2,
\end{aligned}
\end{align}
where the first inequality is due to Cauchy-Schwarz inequality and the second inequality is by Young's inequality $-\langle x,y\rangle \leq \lambda/2 \cdot \|x\|^2 +  (2\lambda)^{-1} \|y\|^2$ with $\lambda = 8\gamma$.

In the rest proof of this lemma, we use $\omega^*_t$ to denote $\omega^*(\theta_t)$. Then, for the term $\nu_t\langle\nabla J(\theta_t) - \nabla F(\theta_t, \omega_t) , \tilde{\theta}_{t+1} - \theta_t \rangle$ in \eqref{eq:decomp_J_grad}, due to $\nabla J(\theta_t) = \nabla_\theta F(\theta_t, \omega^*_t)$, we have
\begin{align}
\begin{aligned}\label{eq:decomp_J_2}
&\nu_t\langle\nabla J(\theta_t) - \nabla F(\theta_t, \omega_t) , \tilde{\theta}_{t+1} - \theta_t \rangle \\
&\qquad  \leq \nu_t \|\nabla_\theta F(\theta_t, \omega_t^*) - \nabla_\theta F(\theta_t, \omega_t)\| \cdot \| \tilde{\theta}_{t+1} - \theta_t \| \\
&\qquad \leq \nu_t L_F\|\omega_t - \omega^*_t\| \cdot \| \tilde{\theta}_{t+1} - \theta_t \| \\
&\qquad \leq 2 L_F^2 \gamma \nu_t \|\omega_t - \omega^*_t\|^2 + \frac{\nu_t}{8\gamma}\| \tilde{\theta}_{t+1} - \theta_t \|^2,
\end{aligned}
\end{align}
where the first inequality is by Cauchy-Schwarz inequality, the second inequality is due to Lipschitz gradient of the function $F(\theta, \omega)$ (Assumption \ref{assump:lip_grad}), i.e.,
\begin{align*}
\|\nabla_\theta F(\theta_t, \omega^*_t) - \nabla_\theta F(\theta_t, \omega_t)\| \leq \|\nabla  F(\theta_t, \omega^*_t) - \nabla  F(\theta_t, \omega_t)\| \leq L_F \|\omega^*_t - \omega_t\|,    
\end{align*}
and the third inequality is by Young's inequality $-\langle x,y\rangle \leq \lambda/2 \cdot \|x\|^2 +  (2\lambda)^{-1} \|y\|^2$ with $\lambda = 4\gamma$.

Thus, substituting \eqref{eq:decomp_J_1}, \eqref{eq:decomp_J_2}, and \eqref{eq:decomp_J_3} into \eqref{eq:decomp_J_grad}, we obtain
\begin{align}
\begin{aligned} \label{eq:decomp_J_end}
\nu_t\langle\nabla J(\theta_{t+1}), \tilde{\theta}_{t+1} - \theta_t \rangle &\leq -\Big( \frac{\nu_t}{\gamma}-\frac{\nu_t}{8\gamma}-\frac{\nu_t}{16\gamma} \Big) \|\tilde{\theta}_{t+1} - \theta_t \|^2 \\
&\quad +  2 L_F^2 \gamma \nu_t  \|\omega_t - \omega_t^*\|^2 + 4\nu_t \gamma \|\nabla_\theta F(\theta_t, \omega_t) - p_t\|^2.
\end{aligned}
\end{align}
Combining \eqref{eq:decomp_J_end} together with \eqref{eq:primal_J_bound}, we obtain
\begin{align*}
J(\theta_{t+1})-J(\theta_t)  & \leq -\Big( \frac{\nu_t}{\gamma}-\frac{\nu_t}{8\gamma}-\frac{\nu_t}{16\gamma}-\frac{L_J \nu_t^2}{2} \Big) \|\tilde{\theta}_{t+1} - \theta_t \|^2  \\
&\quad +  2 L_F^2 \gamma \nu_t \|\omega_t - \omega_t^*\|^2 + 4\nu_t \gamma \|\nabla_\theta F(\theta_t, \omega_t) - p_t\|^2.
\end{align*}
We can simplify the coefficient as
\begin{align*}
-\Big( \frac{\nu_t}{\gamma}-\frac{\nu_t}{8\gamma}-\frac{\nu_t}{16\gamma}-\frac{L_J \nu_t^2}{2} \Big)  \leq -\frac{3\nu_t}{4\gamma}, 
\end{align*}
by setting the parameters as  
\begin{align*}
\nu_t \gamma \leq  \frac{1}{8 L_J}.
\end{align*} 
Since $L_F \geq \mu > 0$ and $ L_J = L_F + L_F^2/\mu \leq 2L^2_F/\mu$ as shown in Lemma \ref{lem:omega_lip}, the above inequality can be guaranteed by the condition of this lemma that 
\begin{align*}
\nu_t \gamma \leq  \frac{\mu}{16 L^2_F}.
\end{align*}
Therefore, we eventually obtain
\begin{align*}
J(\theta_{t+1})-J(\theta_t)  & \leq -  \frac{3\nu_t}{4\gamma} \|\tilde{\theta}_{t+1} - \theta_t \|^2  + 2L_F^2 \gamma \nu_t \|\omega_t - \omega_t^*\|^2  + 4\nu_t \gamma \|\nabla_\theta F(\theta_t, \omega_t) - p_t\|^2.
\end{align*}
This completes the proof.
\end{proof}

\begin{lemma}\label{lem:omega_contract}
Under Assumptions \ref{assump:fea_set}, \ref{assump:lip_grad} and \ref{assump:non_singular}, letting $0< \nu_t \leq 1/8$ and $0 < \eta \leq (4L_F)^{-1}$, with the updating rules shown in Algorithm \ref{alg:single_time}, we have 
\begin{align*}
\|\omega_{t+1} - \omega^*(\theta_t)\|^2 \leq \Big(1- \frac{\nu_t \eta \mu}{2} \Big)\|\omega_t-  \omega^*(\theta_t)\|^2  - \frac{3\nu_t}{4} \|\tilde{\omega}_{t+1}-\omega_t \|^2 + \frac{4 \eta\nu_t}{\mu}\| \nabla_\omega F(\theta_t,\omega_t)-d_t\|^2,
\end{align*}
where $\omega^*(\theta_t)=\argmax_{\omega \in \Omega} F(\theta_t, \omega)$.
\end{lemma}

\begin{proof}

Denoting $\omega^*(\theta_t)$ as $\omega^*_t$ for abbreviation, we first expand the term $\|\omega_{t+1} - \omega^*_t\|$ in the following way
\begin{align}
\begin{aligned}\label{eq:dual_expand}
\|\omega_{t+1} - \omega^*_t\|^2 &= \|\omega_t + \nu_t(\tilde{\omega}_{t+1} - \omega_t ) - \omega^*_t\|^2 \\
&= \|\omega_t- \omega^*_t\|^2 + 2  \nu_t \langle \omega_t - \omega^*_t, \tilde{\omega}_{t+1} - \omega_t \rangle +  \nu_t^2 \|\tilde{\omega}_{t+1} - \omega_t \|^2,
\end{aligned}
\end{align}
where the first equality is by the updating rule $\omega_{t+1} = \omega + \nu_t (\tilde{\omega}_{t+1} - \omega_t)$.

Next, we will bound the second term on the right-hand side of \eqref{eq:dual_expand}, i.e., $2\nu_t \langle \omega_t - \omega^*_t,  \tilde{w}_{t+1} - \omega_t\rangle$. We start by considering the projection operation that 
\begin{align*}
\tilde{\omega}_{t+1} = \mathcal{P}_\Omega(\omega_t + \eta d_t) = \argmin_{\tilde{\omega} \in \Omega} \|\tilde{\omega} - (\omega_t + \eta d_t) \|^2.    
\end{align*}
According to the optimality condition for the constrained convex optimization in Lemma \ref{lem:var_ineq}, we have
\begin{align*}
\langle \tilde{\omega}_{t+1}-(\omega_t + \eta d_t),  \omega-\tilde{\omega}_{t+1} \rangle \geq 0, \quad  \forall \omega \in \Omega,
\end{align*}
which further leads to the following inequality via rearranging the terms
\begin{align}
\langle -d_t + \eta^{-1}(\tilde{\omega}_{t+1} -\omega_t), \omega-\tilde{\omega}_{t+1} \rangle \geq 0, \quad  \forall \omega \in \Omega. \label{eq:dual_expand_vari}
\end{align}
Moreover, due to the strong concavity of $F(\theta, \cdot)$ given any $\theta$ as shown in Assumption \ref{assump:non_singular}, applying Lemma \ref{lem:sc_ineq} to the function $F(\theta_t, \cdot)$, we have
\begin{align*}
F(\theta_t,\omega)- F(\theta_t,\omega_t)-\langle \nabla_\omega F(\theta_t,\omega_t), \omega - \omega_t \rangle \leq -\frac{\mu}{2} \|\omega - \omega_t\|^2.   
\end{align*}
Rearranging the terms in the above inequality and then decomposing the term $\langle \nabla_\omega F(\theta_t,\omega_t), \omega - \omega_t \rangle$, we have the following inequality 
\begin{align}
\begin{aligned}\label{eq:dual_F_expand}
F(\theta_t,\omega)+\frac{\mu}{2} \|\omega - \omega_t\|^2 &\leq F(\theta_t,\omega_t) + \langle \nabla_\omega F(\theta_t,\omega_t), \omega - \omega_t \rangle  \\ 
&= F(\theta_t,\omega_t) +  \langle d_t, \omega - \tilde{\omega}_{t+1} \rangle +\langle \nabla_\omega F(\theta_t,\omega_t)-d_t, \omega - \tilde{\omega}_{t+1} \rangle \\
&\quad + \langle \nabla_\omega F(\theta_t,\omega_t), \tilde{\omega}_{t+1} - \omega_t \rangle.
\end{aligned}
\end{align}
Combining the inequality \eqref{eq:dual_expand_vari} and \eqref{eq:dual_F_expand}, then adding and subtracting a same term $(2\eta)^{-1} \|\tilde{\omega}_{t+1} - \omega_t\|^2$ simultaneously, we further have
\begin{align}
\begin{aligned} \label{eq:dual_F_expand_mid} 
F(\theta_t,\omega)+\frac{\mu}{2} \|\omega - \omega_t\|^2 
&\leq F(\theta_t,\omega_t) + \frac{1}{\eta}\langle \tilde{\omega}_{t+1} - \omega_t, \omega - \tilde{\omega}_{t+1} \rangle +\langle \nabla_\omega F(\theta_t,\omega_t)-d_t, \omega - \tilde{\omega}_{t+1} \rangle  \\
&  \quad + \langle \nabla_\omega F(\theta_t,\omega_t), \tilde{\omega}_{t+1} - \omega_t \rangle - \frac{1}{2\eta} \|\tilde{\omega}_{t+1} - \omega_t\|^2 +\frac{1}{2\eta} \|\tilde{\omega}_{t+1} - \omega_t\|^2. 
\end{aligned}
\end{align}
Since $F(\theta, \omega)$ is gradient Lipschitz by Assumption \ref{assump:lip_grad}, and due to the condition in this lemma 
\begin{align*}
    \eta \leq \frac{1}{2L_F} \leq \frac{1}{L_F}, 
\end{align*} 
we have
\begin{align}
\begin{aligned}
-\frac{1}{2\eta} \|\tilde{\omega}_{t+1} - \omega_t\|^2 &\leq -\frac{L_F}{2} \|\tilde{\omega}_{t+1} - \omega_t\|^2 \\
&\leq F(\theta_t, \tilde{\omega}_{t+1}) -  F(\theta_t, \omega_t) - \langle \nabla_\omega F(\theta_t, \omega_t),  \tilde{\omega}_{t+1} - \omega_t \rangle . \label{eq:dual_F_lips}
\end{aligned}
\end{align}
Summing up both sides of the two inequalities \eqref{eq:dual_F_expand_mid} and \eqref{eq:dual_F_lips} and canceling terms yields 
\begin{align*}
F(\theta_t,\omega)+\frac{\mu}{2} \|\omega - \omega_t\|^2 &\leq F(\theta_t,\tilde{\omega}_{t+1}) + \frac{1}{\eta}\langle \tilde{\omega}_{t+1} - \omega_t, \omega - \tilde{\omega}_{t+1} \rangle \\
&\quad +\langle \nabla_\omega F(\theta_t,\omega_t)-d_t, \omega - \tilde{\omega}_{t+1} \rangle \nonumber  +\frac{1}{2\eta} \|\tilde{\omega}_{t+1} - \omega_t\|^2.
\end{align*}
Note that in the last inequality, we can directly compute
\begin{align*}
&\frac{1}{\eta}\langle \tilde{\omega}_{t+1} - \omega_t, \omega - \tilde{\omega}_{t+1} \rangle + \frac{1}{2\eta} \|\tilde{\omega}_{t+1} - \omega_t\|^2 \\
&\qquad = \frac{1}{\eta}\langle \tilde{\omega}_{t+1} - \omega_t, \omega_t - \tilde{\omega}_{t+1} \rangle + \frac{1}{\eta}\langle \tilde{\omega}_{t+1} - \omega_t, \omega - \omega_t \rangle + \frac{1}{2\eta} \|\tilde{\omega}_{t+1} - \omega_t\|^2 \\
&\qquad  = \frac{1}{\eta}\langle \tilde{\omega}_{t+1} - \omega_t, \omega - \omega_t \rangle - \frac{1}{2\eta} \|\tilde{\omega}_{t+1} - \omega_t\|^2,
\end{align*}
which thus leads to
\begin{align*}
F(\theta_t,\omega)+\frac{\mu}{2} \|\omega - \omega_t\|^2 &\leq F(\theta_t,\tilde{\omega}_{t+1}) + \frac{1}{\eta}\langle \tilde{\omega}_{t+1} - \omega_t, \omega - \omega_t \rangle \\
&\quad +\langle \nabla_\omega F(\theta_t,\omega_t)-d_t, \omega - \tilde{\omega}_{t+1} \rangle - \frac{1}{2\eta} \|\tilde{\omega}_{t+1} - \omega_t\|^2.    
\end{align*}
We let $\omega = \omega^*_t$ and obtain
\begin{align*}
F(\theta_t,\omega)+\frac{\mu}{2} \|\omega^*_t - \omega_t\|^2 &\leq F(\theta_t,\tilde{\omega}_{t+1}) + \frac{1}{\eta}\langle \tilde{\omega}_{t+1} - \omega_t, \omega^*_t - \omega_t \rangle \\
&\quad +\langle \nabla_\omega F(\theta_t,\omega_t)-d_t, \omega^*_t - \tilde{\omega}_{t+1} \rangle - \frac{1}{2\eta} \|\tilde{\omega}_{t+1} - \omega_t\|^2,
\end{align*}
which further yields
\begin{align}
\begin{aligned}\label{eq:dual_al_1}
&\frac{\mu}{2} \|\omega^*_t - \omega_t\|^2 + \frac{1}{2\eta} \|\tilde{\omega}_{t+1} - \omega_t\|^2 \\
&\qquad \leq \frac{1}{\eta}\langle \tilde{\omega}_{t+1} - \omega_t,\omega^*_t  - \omega_t \rangle +\langle \nabla_\omega F(\theta_t,\omega_t)-d_t, \omega^*_t - \tilde{\omega}_{t+1} \rangle, 
\end{aligned}
\end{align}
since $F(\theta_t,\omega^*_t) \geq F(\theta_t,\tilde{\omega}_{t+1})$ due to strong concavity and $\omega^*_t = \argmax_{\omega \in \Omega} F(\theta_t, \omega)$. In addition, for the last term of the above inequality, we further bound it as follows
\begin{align}
\begin{aligned}\label{eq:dual_al_2}
&\langle \nabla_\omega F(\theta_t,\omega_t)-d_t, \omega^*_t - \tilde{\omega}_{t+1} \rangle  \\
&\qquad = \langle \nabla_\omega F(\theta_t,\omega_t)-d_t, \omega^*_t - \omega_t \rangle  + \langle \nabla_\omega F(\theta_t,\omega_t)-d_t, \omega_t - \tilde{\omega}_{t+1} \rangle  \\
&\qquad \leq  \frac{1}{\mu}\| \nabla_\omega F(\theta_t,\omega_t)-d_t\|^2 + \frac{\mu}{4} \|\omega^*_t - \omega_t \|^2  + \frac{1}{\mu} \| \nabla_\omega F(\theta_t,\omega_t)-d_t\|^2 + \frac{\mu}{4} \| \omega_t - \tilde{\omega}_{t+1} \|^2 \\
&\qquad =  \frac{2}{\mu}\| \nabla_\omega F(\theta_t,\omega_t)-d_t\|^2 + \frac{\mu}{4} \|\omega^*_t - \omega_t \|^2 + \frac{\mu}{4} \|\tilde{\omega}_{t+1}-\omega_t \|^2, 
\end{aligned}
\end{align}
where the inequality is by Young's inequality $\langle x,y \rangle \leq  (\lambda/2)\cdot \|x\|^2 + (2\lambda)^{-1}\|y\|^2$ with setting $\lambda = 2/\mu$. Therefore, plugging  \eqref{eq:dual_al_2} into \eqref{eq:dual_al_1}, then multiplying both sides by $2\eta \nu_t$, and rearranging the terms, we obtain 
\begin{align}
\begin{aligned}\label{eq:dual_al_3}
&2\nu_t \langle \tilde{w}_{t+1} - \omega_t,\omega_t - \omega^*_t \rangle \\
&\qquad \leq -\frac{\nu_t \eta \mu}{2} \| \omega_t - \omega^*_t \|^2 - \frac{2\nu_t -\nu_t \eta\mu}{2}\|\tilde{\omega}_{t+1}-\omega_t \|^2  + \frac{4 \eta\nu_t}{\mu} \| \nabla_\omega F(\theta_t,\omega_t)-d_t\|^2,    
\end{aligned}
\end{align}
which gives the upper bound of the second term on the right-hand side of \eqref{eq:dual_expand}.

Combining \eqref{eq:dual_expand} and \eqref{eq:dual_al_3}, we have
\begin{align*}
\|\omega_{t+1} - \omega^*_t\|^2 \leq \frac{2-\nu_t \eta \mu}{2}\|\omega_t- \omega^*_t\|^2  - \frac{2\nu_t - \nu_t \eta\mu-2\nu^2_t}{2} \|\tilde{\omega}_{t+1}-\omega_t \|^2 + \frac{4 \eta\nu_t}{\mu}\| \nabla_\omega F(\theta_t,\omega_t)-d_t\|^2.
 \end{align*}
Thus, according to the condition of this lemma that $
\nu_t \leq 1/8$ and $\eta \leq (4L_F)^{-1} \leq (4\mu)^{-1}$ by the fact $L_F \geq \mu > 0$, we have
\begin{align*}
-\frac{2\nu_t - \nu_t \eta\mu-2\nu^2_t}{2} \leq -\frac{3\nu_t}{4},
\end{align*}
which eventually leads to
\begin{align*}
&\|\omega_{t+1} - \omega^*_t\|^2 \leq \Big(1- \frac{\nu_t \eta \mu}{2} \Big)\|\omega_t- \omega^*_t\|^2  - \frac{3\nu_t}{4} \|\tilde{\omega}_{t+1}-\omega_t \|^2 + \frac{4 \eta\nu_t}{\mu}\| \nabla_\omega F(\theta_t,\omega_t)-d_t\|^2,
\end{align*}
which completes the proof. 
\end{proof}

Based on Lemma \ref{lem:omega_contract}, we obtain the following lemma.

\begin{lemma} \label{lem:decomp_omega_opt} Under Assumptions \ref{assump:fea_set}, \ref{assump:lip_grad} and \ref{assump:non_singular}, letting $0< \nu_t \leq 1/8$ and $0 < \eta \leq (4L_F)^{-1}$, with the updating rules shown in Algorithm \ref{alg:single_time}, we have 
\begin{align*}
\|\omega_{t+1} - \omega^*(\theta_{t+1})\|^2 &\leq \Big(1-\frac{\mu\eta\nu_t}{4}\Big) \|\omega_t - \omega^*(\theta_t)\|^2 - \frac{3\nu_t}{4}\|\tilde{\omega}_{t+1} - \omega_t \|^2 \\
&\quad + \frac{75\eta \nu_t}{16\mu} \|d_t - \nabla_\omega F(\theta_t, \omega_t)\|^2 \nonumber  + \frac{75L_\omega^2 \nu_t }{16\mu\eta}  \|\tilde{\theta}_{t+1}-\theta_t\|^2,
\end{align*}
where $\omega^*(\theta_t)=\argmax_{\omega \in \Omega} F(\theta_t, \omega)$ and $\omega^*(\theta_{t+1})=\argmax_{\omega \in \Omega} F(\theta_{t+1}, \omega)$.
\end{lemma}

\begin{proof}
Denoting $\omega^*(\theta_t)$ and $\omega^*(\theta_{t+1})$ as $\omega^*_t$ and $\omega^*_{t+1}$ for abbreviation, we start the proof by decomposing the term $\|\omega_{t+1} - \omega^*_{t+1}\|^2$ as follows
\begin{align*}
\|\omega_{t+1} - \omega^*_{t+1}\|^2 & = \|\omega_{t+1} - \omega^*_t  + \omega^*_t- \omega^*_{t+1}\|^2 \\
& \leq \Big(1+\frac{\mu\eta\nu_t}{4} \Big)\|\omega_{t+1} - \omega^*_t\|^2 + \Big (1+\frac{4}{\mu\eta\nu_t}\Big) \|\omega^*_t - \omega^*_{t+1}\|^2 \nonumber \\ 
& \leq \Big(1+\frac{\mu\eta\nu_t}{4} \Big)\|\omega_{t+1} - \omega^*_t\|^2 + \Big (1+\frac{4}{\mu\eta\nu_t}\Big)L_\omega^2 \|\theta_{t+1}-\theta_t\|^2,     
\end{align*}
where the first inequality is by Young's inequality $ \|x+y\|^2 \leq (1+\lambda) \|x\|^2 + (1+\lambda^{-1}) \|y\|^2$ with setting $\lambda = \mu\eta\nu_t/4$, and the second inequality is due to the Lipschitz continuity property of $\omega^*(\theta)$ as shown in Lemma \ref{lem:omega_lip}. In addition, plugging the updating rule $\theta_{t+1} = \theta_t + \nu_t (\tilde{\theta}_{t+1} - \theta_t)$ into the above inequality, we obtain
\begin{align}
\|\omega_{t+1} - \omega^*_{t+1}\|^2 \leq \Big(1+\frac{\mu\eta\nu_t}{4} \Big)\|\omega_{t+1} - \omega^*_t\|^2 + \Big (1+\frac{4}{\mu\eta\nu_t}\Big)L_\omega^2 \nu_t^2 \|\tilde{\theta}_{t+1}-\theta_t\|^2. \label{eq:decomp_omega_opt}
\end{align}
Furthermore, according to Lemma \ref{lem:omega_contract}, we have
\begin{align}
\begin{aligned}\label{eq:omega_contra}
&\|\omega_{t+1} - \omega^*_t\|^2 \\
&\qquad \leq \Big(1- \frac{\nu_t \eta \mu}{2} \Big)\|\omega_t- \omega^*_t\|^2  - \frac{3\nu_t}{4} \|\tilde{\omega}_{t+1}-\omega_t \|^2 + \frac{4 \eta\nu_t}{\mu}\| \nabla_\omega F(\theta_t,\omega_t)-d_t\|^2. 
\end{aligned}
\end{align}
Therefore, plugging \eqref{eq:omega_contra} into \eqref{eq:decomp_omega_opt}, we obtain the following inequality 
\begin{align}
\begin{aligned}\label{eq:decomp_omega_al}
&\|\omega_{t+1} - \omega^*_{t+1}\|^2 \\
&\qquad  \leq \Big(1+\frac{\mu\eta\nu_t}{4} \Big)  \Big(1-\frac{\nu_t \eta \mu}{2} \Big)\|\omega_t- \omega^*_t\|^2  - \Big(1+\frac{\mu\eta\nu_t}{4} \Big) \frac{3\nu_t}{4} \|\tilde{\omega}_{t+1}-\omega_t \|^2   \\
&\qquad  \quad  +  \frac{4 \eta\nu_t}{\mu}  \Big(1+\frac{\mu\eta\nu_t}{4} \Big)\| \nabla_\omega F(\theta_t,\omega_t)-d_t\|^2  + \Big(1+\frac{4}{\mu\eta\nu_t}\Big) L_\omega^2  \nu_t^2 \|\tilde{\theta}_{t+1}-\theta_t\|^2. 
\end{aligned}
\end{align}
Now we simplify the coefficients in \eqref{eq:decomp_omega_al}. According to the conditions $0 < \eta \leq (4L_F)^{-1}$, $0 < \nu_t \leq 1/8$ and due to $L_F\geq \mu>0$, we have
\begin{align*}
\eta \leq \frac{1}{4L_F}\leq \frac{1}{4\mu},\quad \text{ and } \ \eta\nu_t \leq \frac{1}{32\mu},  
\end{align*}
which yield
\begin{align}
\begin{aligned}\label{eq:decomp_omega_cond}
&\Big(1+ \frac{\mu\eta\nu_t}{4} \Big) \Big (1-\frac{\mu\eta\nu_t}{2} \Big) = 1-\frac{\mu\eta\nu_t}{2} + \frac{\mu\eta\nu_t}{4}-\frac{\mu^2\eta^2\nu_t^2}{4} \leq 1-\frac{\mu\eta\nu_t}{4},  \\ 
&-\Big(1+\frac{\mu\eta\nu_t}{4} \Big) \frac{3\nu_t}{4}\leq  -\frac{3\nu_t}{4}, \quad  \frac{4 \eta\nu_t}{\mu}  \Big(1+\frac{\mu\eta\nu_t}{4} \Big)= \frac{4 \eta\nu_t}{\mu} +\eta^2\nu_t^2 < \frac{75 \eta\nu_t}{16\mu}, \\
& \text{ and }  \Big(1+\frac{4}{\mu\eta\nu_t}\Big) L_\omega^2  \nu_t^2 \leq \frac{129}{32}\frac{L_\omega^2   \nu_t^2}{\mu\eta\nu_t} < \frac{75 L_\omega^2  \nu_t}{16\mu\eta}.
\end{aligned}
\end{align}
Combining \eqref{eq:decomp_omega_cond} and \eqref{eq:decomp_omega_al}, we eventually obtain
\begin{align*}
\|\omega_{t+1} - \omega^*_{t+1}\|^2 &\leq \Big(1-\frac{\mu\eta\nu_t}{4}\Big) \|\omega_t - \omega^*_t\|^2 - \frac{3\nu_t}{4}\|\tilde{\omega}_{t+1} - \omega_t \|^2 \\
&\quad + \frac{75\eta \nu_t}{16\mu} \|d_t - \nabla_\omega F(\theta_t, \omega_t)\|^2 \nonumber  + \frac{75L_\omega^2 \nu_t }{16\mu\eta}  \|\tilde{\theta}_{t+1}-\theta_t\|^2,
\end{align*}
which completes the proof.
\end{proof}

\begin{lemma} \label{lem:grad_var_bound} Under Assumptions \ref{assump:fea_set}, \ref{assump:lip_grad}, and \ref{assump:bounded_var}, letting $0< \nu_t \leq (8\alpha)^{-1}$, with the updating rules shown in Algorithm \ref{alg:single_time}, we have 
\begin{align}
\begin{aligned}
\EE \|\nabla_\theta F(\theta_{t+1},\omega_{t+1})-p_{t+1}\|^2 &\leq  (1- \alpha\nu_t)  \EE \|\nabla_\theta F(\theta_t,\omega_t)- p_t\|^2 \\
&\quad +  \frac{9\nu_t L_F^2}{8\alpha} \EE (\|\tilde{\theta}_{t+1} - \theta_t \|^2 +\|\tilde{\omega}_{t+1}-\omega_t\|^2)  + \alpha^2\nu_t^2 \sigma^2,\label{eq:grad_var_bound_1}
\end{aligned}
\end{align}
\vspace{-0.5cm}
\begin{align}
\begin{aligned}
\EE \|\nabla_\omega F(\theta_{t+1},\omega_{t+1})-d_{t+1}\|^2 &\leq   (1-\alpha\nu_t)  \EE \|\nabla_\omega F(\theta_t,\omega_t)- d_t\|^2 \\
&\quad + \frac{9\nu_t L_F^2}{8\alpha} \EE (\|\tilde{\theta}_{t+1} - \theta_t \|^2 +\|\tilde{\omega}_{t+1}-\omega_t\|^2)  + \alpha^2 \nu_t^2 \sigma^2.\label{eq:grad_var_bound_2}
\end{aligned}
\end{align}    
\end{lemma}

\begin{proof} We first show the detailed proof for the inequality \eqref{eq:grad_var_bound_1} in this lemma. Then, the proof of the second inequality \eqref{eq:grad_var_bound_2} is very similar to the proof of \eqref{eq:grad_var_bound_1} , for which we only present a proof sketch.  We start our proof by decomposing the term $\nabla_\theta F(\theta_{t+1},\omega_{t+1})-p_{t+1}$ as follows
\begin{align*}
&\nabla_\theta F(\theta_{t+1},\omega_{t+1})-p_{t+1}\\ 
&\quad = \nabla_\theta F(\theta_{t+1},\omega_{t+1})-(1-\alpha\nu_t) p_t - \alpha\nu_t \nabla_\theta f(\theta_{t+1},\omega_{t+1};\xi_{t+1})\\
&\quad = (1-\alpha \nu_t) [\nabla_\theta F(\theta_{t+1},\omega_{t+1})- p_t] + \alpha \nu_t [\nabla_\theta F(\theta_{t+1},\omega_{t+1}) - \nabla_\theta f(\theta_{t+1},\omega_{t+1};\xi_{t+1})],
\end{align*}
where we use the updating rule $p_{t+1} = (1-\alpha\nu_t) p_t + \alpha\nu_t \nabla_\theta f(\theta_{t+1},\omega_{t+1};\xi_{t+1})$ shown in Algorithm \ref{alg:single_time}. Thus, we have
\begin{align}
\begin{aligned} \label{eq:grad_var_decomp}
&\EE \|\nabla_\theta F(\theta_{t+1},\omega_{t+1})-p_{t+1}\|^2\\ 
& = (1-\alpha \nu_t)^2 \EE \|\nabla_\theta F(\theta_{t+1},\omega_{t+1})- p_t\|^2 + \alpha ^2\nu_t^2 \EE \|\nabla_\theta F(\theta_{t+1},\omega_{t+1}) - \nabla_\theta f(\theta_{t+1},\omega_{t+1};\xi_{t+1})\|^2 \\
& \quad + 2(1-\alpha \nu_t)\alpha \nu_t \EE \langle \nabla_\theta F(\theta_{t+1},\omega_{t+1})- p_t, \nabla_\theta F(\theta_{t+1},\omega_{t+1}) - \nabla_\theta f(\theta_{t+1},\omega_{t+1};\xi_{t+1})\rangle \\
&= (1-\alpha \nu_t)^2 \EE \|\nabla_\theta F(\theta_{t+1},\omega_{t+1})- p_t\|^2 + \alpha ^2\nu_t^2 \EE \|\nabla_\theta F(\theta_{t+1},\omega_{t+1}) - \nabla_\theta f(\theta_{t+1},\omega_{t+1};\xi_{t+1})\|^2, 
\end{aligned}
\end{align}
where the last equality is by chain rule of expectation, i.e.,
\begin{align*}
&\EE \langle \nabla_\theta F(\theta_{t+1},\omega_{t+1})- p_t, \nabla_\theta F(\theta_{t+1},\omega_{t+1}) - \nabla_\theta f(\theta_{t+1},\omega_{t+1};\xi_{t+1})\rangle\\
&\quad =  \EE \{\EE_{\xi_{t+1}}  \langle \nabla_\theta F(\theta_{t+1},\omega_{t+1})- p_t, \nabla_\theta F(\theta_{t+1},\omega_{t+1}) - \nabla_\theta f(\theta_{t+1},\omega_{t+1};\xi_{t+1})\rangle \}  =  0,
\end{align*}
since $\EE_{\xi_{t+1}} [\nabla_\theta f(\theta_{t+1},\omega_{t+1};\xi_{t+1})] = \nabla_\theta F(\theta_{t+1},\omega_{t+1})$. 

Next, we bound the term $(1-\alpha \nu_t)^2 \EE \|\nabla_\theta F(\theta_{t+1},\omega_{t+1})- p_t\|^2$ in \eqref{eq:grad_var_decomp} in the following way
\begin{align}
\begin{aligned} \label{eq:grad_var_decomp_t1}
&(1-\alpha \nu_t)^2 \EE \|\nabla_\theta F(\theta_{t+1},\omega_{t+1})- p_t\|^2\\
&\qquad = (1-\alpha \nu_t)^2 \EE \|\nabla_\theta F(\theta_{t+1},\omega_{t+1}) - \nabla_\theta F(\theta_t,\omega_t) + \nabla_\theta F(\theta_t,\omega_t)- p_t\|^2 \\
&\qquad \leq (1-\alpha \nu_t)^2 \Big (1+\frac{1}{\alpha \nu_t} \Big) \EE \|\nabla_\theta F(\theta_{t+1},\omega_{t+1}) - \nabla_\theta F(\theta_t,\omega_t)\|^2 \\
&\qquad \quad + (1-\alpha \nu_t)^2 (1+\alpha \nu_t)  \EE \|\nabla_\theta F(\theta_t,\omega_t)- p_t\|^2 \\
&\qquad \leq \frac{9}{8\alpha \nu_t} \EE \|\nabla_\theta F(\theta_{t+1},\omega_{t+1}) - \nabla_\theta F(\theta_t,\omega_t)\|^2  + (1-\alpha \nu_t)  \EE \|\nabla_\theta F(\theta_t,\omega_t)- p_t\|^2,
\end{aligned}
\end{align}
where the first inequality is by Young's inequality $\|x+y\|^2 \leq (1+\lambda) \|x\|^2 + (1+\lambda^{-1}) \|y\|^2$ with setting $\lambda = \alpha \nu_t$, the second inequality is due to the condition $0< \nu_t \leq (8\alpha)^{-1}$ and then
\begin{align*}
&(1-\alpha \nu_t)^2 \Big(1+\frac{1}{\alpha \nu_t}\Big) \leq 1 + \frac{1}{\alpha \nu_t} \leq  \frac{9}{8\alpha \nu_t} ,\\
& \text{ and }  (1-\alpha \nu_t)^2 (1+ \alpha \nu_t) = 1-\alpha \nu_t -\alpha^2\nu_t^2 + \alpha^3 \nu_t^3  \leq 1-\alpha \nu_t.    
\end{align*}
By the Lipschitz continuity of $\nabla_\theta F(\theta, \omega)$ in Assumption \ref{assump:lip_grad} and the updating rules that $\theta_{t+1} = \theta_t + \nu_t (\tilde{\theta}_t - \theta_t)$ and $\omega_{t+1} = \omega_t + \nu_t (\tilde{\omega}_t - \omega_t)$, we further have
\begin{align}
\begin{aligned}\label{eq:grad_var_decomp_t1_tmp}
\frac{9}{8\alpha \nu_t} \EE \|\nabla_\theta F(\theta_{t+1},\omega_{t+1}) - \nabla_\theta F(\theta_t,\omega_t)\|^2 &\leq \frac{9L_F^2}{8\alpha \nu_t}  (\|\theta_{t+1} - \theta_t \|^2 +\|\omega_{t+1}-\omega_t\|^2)\\
& \leq \frac{9L_F^2 \nu_t}{8\alpha } (\|\tilde{\theta}_{t+1} - \theta_t \|^2 +\|\tilde{\omega}_{t+1}-\omega_t\|^2).
\end{aligned}
\end{align}
Therefore, combining \eqref{eq:grad_var_decomp_t1} and \eqref{eq:grad_var_decomp_t1_tmp}, we obtain
\begin{align}
\begin{aligned} \label{eq:grad_var_decomp_t1_end}
(1-\alpha \nu_t)^2 \EE \|\nabla_\theta F(\theta_{t+1},\omega_{t+1})- p_t\|^2 & \leq \frac{9\nu_t L_F^2}{8\alpha }  \|\tilde{\theta}_{t+1} - \theta_t \|^2   + \frac{9\nu_t L_F^2}{8\alpha } \|\tilde{\omega}_{t+1}-\omega_t\|^2  \\
&\quad + (1-\alpha \nu_t)  \EE \|\nabla_\theta F(\theta_t,\omega_t)- p_t\|^2.
\end{aligned}
\end{align}
On the other hand, due to the bounded variance assumption in Assumption \ref{assump:bounded_var}, for the last term in \eqref{eq:grad_var_decomp}, we have
\begin{align}\label{eq:grad_var_decomp_t2_end}
\alpha^2 \nu_t^2 \EE \|\nabla_\theta F(\theta_{t+1},\omega_{t+1}) - \nabla_\theta f(\theta_{t+1},\omega_{t+1};\xi_{t+1})\|^2 \leq \alpha^2 \nu_t^2 \sigma^2.
\end{align}
Thus, combining \eqref{eq:grad_var_decomp}, \eqref{eq:grad_var_decomp_t1_end} and \eqref{eq:grad_var_decomp_t2_end},  we get
\begin{align*}
\EE \|\nabla_\theta F(\theta_{t+1},\omega_{t+1})-p_{t+1}\|^2 &\leq  (1- \alpha\nu_t )  \EE \|\nabla_\theta F(\theta_t,\omega_t)- p_t\|^2 \\
&\quad +  \frac{9\nu_t L_F^2}{8\alpha} \EE (\|\tilde{\theta}_{t+1} - \theta_t \|^2 +\|\tilde{\omega}_{t+1}-\omega_t\|^2)  + \alpha^2 \nu_t^2 \sigma^2.
\end{align*}
Then, we apply the above analysis for proving \eqref{eq:grad_var_bound_1} to similarly prove the second inequality \eqref{eq:grad_var_bound_2} of this lemma. We have the following decomposition
\begin{align}
\begin{aligned} \label{eq:grad_var_decomp_dual}
\EE \|\nabla_\omega F(\theta_{t+1},\omega_{t+1})-d_{t+1}\|^2
&= (1-\alpha\nu_t)^2 \EE \|\nabla_\omega F(\theta_{t+1},\omega_{t+1})- d_t\|^2 \\
&\quad + \alpha^2\nu_t^2 \EE \|\nabla_\omega F(\theta_{t+1},\omega_{t+1}) - \nabla_\omega f(\theta_{t+1},\omega_{t+1};\xi_{t+1})\|^2, 
\end{aligned}
\end{align}
We bound the first term on the right-hand side of  \eqref{eq:grad_var_decomp_dual} as
\begin{align}
\begin{aligned} \label{eq:grad_var_decomp_t1_dual}
(1-\alpha\nu_t)^2 \EE \|\nabla_\omega F(\theta_{t+1},\omega_{t+1})- d_t\|^2 & \leq \frac{9\nu_t L_F^2}{8\alpha}  \|\tilde{\theta}_{t+1} - \theta_t \|^2   + \frac{9\nu_t L_F^2}{8\alpha} \|\tilde{\omega}_{t+1}-\omega_t\|^2  \\
&\quad + (1-\alpha\nu_t)  \EE \|\nabla_\omega F(\theta_t,\omega_t)- d_t\|^2,
\end{aligned}
\end{align}
with the condition $0< \nu_t \leq (8\alpha)^{-1}$.
Then, we bound the last term of \eqref{eq:grad_var_decomp_dual} as
\begin{align}\label{eq:grad_var_decomp_t2_dual}
\alpha^2\nu_t^2 \EE \|\nabla_\omega F(\theta_{t+1},\omega_{t+1}) - \nabla_\omega f(\theta_{t+1},\omega_{t+1};\xi_{t+1})\|^2 \leq \alpha^2 \nu_t^2 \sigma^2.
\end{align}
Thus, combining \eqref{eq:grad_var_decomp_dual}, \eqref{eq:grad_var_decomp_t1_dual} and \eqref{eq:grad_var_decomp_t2_dual},  we get
\begin{align*}
\EE \|\nabla_\omega F(\theta_{t+1},\omega_{t+1})-d_{t+1}\|^2 &\leq   (1-\alpha \nu_t)  \EE \|\nabla_\omega F(\theta_t,\omega_t)- d_t\|^2 \\
&\quad + \frac{9\nu_t L_F^2}{8\alpha} \EE (\|\tilde{\theta}_{t+1} - \theta_t \|^2 +\|\tilde{\omega}_{t+1}-\omega_t\|^2)  + \alpha^2\nu_t^2 \sigma^2.
\end{align*}
The proof is completed.
\end{proof}

\subsection{Proof of Theorem \ref{thm:single_time}} \label{sec:detail_proof_single_time}

%\begin{theorem}[Theorem \ref{thm:single_time}] Under Assumptions \ref{assump:exist_sol}, \ref{assump:fea_set}, \ref{assump:lip_grad} and \ref{assump:non_singular}, setting the parameters $\alpha = 3$, $0< \eta \leq \mu /(4 L_F^2)$, $0 < \gamma \leq \eta\mu^2 /(9 L_F^2)$, and $\nu_t = a/(t+b)^{1/2}$ with $a = 1/16$ and $b \geq  \max \{ (\gamma L_F^2/\mu)^2, 3 \}$, with the updating rules in Algorithm \ref{alg:single_time}, the convergence rate of this algorithm is
%\begin{align*}
%\frac{1}{T} \sum_{t=0}^{T-1} \Big [\frac{1}{\gamma^2} \EE \|\nabla_\theta F(\theta_t,\omega_t)-p_t\| + \EE \|\tilde{\theta}_{t+1} - \theta_t\|  + L_F \EE \|\omega_t - \omega^*(\theta_t)\| \Big ] \leq \frac{C_1}{T^{1/4}} + \frac{C_2}{\sqrt{T}},
%\end{align*}
%where where $C_1 = [768 Q_0/\gamma +  216 \sigma^2/(\mu\eta)]^{1/2}$, $C_2 = C_1 b^{1/4}$, and $Q_0 = J(\theta_0) - J^* +  10L_F^2\gamma/(\mu\eta) \|\omega_0 - \omega^*(\theta_0)\|^2 + 2\gamma/(\mu\eta) (\|\nabla_\omega F(\theta_0,\omega_0)-d_0\|^2 + \|\nabla_\theta F(\theta_0,\omega_0)-p_0\|^2).$
%\end{theorem}

\begin{proof} We assume that the step size is of the form $\nu_t = a/(t+b)^{1/2}$ where $a=1/16$. We interpret the parameter settings in Theorem \ref{thm:single_time} as follows:
\begin{align*}
\eta \leq \frac{\mu}{4L_F^2} \leq \frac{1}{4L_F}, \quad \text{ and } \nu_t \leq \frac{a}{b^{1/2}} \leq \min \Big \{ \frac{1}{27},\ \  \frac{\mu}{16\gamma L_F^2} \Big \},
\end{align*}
with $L_F \geq \mu > 0$. Thus, with the parameter settings as above, we can apply Lemmas \ref{lem:primal_J_bound}, \ref{lem:decomp_omega_opt}, and \ref{lem:grad_var_bound} in the following proof of Algorithm \ref{alg:single_time}. Then, we proceed to the main proof.

By Lemma \ref{lem:primal_J_bound}, with taking expectation, we have 
\begin{align}
\begin{aligned}\label{eq:thm_bound_1}
\EE[J(\theta_{t+1})-J(\theta_t)]  & \leq -  \frac{3\nu_t}{4\gamma} \EE\|\tilde{\theta}_{t+1} - \theta_t \|^2  + 2 L_F^2 \gamma \nu_t \EE\|\omega_t - \omega^*(\theta_t)\|^2 \\
&\quad + 4\nu_t \gamma \EE \|\nabla_\theta F(\theta_t, \omega_t) - p_t\|^2.
\end{aligned}
\end{align}
In this inequality, the left-hand side will be a telescoping sum if we take a summation from $t =0$ to $T-1$, and moving the first term on the right-hand side to the left can result in a upper bound for $\EE\| \tilde{\theta}_{t+1} - \theta_t\|^2$. To guarantee its convergence, we still need to understand the upper bounds of the remaining terms in \eqref{eq:thm_bound_1}, namely, $\EE \|\omega_t - \omega^*(\theta_t)\|^2$ and $\EE\|p_t - \nabla_\theta F(\theta_t, \omega_t)\|^2$. 

Then, we establish an inequality whose right-hand side indicates a contraction of the term $\EE \|\omega_{t+1} - \omega^*(\theta_{t+1})\|^2$ plus some noise terms as well as a deduction of the term $\EE\|\tilde{\omega}_{t+1} - \omega_t\|^2$ such that this term can be eliminated in the end. According to Lemma \ref{lem:decomp_omega_opt}, taking expectation on both sides, the inequality is in the form of \begin{align*}
\EE\|\omega_{t+1} - \omega^*(\theta_{t+1})\|^2 &\leq \Big(1-\frac{\mu\eta\nu_t}{4}\Big) \EE\|\omega_t - \omega^*(\theta_t)\|^2 - \frac{3\nu_t}{4}\EE \|\tilde{\omega}_{t+1} - \omega_t \|^2 \\
&\quad + \frac{75\eta \nu_t}{16\mu} \EE\|d_t - \nabla F(\theta_t, \omega_t)\|^2 \nonumber  + \frac{75L_\omega^2 \nu_t }{16\mu\eta}  \EE\|\tilde{\theta}_{t+1}-\theta_t\|^2.
\end{align*}
Multiplying both sides of the above inequality by $10L_F^2 \gamma /(\mu\eta)$, we obtain
\begin{align*}
\frac{10L_F^2 \gamma}{\mu\eta}\EE\|\omega_{t+1} - \omega^*(\theta_{t+1})\|^2 & \leq \frac{10L_F^2 \gamma}{\mu\eta}\Big(1-\frac{\mu\eta\nu_t}{4}\Big) \EE\|\omega_t - \omega^*(\theta_t)\|^2 - \frac{15L_F^2 \gamma  \nu_t}{2\mu\eta} \EE\|\tilde{\omega}_{t+1} - \omega_t \|^2\\
& \quad + \frac{375 L_F^2 \gamma \nu_t}{8\mu^2} \EE \|d_t - \nabla F(\theta_t, \omega_t)\|^2  + \frac{375 L_F^2 L_\omega^2 \gamma \nu_t }{8\mu^2\eta^2}  \EE \|\tilde{\theta}_{t+1}-\theta_t\|^2.
\end{align*}
Then, by rearranging the terms, we have
\begin{align}
\begin{aligned}\label{eq:thm_bound_2}
&\frac{10L_F^2 \gamma}{\mu\eta} \big( \EE \|\omega_{t+1} - \omega^*(\theta_{t+1})\|^2 - \EE \|\omega_t - \omega^*(\theta_t)\|^2 \big) \\
&\qquad\leq -\frac{5L_F^2 \gamma \nu_t}{2} \EE \|\omega_t - \omega^*(\theta_t)\|^2 - \frac{15L_F^2 \gamma  \nu_t}{2\mu\eta} \EE\|\tilde{\omega}_{t+1} - \omega_t \|^2\\
&\qquad \quad  + \frac{375 L_F^2 \gamma \nu_t}{8\mu^2} \EE \|d_t - \nabla F(\theta_t, \omega_t)\|^2  + \frac{375 L_F^2 L_\omega^2 \gamma \nu_t }{8\mu^2\eta^2}  \EE \|\tilde{\theta}_{t+1}-\theta_t\|^2.
\end{aligned}
\end{align}
We define 
\begin{align*}
P_t := J(\theta_t) - J^*  +  \frac{10L_F^2 \gamma}{\mu\eta} \|\omega_t - \omega^*(\theta_t)\|^2 , \quad  \forall t\geq 0,
\end{align*}
where $J^*>-\infty$ denotes the minimal value such that $J(\theta) \geq J^*, \forall \theta \in \Theta$ according to Assumption \eqref{assump:exist_sol}. Then, summing up both sides of the two inequalities \eqref{eq:thm_bound_1} and \eqref{eq:thm_bound_2}, we have
\begin{align*}
\EE[P_{t+1}-P_t] &\leq- \Big ( \frac{3\nu_t}{4\gamma} - \frac{375 L_F^2 L_\omega^2 \gamma \nu_t }{8\mu^2\eta^2} \Big) \EE\|\tilde{\theta}_{t+1} - \theta_t\|^2 - \frac{15L_F^2 \gamma  \nu_t}{2\mu\eta} \EE\|\tilde{\omega}_{t+1} - \omega_t \|^2 \\
&\quad + \frac{375 L_F^2 \gamma \nu_t}{8\mu^2} \EE \|d_t - \nabla_\omega F(\theta_t, \omega_t)\|^2  + 4\nu_t \gamma \EE \|p_t - \nabla_\theta F(\theta_t, \omega_t) \|^2     -\frac{ L_F^2 \gamma \nu_t}{2}\EE \|\omega_t - \omega^*(\theta_t)\|^2 .
\end{align*}
According to the conditions that $\eta \geq 9 L_F^2 \gamma /\mu^2 $ such that $\eta^2 \geq 81 L_F^4 \gamma^2 / \mu^4$ and by Lemma \ref{lem:omega_lip} that $L_\omega = L_F/\mu$, then we have 
\begin{align*}
- \Big ( \frac{3\nu_t}{4\gamma} - \frac{225 \nu_t L_F^2 L_\omega^2 }{32\mu^2 \eta}\Big) \leq -\frac{\nu_t}{8\gamma},
\end{align*}
which leads to 
\begin{align}
\begin{aligned}\label{eq:thm_bound_3}
\hspace*{-0.3cm}\EE[P_{t+1}-P_t] &\leq- \frac{\nu_t}{8\gamma} \EE\|\tilde{\theta}_{t+1} - \theta_t\|^2 - \frac{15L_F^2 \gamma  \nu_t}{2\mu\eta} \EE\|\tilde{\omega}_{t+1} - \omega_t \|^2  - \frac{ L_F^2 \gamma \nu_t}{2}\EE \|\omega_t - \omega^*(\theta_t)\|^2 \\
&\quad + \frac{375 L_F^2 \gamma \nu_t}{8\mu^2} \EE \|d_t - \nabla_\omega F(\theta_t, \omega_t)\|^2  +4\nu_t \gamma \EE \|p_t - \nabla_\theta F(\theta_t, \omega_t) \|^2.
\end{aligned}    
\end{align}

The inequality \eqref{eq:thm_bound_3} shows that we can bound the terms $\EE\|\tilde{\theta}_{t+1} - \theta_t\|^2$ and $\EE\|\omega_t - \omega^*(\theta_t)\|^2$ by moving them from the right-hand side to the left, while the term $\EE[P_{t+1} - P_t]$ becomes a telescoping sum if we taking summation from $t=0$ to $T-1$. In view of the terms $\EE \|d_t - \nabla_\omega F(\theta_t, \omega_t)\|^2$ and $\EE \|p_t - \nabla_\theta F(\theta_t, \omega_t) \|^2$ on the right-hand side of \eqref{eq:thm_bound_3}, we expect to find contraction for the two terms such that the convergence can be guaranteed. By the result of Lemma \ref{lem:grad_var_bound}, and setting 
\begin{align*}
\alpha = 3,    
\end{align*}
 we can have the contraction of the two terms plus some noise terms, which are 
\begin{align}
\begin{aligned}\label{eq:thm_bound_4}
\EE \|\nabla_\theta F(\theta_{t+1},\omega_{t+1})-p_{t+1}\|^2 &\leq  (1- 3\nu_t)  \EE \|\nabla_\theta F(\theta_t,\omega_t)- p_t\|^2 \\
&\quad +  \frac{3\nu_t L_F^2}{8} \EE (\|\tilde{\theta}_{t+1} - \theta_t \|^2 +\|\tilde{\omega}_{t+1}-\omega_t\|^2)  + 9\nu_t^2 \sigma^2,
\end{aligned}
\end{align} 
\vspace{-0.7cm}
\begin{align}
\begin{aligned}\label{eq:thm_bound_5}
\EE \|\nabla_\omega F(\theta_{t+1},\omega_{t+1})-d_{t+1}\|^2 &\leq   (1-3\nu_t)  \EE \|\nabla_\omega F(\theta_t,\omega_t)- d_t\|^2 \\
&\quad + \frac{3\nu_t L_F^2}{8} \EE (\|\tilde{\theta}_{t+1} - \theta_t \|^2 +\|\tilde{\omega}_{t+1}-\omega_t\|^2)  + 9 \nu_t^2 \sigma^2.
\end{aligned}
\end{align}
Multiplying both sides of \eqref{eq:thm_bound_4} and \eqref{eq:thm_bound_5} by $2\gamma/(\mu\eta)$ and combining them with \eqref{eq:thm_bound_3}, by defining the Lyapunov function as 
\begin{align*}
Q_t :=& P_t + \frac{2\gamma}{\mu\eta} \|\nabla_\theta F(\theta_{t},\omega_{t})-p_{t}\|^2 + \frac{2\gamma}{\mu\eta} \|\nabla_\omega F(\theta_{t},\omega_{t})-d_{t}\|^2, \quad \forall t\geq 0,
\end{align*}
we have
\begin{align*}
&\EE[Q_{t+1}-Q_t] \\
&\quad \leq-\Big ( \frac{\nu_t}{8\gamma} -\frac{3\gamma\nu_t L_F^2}{2\mu\eta} \Big )  \EE\|\tilde{\theta}_{t+1} - \theta_t\|^2 - \frac{L_F^2 \gamma \nu_t}{2}\EE \|\omega_t - \omega^*(\theta_t)\|^2- \frac{6L_F^2 \gamma  \nu_t}{\mu\eta}\EE \|\tilde{\omega}_{t+1} - \omega_t\|^2 + \frac{36\sigma^2 \nu_t^2\gamma}{\mu\eta}  \\
&\quad \quad - \Big( \frac{12\nu_t\gamma}{\mu\eta}-4\nu_t \gamma \Big) \EE \|\nabla_\theta F(\theta_t,\omega_t)- p_t\|^2 - \Big ( \frac{12\nu_t\gamma}{\mu\eta}-\frac{375 L_F^2 \gamma \nu_t}{8\mu^2} \Big) \EE \|\nabla_\omega F(\theta_t,\omega_t)- d_t\|^2,
\end{align*}
where the coefficient of the term $\EE \|\tilde{\omega}_{t+1} - \omega_t\|^2$ is by direct computation. According to the conditions of this theorem that $ \eta \leq \mu/(4L_F^2)$ and $\eta \geq 9 L_F^2 \gamma/ \mu^{2}$ with $L_F \geq \mu > 0$, we can simplify the coefficients of the last inequality by
\begin{align*}
\hspace*{-0.1cm}-\Big ( \frac{\nu_t}{8\gamma} -\frac{3\gamma\nu_t L_F^2}{2\mu\eta} \Big ) \leq - \frac{\nu_t}{16\gamma}, ~ - \Big( \frac{12\nu_t\gamma}{\mu\eta}-4\gamma\nu_t \Big) \leq -\frac{4\nu_t\gamma}{\mu\eta}, ~ \text{ and } - \Big ( \frac{12\nu_t\gamma}{\mu\eta}-\frac{375 L_F^2 \gamma \nu_t}{8\mu^2} \Big) < 0.
\end{align*}
Then, we have
\begin{align}
\begin{aligned}\label{eq:thm_bound_al}
\EE[Q_{t+1}-Q_t] &\leq- \frac{\nu_t}{16\gamma}  \EE\|\tilde{\theta}_{t+1} - \theta_t\|^2 - \frac{\gamma \nu_t L_F^2}{2}\EE \|\omega_t - \omega^*(\theta_t)\|^2  \\
& \quad  + \frac{36\sigma^2 \nu_t^2\gamma}{\mu\eta}   - \frac{4\nu_t\gamma}{\mu\eta} \EE \|\nabla_\theta F(\theta_t,\omega_t)- p_t\|^2,   
\end{aligned}
\end{align}
where we drop the term $\EE \|\tilde{\omega}_{t+1} - \omega_t\|^2$ due to its negative coefficient. Note that $J(\theta_t) - J^* \geq 0$ as shown above. Therefore, we have $Q_t \geq 0$ for all $t \geq 0$.

Now we are ready to prove the convergence of Algorithm \ref{alg:single_time} with the inequality \eqref{eq:thm_bound_al}. Taking summation on both sides of \eqref{eq:thm_bound_al} over $t=0,\ldots, T-1$ and rearranging the terms lead to 
\begin{align*}
&\sum_{t=0}^{T-1} \frac{\nu_t \gamma}{16} \left(\frac{1}{\gamma^2}\EE \|\tilde{\theta}_{t+1} - \theta_t\|^2 + \frac{64}{\mu\eta} \EE \|p_t - \nabla_\theta F(\theta_t, \omega_t)\|^2 + 8 L_F^2 \EE \|\omega_t - \omega^*(\theta_t)\|^2 \right)\\
&\qquad \leq \frac{36  \sigma^2 \gamma \sum_{t=0}^{T-1}\nu_t^2}{\mu\eta} + \EE [Q_0 - Q_T] \leq \frac{36  \sigma^2 \gamma \sum_{t=0}^{T-1}\nu_t^2}{\mu\eta} + Q_0,
\end{align*}
where we use the fact that $Q_T \geq 0$. Letting $\{\nu_t\}_{t\geq 0}$ be a non-increasing sequence, we know $\nu_t \geq \nu_T$ for any $0 \leq  t \leq T$.  Since we have $ 1/(\mu \eta) \geq 1$ due to the conditions for the values of $\eta$ and $\gamma$, we obtain
\begin{align*}
&\frac{\nu_T \gamma}{16} \sum_{t=0}^{T-1}  \left(\frac{1}{\gamma^2}\EE \|\tilde{\theta}_{t+1} - \theta_t\|^2 + \EE \|p_t - \nabla_\theta F(\theta_t, \omega_t)\|^2 + L_F^2 \EE \|\omega_t - \omega^*(\theta_t)\|^2 \right)\\
&\qquad \leq \sum_{t=0}^{T-1} \frac{\nu_t \gamma}{16} \left(\frac{1}{\gamma^2}\EE \|\tilde{\theta}_{t+1} - \theta_t\|^2 + \frac{64}{\mu\eta} \EE \|p_t - \nabla_\theta F(\theta_t, \omega_t)\|^2 + 8 L_F^2 \EE \|\omega_t - \omega^*(\theta_t)\|^2 \right)\\
&\qquad  \leq \frac{36  \sigma^2 \gamma \sum_{t=0}^{T-1}\nu_t^2}{\mu\eta} + Q_0.
\end{align*}
Multiplying both sides by $16/(T\nu_T \gamma)$ yields
\begin{align}
\begin{aligned}\label{eq:last_bound_al}
&\frac{1}{T} \sum_{t=0}^{T-1}\left (\frac{1}{\gamma^2}\EE \|\tilde{\theta}_{t+1} - \theta_t\|^2 + \EE \|p_t - \nabla_\theta F(\theta_t, \omega_t)\|^2 + L_F^2\EE \|\omega_t - \omega^*(\theta_t)\|^2 \right) \\
&\qquad \leq \frac{36 \times 16   \sigma^2 }{\mu\eta T \nu_T  } \sum_{t=0}^{T-1} \nu_t^2 + \frac{16Q_0}{\gamma \nu_T T}. 
\end{aligned}
\end{align}
According to the setting of the step size that $\nu_t = a/(t+b)^{1/2}$ with $a = 1/16$ and $b \geq  \max \{ (\gamma L_F^2/\mu)^2, 3 \}$, we can bound the two terms on the right-hand side of \eqref{eq:last_bound_al} in the following way
\begin{align}
\begin{aligned}
\label{eq:last_bound_al_1}
\frac{16Q_0}{\gamma \nu_T T}  + \frac{36 \times 16   \sigma^2 }{\mu\eta T \nu_T  } \sum_{t=0}^{T-1} \nu_t^2  &\leq \frac{256 Q_0 (T + b)^{1/2}}{\gamma  T} +   \frac{36\sigma^2}{\mu\eta T } \int_{0}^T \frac{1}{t+b} \mathrm{d} t \\
&\leq \frac{256 Q_0 (T + b)^{1/2}}{\gamma  T} +   \frac{36 \sigma^2}{\mu\eta T }  (T+b)^{1/2} \log (T+b)\\
&\leq \Big( \frac{256 Q_0 }{\gamma  } +   \frac{36 \sigma^2}{\mu\eta }   \Big ) \Big ( \frac{1}{\sqrt{T}} + \frac{\sqrt{b}}{T} \Big ) \log (T+b),
\end{aligned}
\end{align}
where the last inequality is due to $\sqrt{x+y}\leq \sqrt{x} + \sqrt{y}$ for $x,y \geq 0$. 

Finally, we combine \eqref{eq:last_bound_al} and \eqref{eq:last_bound_al_1} and obtain
\begin{align*}
&\frac{1}{T} \sum_{t=0}^{T-1} \left (\frac{1}{\gamma^2}\EE \|\tilde{\theta}_{t+1} - \theta_t\|^2 + \EE \|p_t - \nabla_\theta F(\theta_t, \omega_t)\|^2 + L_F^2 \EE \|\omega_t - \omega^*(\theta_t)\|^2 \right) \\
&\qquad \leq \frac{\overline{C}_1  \log (T+b) }{\sqrt{T}} + \frac{\overline{C}_2  \log (T+b)}{T},
\end{align*}
where $\overline{C}_1 = 256 Q_0/\gamma +  36 \sigma^2/(\mu\eta)$ and $\overline{C}_2 = \overline{C}_1 \sqrt{b}$. 

Moreover, by Jensen's inequality, there is
\begin{align*}
&\frac{1}{T} \sum_{t=0}^{T-1} \EE \left (\frac{1}{\gamma}\|\tilde{\theta}_{t+1} - \theta_t\| + \|p_t - \nabla_\theta F(\theta_t, \omega_t)\| + L_F  \|\omega_t - \omega^*(\theta_t)\| \right ) \\
&\qquad \leq  \left [ \frac{3}{T} \sum_{t=0}^{T-1} \EE \left(\frac{1}{\gamma^2}\|\tilde{\theta}_{t+1} - \theta_t\|^2 +  \|p_t - \nabla_\theta F(\theta_t, \omega_t)\|^2 + L_F^2  \|\omega_t - \omega^*(\theta_t)\|^2 \right ) \right]^{1/2}.
\end{align*}
Thus, we eventually obtain
\begin{align*}
&\frac{1}{T} \sum_{t=0}^{T-1} \EE \left (\frac{1}{\gamma} \|\tilde{\theta}_{t+1} - \theta_t\| + \|p_t - \nabla_\theta F(\theta_t, \omega_t)\| + L_F \|\omega_t - \omega^*(\theta_t)\| \right) \\
&\qquad \leq \frac{C_1  \log (T+b) }{T^{1/4}} + \frac{C_2 \log (T+b) }{\sqrt{T}}  = \widetilde{O}\left( \frac{1}{T^{1/4}} \right), 
\end{align*}
where $C_1 = [768 Q_0/\gamma +  108 \sigma^2/(\mu\eta)]^{1/2}$ and $C_2 = C_1 b^{1/4}$. This completes the proof.
\end{proof}

\subsection{Proof of Theorem \ref{thm:single_time_fix}}\label{sec:detail_proof_single_time_fix}

\begin{proof} Our analysis in Section \ref{sec:detail_proof_single_time} is for non-increasing step size $\nu_t$ satisfying the conditions in Theorem \ref{thm:single_time} which are the same as in this theorem. Thus, the proof before \eqref{eq:last_bound_al} can be adapted here for a fixed step size. Thus, we start our proof directly from \eqref{eq:last_bound_al} in Section \ref{sec:detail_proof_single_time} with replacing $\nu_t$ by the fixed step size $\nu$, which can be rewritten as
\begin{align}
\begin{aligned}
&\frac{1}{T} \sum_{t=0}^{T-1}\left (\frac{1}{\gamma^2}\EE \|\tilde{\theta}_{t+1} - \theta_t\|^2 + \EE \|p_t - \nabla_\theta F(\theta_t, \omega_t)\|^2 + L_F^2\EE \|\omega_t - \omega^*(\theta_t)\|^2 \right)  \leq \frac{36 \times 16   \sigma^2 }{\mu\eta }  \nu + \frac{16Q_0}{\gamma \nu T}.
\end{aligned}
\end{align}
According to the setting of the step size that $\nu = 1/[16(T+b)^{1/2}]$, we can bound the right-hand side as follows
\begin{align}
\begin{aligned}
\frac{36 \times 16   \sigma^2 }{\mu\eta }  \nu + \frac{16Q_0}{\gamma \nu T}  &\leq \frac{36\sigma^2}{\mu\eta (T+b)^{1/2}} + \frac{256 Q_0 (T + b)^{1/2}}{\gamma  T} \\
&\leq \frac{256 Q_0 }{\gamma  T^{1/2}} +  \frac{36\sigma^2}{\mu\eta (T+b)^{1/2}} + \frac{256 Q_0 b^{1/2}}{\gamma  T},
\end{aligned}
\end{align}
where the last inequality is due to $\sqrt{x+y}\leq \sqrt{x} + \sqrt{y}$ for $x,y \geq 0$. This inequality shows that 
\begin{align*}
&\frac{1}{T} \sum_{t=0}^{T-1} \left (\frac{1}{\gamma^2}\EE \|\tilde{\theta}_{t+1} - \theta_t\|^2 + \EE \|p_t - \nabla_\theta F(\theta_t, \omega_t)\|^2 + L_F^2 \EE \|\omega_t - \omega^*(\theta_t)\|^2 \right) \\
&\qquad \leq \frac{256 Q_0 }{\gamma  T^{1/2}} +   \frac{36\sigma^2}{\mu\eta (T+b)^{1/2}} + \frac{256 Q_0 b^{1/2}}{\gamma  T}.
\end{align*}
Furthermore, by Jensen's inequality, we have
\begin{align*}
&\frac{1}{T} \sum_{t=0}^{T-1} \EE \left (\frac{1}{\gamma}\|\tilde{\theta}_{t+1} - \theta_t\| + \|p_t - \nabla_\theta F(\theta_t, \omega_t)\| + L_F  \|\omega_t - \omega^*(\theta_t)\| \right ) \\
&\qquad \leq  \left [ \frac{3}{T} \sum_{t=0}^{T-1} \EE \left(\frac{1}{\gamma^2}\|\tilde{\theta}_{t+1} - \theta_t\|^2 +  \|p_t - \nabla_\theta F(\theta_t, \omega_t)\|^2 + L_F^2  \|\omega_t - \omega^*(\theta_t)\|^2 \right ) \right]^{1/2}.
\end{align*}
Thus, we eventually obtain
\begin{align*}
\frac{1}{T} \sum_{t=0}^{T-1} \EE \left (\frac{1}{\gamma} \|\tilde{\theta}_{t+1} - \theta_t\| + \|p_t - \nabla_\theta F(\theta_t, \omega_t)\| + L_F \|\omega_t - \omega^*(\theta_t)\| \right)  \leq O\left( \frac{1}{T^{1/4}} \right), 
\end{align*}
which leads to $T_{\varepsilon} \geq O(\varepsilon^{-4})$ sample complexity to achieve an $\varepsilon$ error. This completes the proof.
\end{proof}

%% file: vr_proof.tex
\section{Proofs for Algorithm \ref{alg:variance_reduce}} \label{sec:proof_variance_reduce}

First, we provide three lemmas without proofs, i.e., Lemma \ref{lem:vr_primal_J_bound}, Lemma \ref{lem:vr_omega_contract}, and Lemma \ref{lem:vr_decomp_omega_opt}, which are modified a little from Lemmas \ref{lem:primal_J_bound}, \ref{lem:omega_contract}, and \ref{lem:decomp_omega_opt} respectively. Specifically, the first three lemmas are only associated with the updating rules of $\theta_{t+1}$ and $\omega_{t+1}$ which are the same in both Algorithm \ref{alg:single_time} and \ref{alg:variance_reduce}. The difference lies in the Lipschitz gradient assumptions that are used in the lemmas, where we replace Assumption \ref{assump:lip_grad} with Assumption \ref{assump:lip_sto_grad}. The proofs of the following three lemmas follows exactly the proofs of Lemmas \ref{lem:primal_J_bound}, \ref{lem:omega_contract}, and \ref{lem:decomp_omega_opt}. 

After the first three lemmas, we then establish Lemma \ref{lem:vr_grad_var_bound} with a detailed proof, which is related to the gradient approximation with variance reduction.

\begin{lemma} \label{lem:vr_primal_J_bound} Under Assumptions \ref{assump:fea_set}, \ref{assump:bounded_var}, and \ref{assump:lip_sto_grad}, letting $0 < \gamma \nu_t \leq \mu/(16L_f^2)$ and $\nu_t \leq 1 $, with the updating rules shown in Algorithm \ref{alg:variance_reduce}, we have 
\begin{align*}
J(\theta_{t+1})-J(\theta_t)  & \leq -  \frac{3\nu_t}{4\gamma} \|\tilde{\theta}_{t+1} - \theta_t \|^2  + 2L_f^2 \gamma \nu_t \|\omega_t - \omega(\theta_t)^*\|^2 + 4\gamma\nu_t  \|\nabla_\theta F(\theta_t, \omega_t) - p_t\|^2,
\end{align*}
where $J(\theta)=\max_{\omega\in \Omega}F(\theta, \omega)$ and $\omega^*(\theta):=\argmax_{\omega \in \Omega} F(\theta, \omega)$.
\end{lemma}
\begin{proof} The proof of this lemma follows the proof of Lemma \ref{lem:primal_J_bound}.
\end{proof}

\begin{lemma}\label{lem:vr_omega_contract}
Under Assumptions \ref{assump:fea_set}, \ref{assump:non_singular} and \ref{assump:lip_sto_grad}, letting $0< \nu_t \leq 1/8$ and $0 < \eta \leq (4L_f)^{-1}$, with the updating rules shown in Algorithm \ref{alg:variance_reduce}, we have 
\begin{align*}
\hspace*{-0.15cm}\|\omega_{t+1} - \omega^*(\theta_t)\|^2 \leq \Big(1- \frac{\nu_t \eta \mu}{2} \Big)\|\omega_t-  \omega^*(\theta_t)\|^2  - \frac{3\nu_t}{4} \|\tilde{\omega}_{t+1}-\omega_t \|^2 + \frac{4 \eta\nu_t}{\mu}\| \nabla_\omega F(\theta_t,\omega_t)-d_t\|^2,
\end{align*}
where $\omega^*(\theta_t)=\argmax_{\omega \in \Omega} F(\theta_t, \omega)$.
\end{lemma}
\begin{proof} The proof of this lemma follows the proof of Lemma \ref{lem:omega_contract}.
\end{proof}

\begin{lemma} \label{lem:vr_decomp_omega_opt} Under Assumptions \ref{assump:fea_set}, \ref{assump:non_singular} and \ref{assump:lip_sto_grad}, letting $0< \nu_t \leq 1/8$ and $0 < \eta \leq (4L_f)^{-1}$, with the updating rules shown in Algorithm \ref{alg:variance_reduce}, we have 
\begin{align*}
\|\omega_{t+1} - \omega^*(\theta_{t+1})\|^2 &\leq \Big(1-\frac{\mu\eta\nu_t}{4}\Big) \|\omega_t - \omega^*(\theta_t)\|^2 - \frac{3\nu_t}{4}\|\tilde{\omega}_{t+1} - \omega_t \|^2 \\
&\quad + \frac{75\eta \nu_t}{16\mu} \|d_t - \nabla_\omega F(\theta_t, \omega_t)\|^2 \nonumber  + \frac{75L_\omega^2 \nu_t }{16\mu\eta}  \|\tilde{\theta}_{t+1}-\theta_t\|^2,
\end{align*}
where $\omega^*(\theta_t)=\argmax_{\omega \in \Omega} F(\theta_t, \omega)$ and $\omega^*(\theta_{t+1})=\argmax_{\omega \in \Omega} F(\theta_{t+1}, \omega)$.
\end{lemma}
\begin{proof} The proof of this lemma follows the proof of Lemma \ref{lem:decomp_omega_opt}.
\end{proof}

\begin{lemma} \label{lem:vr_grad_var_bound}  Under Assumptions \ref{assump:fea_set}, \ref{assump:bounded_var}, and \ref{assump:lip_sto_grad}, letting $0< \nu_t \leq  \alpha^{-1/2}$, with the updating rules shown in Algorithm \ref{alg:variance_reduce}, the following inequalities hold 
\begin{align}
\begin{aligned} \label{eq:vr_grad_var_bound_1}
\EE \|\nabla_\theta F(\theta_{t+1},\omega_{t+1})-p_{t+1}\|^2 &\leq   (1- \alpha \nu^2_t) \EE\|\nabla_\theta F(\theta_t,\omega_t)-p_t\|^2 +  2 \alpha^2 \nu_t^4 \sigma^2 \\
&\quad  +  2 L_f^2 \nu_t^2 \EE[\|\tilde{\theta}_{t+1}-\theta_t\|^2 + \|\tilde{\omega}_{t+1}-\omega_t\|^2],
\end{aligned}
\end{align}
\begin{align}
\begin{aligned}\label{eq:vr_grad_var_bound_2}
\EE \|\nabla_\omega F(\theta_{t+1},\omega_{t+1})-d_{t+1}\|^2 &\leq   (1- \alpha \nu^2_t) \EE\|\nabla_\omega F(\theta_t,\omega_t)-d_t\|^2 +  2 \alpha^2 \nu_t^4 \sigma^2 \\
&\quad  +  2  L_f^2 \nu_t^2 \EE[\|\tilde{\theta}_{t+1}-\theta_t\|^2 + \|\tilde{\omega}_{t+1}-\omega_t\|^2].
\end{aligned}
\end{align}  

\end{lemma}

\begin{proof} In this proof, we first show the proof for the inequality \eqref{eq:vr_grad_var_bound_1} in detail. Then, we apply a similar analysis for the second inequality \eqref{eq:vr_grad_var_bound_2}, for which we only present a proof sketch. 

We start the proof by decomposing the term $\nabla_\theta F(\theta_{t+1},\omega_{t+1})-p_{t+1}$ on the left-hand side of \eqref{eq:vr_grad_var_bound_1} as follows:
\begin{align*}
&\nabla_\theta F(\theta_{t+1},\omega_{t+1})-p_{t+1}\\ 
&\qquad= \nabla_\theta F(\theta_{t+1},\omega_{t+1})-  (1- \alpha \nu^2_t) (p_t-\nabla_\theta f(\theta_t,\omega_t; \xi_{t+1})) -  \nabla_\theta f(\theta_{t+1},\omega_{t+1}; \xi_{t+1})\\
&\qquad= \alpha \nu_t^2 [\nabla_\theta F(\theta_{t+1},\omega_{t+1}) - \nabla_\theta f(\theta_{t+1},\omega_{t+1}; \xi_{t+1})] +  (1- \alpha \nu^2_t) [\nabla_\theta F(\theta_t,\omega_t)-p_t] \\
&\qquad \quad + (1-\alpha \nu_t^2)\{[\nabla_\theta F(\theta_{t+1},\omega_{t+1}) - \nabla_\theta F(\theta_t,\omega_t) ] - [\nabla_\theta f(\theta_{t+1},\omega_{t+1};\xi_{t+1}) - \nabla_\theta f(\theta_t,\omega_t;\xi_{t+1}) ]  \},
\end{align*}
where we apply the updating rule for $p_{t+1}$ in Algorithm \ref{alg:variance_reduce}, which is
\begin{align*}
p_{t+1} = (1- \alpha \nu^2_t) (p_t-\nabla_\theta f(\theta_t,\omega_t; \xi_{t+1})) +  \nabla_\theta f(\theta_{t+1},\omega_{t+1}; \xi_{t+1}).
\end{align*}
Therefore, the expectation of its norm can be decomposed as
\begin{align}
\begin{aligned}\label{eq:vr_grad_var_bound_al}
&\hspace*{-0.2cm}\EE \|\nabla_\theta F(\theta_{t+1},\omega_{t+1})-p_{t+1}\|^2\\ 
&\hspace*{-0.2cm} =  (1- \alpha \nu^2_t)^2\EE\|\nabla_\theta F(\theta_t,\omega_t)-p_t\|^2 + \underbrace{\EE \big \| \alpha \nu_t^2 [\nabla_\theta F(\theta_{t+1},\omega_{t+1}) - \nabla_\theta f(\theta_{t+1},\omega_{t+1}; \xi_{t+1})] \cdots }_{\textcircled{1} \cdots} \\
&\hspace*{-0.2cm}~~\underbrace{  + (1-\alpha \nu_t^2)\{[\nabla_\theta F(\theta_{t+1},\omega_{t+1}) - \nabla_\theta F(\theta_t,\omega_t) ] - [\nabla_\theta f(\theta_{t+1},\omega_{t+1};\xi_{t+1}) - \nabla_\theta f(\theta_t,\omega_t;\xi_{t+1}) ] \} \big \|^2}_{\cdots \textcircled{1}} \\
&\hspace*{-0.2cm}~~ + \underbrace{ 2 \EE \big \langle (1- \alpha \nu^2_t) [\nabla_\theta F(\theta_t,\omega_t)-p_t ], \ \   \alpha \nu_t^2 [\nabla_\theta F(\theta_{t+1},\omega_{t+1}) - \nabla_\theta f(\theta_{t+1},\omega_{t+1}; \xi_{t+1})] \cdots}_{\textcircled{2} \cdots}\\
&\hspace*{-0.2cm}~~ \underbrace{ + (1-\alpha \nu_t^2)\{[\nabla_\theta F(\theta_{t+1},\omega_{t+1}) -  \nabla_\theta F(\theta_t,\omega_t) ] - [\nabla_\theta f(\theta_{t+1},\omega_{t+1};\xi_{t+1}) - \nabla_\theta f(\theta_t,\omega_t;\xi_{t+1})]  \} \big \rangle}_{\cdots \textcircled{2} },
\end{aligned}
\end{align}
Next, we bound the terms \textcircled{1} and \textcircled{2} respectively. We first consider to bound the term \textcircled{2}. According to the chain rule of expectation, term \textcircled{2} can be equivalently written as $\textcircled{2} = \EE\big [\EE_{\xi_{t+1}} \big\langle \textcircled{3}, \textcircled{4} \big \rangle \big]$, where we replace the long terms in the inner product with \textcircled{3} and \textcircled{4} for abbreviation. Thus, due to 
\begin{align*}
&\EE_{\xi_{t+1}} \nabla_\theta f(\theta_{t+1},\omega_{t+1}; \xi_{t+1}) = \nabla_\theta F(\theta_{t+1},\omega_{t+1}),
\end{align*}
as well as
\begin{align*}
\EE_{\xi_{t+1}}[\nabla_\theta f(\theta_{t+1},\omega_{t+1};\xi_{t+1}) - \nabla_\theta f(\theta_t,\omega_t;\xi_{t+1})] = \nabla_\theta F(\theta_{t+1},\omega_{t+1}) -  \nabla_\theta F(\theta_t,\omega_t) ,
\end{align*}
we have 
\begin{align} \label{eq:vr_grad_var_prod}
\textcircled{2} = \EE\big[\EE_{\xi_{t+1}} \big\langle \textcircled{3},\  \textcircled{4} \big\rangle \big] = \EE\big[ \big\langle \textcircled{3},\  0  \big\rangle \big] = 0. 
\end{align}
Then, we bound the term \textcircled{1} in \eqref{eq:vr_grad_var_bound_al}. By the inequality $\|x+y\|^2 \leq 2\|x\|^2 + 2\|y\|^2$, we can decompose the term \textcircled{1} as
\begin{align*}
 \textcircled{1} &\leq 2 \alpha^2 \nu_t^4 \EE  \| \nabla_\theta F(\theta_{t+1},\omega_{t+1}) - \nabla_\theta f(\theta_{t+1},\omega_{t+1}; \xi_{t+1}) \|^2 \\
&\quad +  2(1-\alpha \nu_t^2)^2 \EE \|\nabla_\theta f(\theta_{t+1},\omega_{t+1};\xi_{t+1}) - \nabla_\theta f(\theta_t,\omega_t;\xi_{t+1})  \|^2 \\
&\leq 2 \alpha^2 \nu_t^4 \sigma^2 +  2(1-\alpha \nu_t^2)^2 \EE \|\nabla_\theta f(\theta_{t+1},\omega_{t+1};\xi_{t+1}) - \nabla_\theta f(\theta_t,\omega_t;\xi_{t+1})  \|^2,
\end{align*}
where the second inequality is due to the bounded variance assumption in Assumption \ref{assump:bounded_var}. Additionally, we bound the term $\EE\|\nabla_\theta f(\theta_{t+1},\omega_{t+1};\xi_{t+1}) - \nabla_\theta f(\theta_t,\omega_t;\xi_{t+1})\|$ above as
\begin{align*}
&\EE\|\nabla_\theta f(\theta_{t+1},\omega_{t+1};\xi_{t+1}) - \nabla_\theta f(\theta_t,\omega_t;\xi_{t+1})\| \\
&\qquad \leq L_f^2 \EE[ \|\theta_{t+1}-\theta_t\|^2 + \|\omega_{t+1}-\omega_t\|^2] \\
&\qquad \leq L_f^2 \nu_t^2 \EE[\|\tilde{\theta}_{t+1}-\theta_t\|^2 + \|\tilde{\omega}_{t+1}-\omega_t\|^2],
\end{align*}
where the first inequality is by the Lipschitz continuity of the stochastic gradient in Assumption \ref{assump:lip_sto_grad} and the second inequality is by the updating rules $\theta_{t+1} = \theta_t + \nu_t (\tilde{\theta}_{t+1} - \theta_t)$ and $\omega_{t+1} = \omega_t + \nu_t (\tilde{\omega}_{t+1} - \omega_t)$ in Algorithm \ref{alg:variance_reduce}. Thus, we can bound \textcircled{1} by
\begin{align} \label{eq:vr_grad_var_norm}
 \textcircled{1} \leq 2 \alpha^2 \nu_t^4 \sigma^2 +  2(1-\alpha \nu_t^2)^2 L_f^2 \nu_t^2 \EE[\|\tilde{\theta}_{t+1}-\theta_t\|^2 + \|\tilde{\omega}_{t+1}-\omega_t\|^2].
\end{align}
Now combining the inequalities \eqref{eq:vr_grad_var_bound_al}, \eqref{eq:vr_grad_var_prod} and \eqref{eq:vr_grad_var_norm} and due to $(1-\alpha \nu_t^2) \leq 1$, then we obtain
\begin{align*}
\EE \|\nabla_\theta F(\theta_{t+1},\omega_{t+1})-p_{t+1}\|^2 &\leq   (1- \alpha \nu^2_t)\EE\|\nabla_\theta F(\theta_t,\omega_t)-p_t\|^2 +  2 \alpha^2 \nu_t^4 \sigma^2 \\
&\quad  +  2 L_f^2 \nu_t^2 \EE[\|\tilde{\theta}_{t+1}-\theta_t\|^2 + \|\tilde{\omega}_{t+1}-\omega_t\|^2],
\end{align*}
which completes the proof of the first inequality \eqref{eq:vr_grad_var_bound_1} in this lemma.

We apply a similar analysis as above to prove the second inequality \eqref{eq:grad_var_bound_2} in this lemma. We also have a similar decomposition
\begin{align*}
&\nabla_\omega F(\theta_{t+1},\omega_{t+1})-d_{t+1}\\ 
&\qquad = \nabla_\omega F(\theta_{t+1},\omega_{t+1})-  (1- \alpha \nu^2_t) (d_t-\nabla_\omega f(\theta_t,\omega_t; \xi_{t+1})) -  \nabla_\omega f(\theta_{t+1},\omega_{t+1}; \xi_{t+1})\\
&\qquad = \alpha \nu_t^2 [\nabla_\omega F(\theta_{t+1},\omega_{t+1}) - \nabla_\omega f(\theta_{t+1},\omega_{t+1}; \xi_{t+1})] +  (1- \alpha \nu^2_t) [\nabla_\omega F(\theta_t,\omega_t)-d_t] \\
&\qquad \quad + (1-\alpha \nu_t^2)\{[\nabla_\omega F(\theta_{t+1},\omega_{t+1}) - \nabla_\omega F(\theta_t,\omega_t) ] - [\nabla_\omega f(\theta_{t+1},\omega_{t+1};\xi_{t+1}) - \nabla_\omega f(\theta_t,\omega_t;\xi_{t+1}) ]  \},
\end{align*}
where we apply the updating rule for $d_{t+1}$ in Algorithm \ref{alg:variance_reduce}, which is
\begin{align*}
d_{t+1} = (1- \alpha \nu^2_t) (d_t-\nabla_\omega f(\theta_t,\omega_t; \xi_{t+1})) +  \nabla_\omega f(\theta_{t+1},\omega_{t+1}; \xi_{t+1}).
\end{align*}
Therefore, the expectation of its norm can be decomposed as
\begin{align*}
&\EE \|\nabla_\omega F(\theta_{t+1},\omega_{t+1})-d_{t+1}\|^2\\ 
&\qquad =  (1- \alpha \nu^2_t)^2\EE\|\nabla_\omega F(\theta_t,\omega_t)-d_t\|^2 + \EE \big \| \alpha \nu_t^2 [\nabla_\omega F(\theta_{t+1},\omega_{t+1}) - \nabla_\omega f(\theta_{t+1},\omega_{t+1}; \xi_{t+1})]\\
&\qquad\quad  + (1-\alpha \nu_t^2)\{[\nabla_\omega F(\theta_{t+1},\omega_{t+1}) - \nabla_\omega F(\theta_t,\omega_t) ] - [\nabla_\omega f(\theta_{t+1},\omega_{t+1};\xi_{t+1}) - \nabla_\omega f(\theta_t,\omega_t;\xi_{t+1}) ] \} \big \|^2  \\
&\qquad\quad + 2 \EE \big \langle (1- \alpha \nu^2_t) [\nabla_\omega F(\theta_t,\omega_t)-d_t ], \ \   \alpha \nu_t^2 [\nabla_\omega F(\theta_{t+1},\omega_{t+1}) - \nabla_\omega f(\theta_{t+1},\omega_{t+1}; \xi_{t+1})] \\
&\qquad\quad + (1-\alpha \nu_t^2)\{[\nabla_\omega F(\theta_{t+1},\omega_{t+1}) -  \nabla_\omega F(\theta_t,\omega_t) ] - [\nabla_\omega f(\theta_{t+1},\omega_{t+1};\xi_{t+1}) - \nabla_\omega f(\theta_t,\omega_t;\xi_{t+1})]  \} \big \rangle.
\end{align*}
By proving the similarly bounds as \eqref{eq:vr_grad_var_prod} and \eqref{eq:vr_grad_var_norm} and applying them to the above inequality, we eventually have
\begin{align*}
\EE \|\nabla_\omega F(\theta_{t+1},\omega_{t+1})-d_{t+1}\|^2 &\leq   (1- \alpha \nu^2_t)\EE\|\nabla_\omega F(\theta_t,\omega_t)-d_t\|^2 +  2 \alpha^2 \nu_t^4 \sigma^2 \\
&\quad  +  2 L_f^2 \nu_t^2 \EE[\|\tilde{\theta}_{t+1}-\theta_t\|^2 + \|\tilde{\omega}_{t+1}-\omega_t\|^2].
\end{align*}
This completes the proof.
\end{proof}

\subsection{Proof of Theorem \ref{thm:variance_reduce}} \label{sec:detail_proof_variance_reduce}

%\begin{theorem}[Theorem \ref{thm:variance_reduce}] Under Assumptions \ref{assump:exist_sol}, \ref{assump:fea_set}, \ref{assump:non_singular}, and \ref{assump:lip_sto_grad}, setting the parameters $\alpha =  6$, $0< \eta \leq \mu /(6 L_f^2)$, $0 < \gamma \leq  \mu^2 \eta /(9  L_f^2)$, and $\nu_t = a/(t+b)^{1/3}$ with $a = 1/3$ and $b \geq  \max \{ [6\gamma (L_f + L_f^2/\mu)]^3, 256 \}$, with the updating rules in Algorithm \ref{alg:variance_reduce}, the convergence rate of this algorithm is
%\begin{align*}
%\frac{1}{T} \sum_{t=0}^{T-1} \Big [\frac{1}{\gamma} \EE \|\nabla_\theta F(\theta_t,\omega_t)-p_t\| + \EE \|\tilde{\theta}_{t+1} - \theta_t\|  + L_f \EE \|\omega_t - \omega^*(\theta_t)\| \Big ] \leq \frac{C_1}{T^{1/3}} + \frac{C_2}{\sqrt{T}},
%\end{align*}
%where $C_1 = [150 S_0/\gamma +  1536 \sigma^2/(\mu\eta)]^{1/2}$, $C_2 = C_1 b^{1/6}$, and $S_0 = J(\theta_0) - J^* + 10L_f^2 \gamma/(\mu\eta) \|\omega_0 - \omega^*(\theta_0)\|^2 + 2\gamma/(\mu\eta\nu_0) (\|\nabla_\omega F(\theta_0,\omega_0)-d_0\|^2 + \|\nabla_\theta F(\theta_0,\omega_0)-p_0\|^2).$
%\end{theorem}

\begin{proof} We assume that the step size is of the form $\nu_t = a/(t+b)^{1/3}$ where $a = 1/3$. We interpret the parameter settings in Theorem \ref{thm:variance_reduce} as follows:
\begin{align}
\eta \leq \frac{\mu}{6L_f^2} < \frac{1}{4L_f}, \quad \text{ and } \nu_t \leq \frac{a}{b^{1/3}} \leq \min \Big \{ \frac{1}{19},\ \  \frac{\mu}{18\gamma L_f^2} \Big \}, \label{eq:vr_nu_cond}
\end{align}
with $L_f \geq \mu > 0$. Thus, we can apply Lemmas \ref{lem:vr_primal_J_bound}, \ref{lem:vr_decomp_omega_opt}, and \ref{lem:vr_grad_var_bound} in the following proof of Algorithm \ref{alg:variance_reduce}. Then, we proceed to the main proof.

By Lemma \ref{lem:vr_primal_J_bound}, taking expectation on both sides, we have 
\begin{align}
\begin{aligned}\label{eq:vr_thm_bound_1}
\EE J(\theta_{t+1})- \EE J(\theta_t)  & \leq -  \frac{3\nu_t}{4\gamma} \EE\|\tilde{\theta}_{t+1} - \theta_t \|^2  + 2L_f^2 \gamma \nu_t \EE\|\omega_t - \omega^*(\theta_t)\|^2 \\
&\quad + 4\nu_t \gamma \EE \|\nabla_\theta F(\theta_t, \omega_t) - p_t\|^2,
\end{aligned}
\end{align}
where we can have the upper bound of $\EE\| \tilde{\theta}_{t+1} - \theta_t\|^2$ by moving it to the left-hand side.  The above inequality indicates that we need to further establish upper bounds of the terms $\EE \|\omega_t - \omega^*(\theta_t)\|^2$ and $\EE\|p_t - \nabla_\theta F(\theta_t, \omega_t)\|^2$.

According to Lemma \ref{lem:vr_decomp_omega_opt}, we have the following inequality providing upper bounds for the term  $\EE \|\omega_{t+1} - \omega^*(\theta_{t+1})\|^2$, which is
\begin{align*}
\EE\|\omega_{t+1} - \omega^*(\theta_{t+1})\|^2 &\leq \Big(1-\frac{\mu\eta\nu_t}{4}\Big) \EE \|\omega_t - \omega^*(\theta_t)\|^2 - \frac{3\nu_t}{4}\EE\|\tilde{\omega}_{t+1} - \omega_t \|^2 \\
&\quad + \frac{75\eta \nu_t}{16\mu} \EE\|d_t - \nabla F(\theta_t, \omega_t)\|^2 \nonumber  + \frac{75L_\omega^2 \nu_t }{16\mu\eta} \EE \|\tilde{\theta}_{t+1}-\theta_t\|^2.
\end{align*}
which shows a contraction of $\EE \|\omega_{t+1} - \omega^*(\theta_{t+1})\|^2$ plus some noise terms as well as a deduction of the term $\EE\|\tilde{\omega}_{t+1} - \omega_t\|^2$ such that this term appearing in any later inequalities can be eliminated in the end. 

We further apply similar derivation for \eqref{eq:thm_bound_2} in the proof of Theorem \ref{thm:single_time} to the above inequality. Multiplying both sides of the above inequality by $10L_f^2 \gamma /(\mu\eta)$ and rearranging terms, we obtain
\begin{align}
\begin{aligned}\label{eq:vr_thm_bound_2}
&\frac{10L_f^2 \gamma}{\mu\eta} \big( \EE \|\omega_{t+1} - \omega^*(\theta_{t+1})\|^2 - \EE \|\omega_t - \omega^*(\theta_t)\|^2 \big) \\
&\qquad\leq -\frac{5L_f^2 \gamma \nu_t}{2} \EE \|\omega_t - \omega^*(\theta_t)\|^2 - \frac{15L_f^2  \nu_t \gamma }{2\mu\eta} \EE\|\tilde{\omega}_{t+1} - \omega_t \|^2\\
&\qquad \quad  + \frac{375 L_f^2  \nu_t \gamma }{8\mu^2} \EE \|d_t - \nabla F(\theta_t, \omega_t)\|^2  + \frac{375 L_f^2 L_\omega^2 \nu_t \gamma }{8\mu^2\eta^2}  \EE \|\tilde{\theta}_{t+1}-\theta_t\|^2.
\end{aligned}
\end{align}
Moreover, we define 
\begin{align*}
R_t := J(\theta_t) - J^* +  \frac{10L_f^2 \gamma}{\mu\eta} \|\omega_t - \omega^*(\theta_t)\|^2 , \quad  \forall t\geq 0,
\end{align*}
such that $R_t \geq 0$ with $J^*>-\infty$ being the minimal value of $J(\theta)$. Then, summing both sides of the two inequalities \eqref{eq:vr_thm_bound_1} and \eqref{eq:vr_thm_bound_2}, we have
\begin{align*}
&\EE[R_{t+1}-R_t] \\
&\qquad \leq- \Big ( \frac{3\nu_t}{4\gamma} - \frac{375 \nu_t L_f^2 L_\omega^2 \gamma}{8\mu^2 \eta^2}\Big) \EE\|\tilde{\theta}_{t+1} - \theta_t\|^2 - \frac{15 L_f^2 \nu_t \gamma}{2\mu\eta} \EE\|\tilde{\omega}_{t+1} - \omega_t \|^2 -\frac{L_f^2 \gamma \nu_t}{2} \EE \|\omega_t - \omega^*(\theta_t)\|^2 \\
&\qquad \quad + \frac{375 \nu_t L_f^2 \gamma}{8\mu^2} \EE \|d_t - \nabla_\omega F(\theta_t, \omega_t)\|^2  + 4  \nu_t \gamma \EE \|p_t - \nabla_\theta F(\theta_t, \omega_t) \|^2.
\end{align*}
According to the conditions that $\eta \geq 9 L_f^2  \gamma/\mu^2$ and $L_\omega = L_f/\mu$ by Lemma \ref{lem:omega_lip}, we have
\begin{align*}
- \Big ( \frac{3\nu_t}{4\gamma} - \frac{375 \nu_t L_f^2 L_\omega^2 }{8\nu^2 \eta}\Big) \leq -\frac{\nu_t}{8\gamma},
\end{align*}
such that 
\begin{align}
\begin{aligned}\label{eq:vr_thm_bound_3}
\EE[R_{t+1}-R_t] &\leq- \frac{\nu_t}{8\gamma} \EE\|\tilde{\theta}_{t+1} - \theta_t\|^2 - \frac{15 L_f^2 \nu_t \gamma}{2\mu\eta} \EE\|\tilde{\omega}_{t+1} - \omega_t \|^2 -\frac{L_f^2 \gamma \nu_t}{2} \EE \|\omega_t - \omega^*(\theta_t)\|^2 \\
&\quad+ \frac{375 L_f^2 \nu_t \gamma}{8\mu^2} \EE \|d_t - \nabla_\omega F(\theta_t, \omega_t)\|^2  + 4 \gamma \nu_t  \EE \|p_t - \nabla_\theta F(\theta_t, \omega_t) \|^2.
\end{aligned}    
\end{align}
The inequality \eqref{eq:vr_thm_bound_3} shows that we need to further bound $\EE \|d_t - \nabla_\omega F(\theta_t, \omega_t)\|^2$ and $\EE \|p_t - \nabla_\theta F(\theta_t, \omega_t) \|^2$ on the right-hand side. By the result of Lemma \ref{lem:vr_grad_var_bound}, we can have contraction of the two terms plus some noise terms, which are 
\begin{align}
\begin{aligned} \label{eq:vr_thm_bound_4}
\EE \|\nabla_\theta F(\theta_{t+1},\omega_{t+1})-p_{t+1}\|^2 &\leq   (1- \alpha \nu^2_t) \EE\|\nabla_\theta F(\theta_t,\omega_t)-p_t\|^2 +  2 \alpha^2 \nu_t^4 \sigma^2 \\
&\quad  +  2 L_f^2 \nu_t^2 \EE[\|\tilde{\theta}_{t+1}-\theta_t\|^2 + \|\tilde{\omega}_{t+1}-\omega_t\|^2],
\end{aligned}
\end{align}
\vspace{-0.5cm}
\begin{align}
\begin{aligned}\label{eq:vr_thm_bound_5}
\EE \|\nabla_\omega F(\theta_{t+1},\omega_{t+1})-d_{t+1}\|^2 &\leq   (1- \alpha \nu^2_t) \EE\|\nabla_\omega F(\theta_t,\omega_t)-d_t\|^2 +  2 \alpha^2 \nu_t^4 \sigma^2 \\
&\quad  +  2  L_f^2 \nu_t^2 \EE[\|\tilde{\theta}_{t+1}-\theta_t\|^2 + \|\tilde{\omega}_{t+1}-\omega_t\|^2].
\end{aligned}
\end{align}  
Multiplying both sides of \eqref{eq:vr_thm_bound_4} by $\nu_t^{-1}$ and also subtracting $ \nu_{t-1}^{-1}\EE \|\nabla_\theta F(\theta_t,\omega_t)-p_t\|^2$ from both sides, we have
\begin{align*}
&\frac{1}{\nu_t}\EE \|\nabla_\theta F(\theta_{t+1},\omega_{t+1})-p_{t+1}\|^2 - \frac{1}{\nu_{t-1}}\EE \|\nabla_\theta F(\theta_t,\omega_t)-p_t\|^2\\ 
&\qquad \leq   \Big(\frac{1}{\nu}_t- \alpha \nu_t-\frac{1}{\nu_{t-1}} \Big)\EE \|\nabla_\theta F(\theta_t,\omega_t)-p_t\|^2
+ 2\alpha^2 \nu_t^3 \sigma^2 + 2 L_f^2 \nu_t  \EE[\|\tilde{\theta}_{t+1} - \theta_t  \|^2 + \|\tilde{\omega}_{t+1}-\omega_t  \|^2].    
\end{align*}
Here we need to simplify the coefficient $\nu_t^{-1} - \nu_{t-1}^{-1}$ shown as below
\begin{align} 
\begin{aligned}\label{eq:inverse_minus}
\frac{1}{\nu_t} - \frac{1}{\nu_{t-1}} =& \frac{1}{a} (b+t)^{1/3} - \frac{1}{a} (b+t-1)^{1/3}\\
\leq& \frac{1}{3a(t+b-1)^{2/3}}  =  \frac{2^{2/3}}{3a [2(t+b-1)]^{2/3}}     \\
\leq& \frac{2^{2/3}}{3a^3}\frac{a^2}{(b+t)^{2/3}} = \frac{2^{2/3}}{3a^3} \nu_t^2,
\end{aligned}
\end{align}
where the first inequality is due to $(x+y)^{1/3}- x^{1/3} \leq yx^{-2/3} / 3$ and the second inequality is by the condition $b\geq 1$ when we let $t \geq 1$. In addition, if we set $ \nu_{-1} = \nu_0$, namely, $t=0$, the above inequality will also hold. 

Thus, by setting $\alpha = 6$,
we let 
\begin{align*}
\frac{1}{\nu_t} - \frac{1}{\nu_{t-1}} - \alpha \nu_t \leq -6 \nu_t + \frac{2^{2/3} \nu_t^2}{3a^3} \leq -\frac{21}{4} \nu_t,
\end{align*}
which requires 
\begin{align*}
\nu_t\leq \frac{18 a^3}{2^{11/3}}, \quad \text{ and } 6 \nu_t^2 \leq 1,
\end{align*} 
that can be guaranteed by the conditions
\begin{align*}
a = \frac{1}{3},\quad  \text{ and } \ \nu_t\leq \frac{1}{19},
\end{align*}
as shown in \eqref{eq:vr_nu_cond}. Hence, we have
\begin{align}
\begin{aligned} \label{eq:vr_thm_bound_6_al}
&\frac{1}{\nu_t}\EE \|\nabla_\theta F(\theta_{t+1},\omega_{t+1})-p_{t+1}\|^2 - \frac{1}{\nu_{t-1}}\EE \|\nabla_\theta F(\theta_t,\omega_t)-p_t\|^2  \\ 
&\qquad \leq   - \frac{21}{4}\nu_t \EE \|\nabla_\theta F(\theta_t,\omega_t)-p_t\|^2
+ 72 \nu_t^3 \sigma^2 +   2\nu_t L_f^2 \EE[\|\tilde{\theta}_{t+1} - \theta_t  \|^2 + \|\tilde{\omega}_{t+1}-\omega_t  \|^2].
\end{aligned}
\end{align}
In a similar way, by \eqref{eq:vr_thm_bound_5}, we also have
\begin{align}
\begin{aligned} \label{eq:vr_thm_bound_7_al}
&\frac{1}{\nu_t}\EE \|\nabla_\omega F(\theta_{t+1},\omega_{t+1})-d_{t+1}\|^2 - \frac{1}{\nu_{t-1}}\EE \|\nabla_\omega F(\theta_t,\omega_t)-d_t\|^2  \\ 
&\qquad \leq   - \frac{21}{4}  \nu_t \EE \|\nabla_\omega F(\theta_t,\omega_t)-d_t\|^2
+ 72 \nu_t^3 \sigma^2 + 2 \nu_t L_f^2 \EE[\|\tilde{\theta}_{t+1} - \theta_t  \|^2 + \|\tilde{\omega}_{t+1}-\omega_t  \|^2].
\end{aligned}
\end{align}
Multiplying both sides of \eqref{eq:vr_thm_bound_6_al} and \eqref{eq:vr_thm_bound_7_al} by $2\gamma/(\mu\eta)$, we have the following two inequalities 
\begin{align}
\begin{aligned} \label{eq:vr_thm_bound_6}
&\frac{2\gamma}{\mu\eta\nu_t}\EE \|\nabla_\theta F(\theta_{t+1},\omega_{t+1})-p_{t+1}\|^2 - \frac{2\gamma}{\mu\eta\nu_{t-1}}\EE \|\nabla_\theta F(\theta_t,\omega_t)-p_t\|^2  \\ 
&\qquad \leq   - \frac{21\gamma \nu_t}{2\mu\eta} \EE \|\nabla_\theta F(\theta_t,\omega_t)-p_t\|^2
+ \frac{144  \nu_t^3 \sigma^2\gamma}{\mu\eta} +   \frac{4\nu_t L_f^2\gamma}{\mu\eta} \EE[\|\tilde{\theta}_{t+1} - \theta_t  \|^2 + \|\tilde{\omega}_{t+1}-\omega_t  \|^2],
\end{aligned}
\end{align}
\vspace{-0.5cm}
\begin{align}
\begin{aligned} \label{eq:vr_thm_bound_7}
&\frac{2\gamma}{\mu\eta\nu_t}\EE \|\nabla_\omega F(\theta_{t+1},\omega_{t+1})-d_{t+1}\|^2 - \frac{2\gamma}{\mu\eta\nu_{t-1}}\EE \|\nabla_\omega F(\theta_t,\omega_t)-d_t\|^2  \\ 
&\qquad\leq   - \frac{21\gamma \nu_t}{2\mu\eta}\EE \|\nabla_\omega F(\theta_t,\omega_t)-d_t\|^2
+ \frac{144 \nu_t^3 \sigma^2\gamma}{\mu\eta}  + \frac{4\nu_t L_f^2\gamma}{\mu\eta} \EE[\|\tilde{\theta}_{t+1} - \theta_t  \|^2 + \|\tilde{\omega}_{t+1}-\omega_t  \|^2].
\end{aligned}
\end{align}
Then, we define the Lyapunov function as
\begin{align*}
S_t := R_t + \frac{2\gamma}{\mu\eta\nu_{t-1}} \|\nabla_\omega F(\theta_t,\omega_t)-d_t\|^2 + \frac{2\gamma}{\mu\eta\nu_{t-1}} \|\nabla_\theta F(\theta_t,\omega_t)-p_t\|^2 , \quad \forall t \geq 0,
\end{align*}
where we have $S_t \geq 0$ since $R_t\geq 0$. Summing up \eqref{eq:vr_thm_bound_3}, \eqref{eq:vr_thm_bound_6}, and \eqref{eq:vr_thm_bound_7}, we have the following inequality
\begin{align*}
&\EE[S_{t+1}-S_t] \\
&\qquad \leq-\Big( \frac{\nu_t}{8\gamma} - \frac{L_f^2 \nu_t\gamma}{\mu\eta}\Big) \EE\|\tilde{\theta}_{t+1} - \theta_t\|^2  - \frac{\gamma \nu_t L_f^2 }{2}\EE \|\omega_t - \omega^*(\theta_t)\|^2 + \frac{36\times 8 \nu_t^3 \sigma^2\gamma}{\mu\eta} \\
&\qquad\quad-\Big(\frac{21\gamma \nu_t}{2\mu\eta}-  \frac{375 \nu_t L_f^2 \gamma}{8\mu^2} \Big) \EE \|d_t - \nabla_\omega F(\theta_t, \omega_t)\|^2  - \Big(\frac{21\gamma\nu_t}{2\mu\eta} - 4 \gamma \nu_t \Big)  \EE \|p_t - \nabla_\theta F(\theta_t, \omega_t) \|^2.  
\end{align*}
where the term $\EE \|\tilde{\omega}_{t+1} - \omega_t\|^2$ has been dropped due to its negative coefficient. We simplify the coefficients according to the following inequalities
\begin{align*}
&-\Big( \frac{\nu_t}{8\gamma} - \frac{8L_f^2 \nu_t\gamma}{\mu\eta}\Big) \leq -\frac{\nu_t}{16\gamma}, ~ -\Big(\frac{21\gamma\nu_t}{2\mu\eta}-  \frac{375 \nu_t L_f^2 \gamma}{8\mu^2} \Big) \leq 0, ~ \text{ and }
- \Big(\frac{21\gamma\nu_t}{2\mu\eta} - 4 \gamma \nu_t \Big) \leq -\frac{\gamma\nu_t}{\mu\eta},
\end{align*}
which are guaranteed by the conditions of this theorem that $ \eta \leq \mu /(6L_f^2)$ and $\eta \geq 9 \gamma L_f^2 /\mu^2 $ with $L_f\geq \mu > 0$. Therefore, we have
\begin{align}
\begin{aligned}\label{eq:vr_thm_bound_al_1}
&\EE[S_{t+1}-S_t] \leq-\frac{\nu_t}{16\gamma} \EE\|\tilde{\theta}_{t+1} - \theta_t\|^2  - \frac{\gamma \nu_t L_f^2}{2}\EE \|\omega_t - \omega^*(\theta_t)\|^2\\
&\qquad \qquad \quad \qquad -\frac{\gamma\nu_t}{\mu\eta}\EE \|p_t - \nabla_\theta F(\theta_t, \omega_t) \|^2+ \frac{36\times 8 \nu_t^3 \sigma^2\gamma}{\mu\eta}.
\end{aligned}
\end{align}
We are in position to prove the convergence of Algorithm \ref{alg:variance_reduce} with the inequality \eqref{eq:vr_thm_bound_al_1}. Taking summation on both sides of \eqref{eq:vr_thm_bound_al_1} from $t=0$ to $T-1$ and rearranging the terms lead to 
\begin{align*}
&\sum_{t=0}^{T-1} \frac{\nu_t \gamma}{16} \left(\frac{1}{\gamma^2} \EE \|\tilde{\theta}_{t+1} - \theta_t\|^2 + \frac{16}{\eta \nu_t}  \EE \|\nabla_\theta F(\theta_t,\omega_t)-p_t\|^2 + 8L_f^2 \EE \|\omega_t - \omega^*(\theta_t)\|^2 \right)\\
&\qquad \leq  \frac{288\sigma^2 \gamma\sum_{t=0}^{T-1}\nu_t^3 }{\mu\eta } + \EE[S_0-S_T]\leq  \frac{288\sigma^2 \gamma\sum_{t=0}^{T-1}\nu_t^3 }{\mu\eta } + S_0,
\end{align*}
where we use the fact that $S_T \geq 0$ in the last inequality. Letting $\{\nu_t\}_{t\geq 0}$ be a non-increasing sequence, we know $\nu_t \geq \nu_T$ for any $0 \leq  t \leq T$.  Since we have $(\eta \nu_t)^{-1}\geq 1$ by the settings of $\eta$, $\gamma$ and $\nu_t$ shown above, we obtain
\begin{align*}
&\frac{\nu_T \gamma}{16} \sum_{t=0}^{T-1}  \left(\frac{1}{\gamma^2}\EE \|\tilde{\theta}_{t+1} - \theta_t\|^2 + \EE \|p_t - \nabla_\theta F(\theta_t, \omega_t)\|^2 + L_F^2 \EE \|\omega_t - \omega^*(\theta_t)\|^2 \right)\\
&\qquad \leq \sum_{t=0}^{T-1} \frac{\nu_t \gamma}{16} \left(\frac{1}{\gamma^2} \EE \|\tilde{\theta}_{t+1} - \theta_t\|^2 + \frac{16}{\eta \nu_t}  \EE \|\nabla_\theta F(\theta_t,\omega_t)-p_t\|^2 + 8L_f^2 \EE \|\omega_t - \omega^*(\theta_t)\|^2  \right)\\
&\qquad  \leq \frac{288\sigma^2 \gamma\sum_{t=0}^{T-1}\nu_t^3 }{\mu\eta } + S_0.
\end{align*}
Multiplying both sides by $16/(T\nu_T \gamma)$ yields
\begin{align}
\begin{aligned}\label{eq:vr_thm_bound_al_3}
&\frac{1}{T} \sum_{t=0}^{T-1}\left (\frac{1}{\gamma^2} \EE \|\tilde{\theta}_{t+1} - \theta_t\|^2 + \frac{16}{\eta \nu_t}  \EE \|\nabla_\theta F(\theta_t,\omega_t)-p_t\|^2 + 8L_f^2 \EE \|\omega_t - \omega^*(\theta_t)\|^2 \right) \\
&\qquad \leq \frac{288 \times 16   \sigma^2 }{\mu\eta T \nu_T  } \sum_{t=0}^{T-1} \nu_t^3 + \frac{16S_0}{\gamma \nu_T T}. 
\end{aligned}
\end{align}
According to the setting of the step size that $\nu_t = a/(t+b)^{1/3}$ with $a = 1/3$ and $b \geq  \max \{ (8\gamma L_f^2/\mu)^3, 256 \}$, we can bound the two terms on the right-hand side of \eqref{eq:vr_thm_bound_al_3} in the following way
\begin{align}
\begin{aligned}
\label{eq:vr_thm_bound_al_4}
\frac{16S_0}{\gamma \nu_T T} + \frac{288 \times 16   \sigma^2 }{\mu\eta T \nu_T  } \sum_{t=0}^{T-1} \nu_t^3  &\leq \frac{50 S_0 (T + b)^{1/3}}{\gamma  T} +   \frac{512\sigma^2}{\mu\eta T } \int_{0}^T \frac{1}{t+b} \mathrm{d} t \\
&\leq \frac{50 S_0 (T + b)^{1/3}}{\gamma  T} +   \frac{512 \sigma^2}{\mu\eta T }  (T+b)^{1/3} \log(T+b)\\
&\leq \Big( \frac{50 S_0 }{\gamma  } +   \frac{512 \sigma^2}{\mu\eta }  \Big ) \Big ( \frac{1}{T^{2/3}} + \frac{b^{1/3}}{T} \Big ) \log(T+b) ,
\end{aligned}
\end{align}
where the last inequality is due to $(x+y)^{1/3}\leq x^{1/3} + y^{1/3}$ for $x,y \geq 0$.

Finally, we combine \eqref{eq:vr_thm_bound_al_3} and \eqref{eq:vr_thm_bound_al_4} and obtain
\begin{align*}
&\frac{1}{T} \sum_{t=0}^{T-1} \Big [\frac{1}{\gamma^2} \EE \|\nabla_\theta F(\theta_t,\omega_t)-p_t\|^2 + \EE \|\tilde{\theta}_{t+1} - \theta_t\|^2  + L_F^2 \EE \|\omega_t - \omega^*(\theta_t)\|^2 \Big ] \\
&\qquad \leq \frac{\overline{C}_1  \log (T+b) }{T^{2/3}} + \frac{\overline{C}_2  \log (T+b)}{T},
\end{align*}
where $\overline{C}_1 = 50 S_0/\gamma +  512 \sigma^2/(\mu\eta)$ and $\overline{C}_2 = \overline{C}_1 b^{1/3}$. 
Moreover, by Jensen's inequality, there is
\begin{align*}
&\frac{1}{T} \sum_{t=0}^{T-1} \EE \left (\frac{1}{\gamma} \|\tilde{\theta}_{t+1} - \theta_t\| +  \|p_t - \nabla_\theta F(\theta_t, \omega_t)\| + L_f  \|\omega_t - \omega^*(\theta_t)\| \right ) \\
&\qquad \leq  \left [ \frac{3}{T} \sum_{t=0}^{T-1} \EE \left(\frac{1}{\gamma^2} \|\tilde{\theta}_{t+1} - \theta_t\|^2 +  \|p_t - \nabla_\theta F(\theta_t, \omega_t)\|^2 + L_f^2  \|\omega_t - \omega^*(\theta_t)\|^2  \right ) \right]^{1/2}.
\end{align*}
Thus, we eventually obtain
\begin{align*}
&\frac{1}{T} \sum_{t=0}^{T-1} \EE \left (\frac{1}{\gamma} \|\tilde{\theta}_{t+1} - \theta_t\| +  \|p_t - \nabla_\theta F(\theta_t, \omega_t)\| + L_f  \|\omega_t - \omega^*(\theta_t)\| \right) \\
&\qquad \leq \frac{C_1  \log (T+b) }{T^{1/3}} + \frac{C_2 \log (T+b) }{\sqrt{T}} = \widetilde{O}\left( \frac{1}{T^{1/3}} \right), 
\end{align*}
where $C_1 = [150 S_0/\gamma +  1536 \sigma^2/(\mu\eta)]^{1/2}$ and $C_2 = C_1 b^{1/6}$. This completes the proof.
\end{proof} 

\subsection{Proof of Theorem \ref{thm:variance_reduce_fix}}\label{sec:detail_proof_variance_reduce_fix}

\begin{proof} The analysis in Section \ref{sec:detail_proof_variance_reduce} is for non-increasing step size $\nu_t$ satisfying the conditions in Theorem \ref{thm:variance_reduce} which are the same as in this theorem. Then, the proof before \eqref{eq:vr_thm_bound_al_3} can be adapted here for a fixed step size. The barrier may be that the analysis from \eqref{eq:inverse_minus} to \eqref{eq:vr_thm_bound_7_al} is for the decaying step size setting. In this proof, we modify this part by letting $\nu_t = \nu, \forall t\geq 0$, and 
\begin{align*}
\frac{1}{\nu_t} - \frac{1}{\nu_{t-1}} - \alpha \nu_t = \frac{1}{\nu} - \frac{1}{\nu} - \alpha \nu  = -6 \nu < -\frac{21}{4} \nu,
\end{align*} 
where we set $\alpha = 6$. This shows that we can still use the proof in Section \ref{sec:detail_proof_variance_reduce}. Thus, we start our proof from \eqref{eq:vr_thm_bound_al_4} with replacing $\nu_t$ by the fixed step size $\nu$, which can be rewritten as
\begin{align}
\begin{aligned}
&\frac{1}{T} \sum_{t=0}^{T-1}\left (\frac{1}{\gamma^2} \EE \|\tilde{\theta}_{t+1} - \theta_t\|^2 + \frac{16}{\eta \nu_t}  \EE \|\nabla_\theta F(\theta_t,\omega_t)-p_t\|^2 + 8L_f^2 \EE \|\omega_t - \omega^*(\theta_t)\|^2 \right) \\
&\qquad \leq \frac{288 \times 16   \sigma^2 }{\mu\eta }  \nu^2 + \frac{16S_0}{\gamma \nu T}. 
\end{aligned}
\end{align}
Furthermore, we bound the right-hand side of the above inequality as
\begin{align}
\begin{aligned}
\frac{288 \times 16   \sigma^2 }{\mu\eta }  \nu^2 + \frac{16S_0}{\gamma \nu T}  &\leq \frac{512\sigma^2}{\mu\eta (T+b)^{2/3} }   +  \frac{50 S_0 (T + b)^{1/3}}{\gamma  T}\\
&\qquad \leq \frac{50 S_0}{\gamma  T^{2/3}} + \frac{512\sigma^2}{\mu\eta (T+b)^{2/3} } +  \frac{50 S_0  b^{1/3}}{\gamma  T},
\end{aligned}
\end{align}
where we use the setting of the step size $\nu=1/[3(T+b)^{1/3}]$, and the last inequality is due to $(x+y)^{1/3}\leq x^{1/3} + y^{1/3}$ for $x,y \geq 0$. Then, we obtain
\begin{align*}
&\frac{1}{T} \sum_{t=0}^{T-1} \Big [\frac{1}{\gamma^2} \EE \|\nabla_\theta F(\theta_t,\omega_t)-p_t\|^2 + \EE \|\tilde{\theta}_{t+1} - \theta_t\|^2  + L_F^2 \EE \|\omega_t - \omega^*(\theta_t)\|^2 \Big ] \\
&\qquad \leq  \frac{50 S_0}{\gamma  T^{2/3}} + \frac{512\sigma^2}{\mu\eta (T+b)^{2/3} } +  \frac{50 S_0  b^{1/3}}{\gamma  T}.
\end{align*}
Due to Jensen's inequality, we have
\begin{align*}
&\frac{1}{T} \sum_{t=0}^{T-1} \EE \left (\frac{1}{\gamma} \|\tilde{\theta}_{t+1} - \theta_t\| +  \|p_t - \nabla_\theta F(\theta_t, \omega_t)\| + L_f  \|\omega_t - \omega^*(\theta_t)\| \right ) \\
&\qquad \leq  \left [ \frac{3}{T} \sum_{t=0}^{T-1} \EE \left(\frac{1}{\gamma^2} \|\tilde{\theta}_{t+1} - \theta_t\|^2 +  \|p_t - \nabla_\theta F(\theta_t, \omega_t)\|^2 + L_f^2  \|\omega_t - \omega^*(\theta_t)\|^2  \right ) \right]^{1/2}.
\end{align*}
Thus, we eventually obtain
\begin{align*}
&\frac{1}{T} \sum_{t=0}^{T-1} \EE \left (\frac{1}{\gamma} \|\tilde{\theta}_{t+1} - \theta_t\| +  \|p_t - \nabla_\theta F(\theta_t, \omega_t)\| + L_f  \|\omega_t - \omega^*(\theta_t)\| \right) \leq O \left ( \frac{1}{T^{1/3}}  \right), 
\end{align*}
which leads to $T_{\varepsilon} \geq O(\varepsilon^{-3})$ sample complexity to achieve an $\varepsilon$ error.  This completes the proof.
\end{proof}